%% file: main.tex
\documentclass{article}

\usepackage[nonatbib,preprint]{neurips_2020}
\usepackage[utf8]{inputenc} 
\usepackage[T1]{fontenc}    
\usepackage{hyperref}       
\usepackage{url}            
\usepackage{booktabs}       
\usepackage{amsfonts}       
\usepackage{nicefrac}       
\usepackage{microtype}      

\usepackage{amsmath}
\usepackage{amssymb,amsthm,mathtools,bbm,verbatim}
\usepackage[pdftex]{graphicx}
\usepackage{cite}
\usepackage{caption}
\usepackage{subcaption}
\usepackage{algorithm}
\usepackage{algpseudocode} 
\allowdisplaybreaks 
\hypersetup{colorlinks = true, linkcolor = red, citecolor = blue, urlcolor  = blue}
\usepackage{float}

\def\R{{\mathbb{R}}}
\def\N{{\mathbb{N}}}
\def\1{{\textbf{1}}}
\def\0{{\textbf{0}}}
\def\I{{\mathbbm{1}}}
\def\P{{\mathbb{P}}}
\def\E{{\mathbb{E}}}
\def\VAR{{\mathbb{V}\mathbb{A}\mathbb{R}}}
\def\L{{\mathcal{L}}}

\def\S{{\mathcal{S}}}
\def\F{{\mathcal{F}}}
\def\Q{{\mathcal{Q}}}
\def\G{{\mathcal{G}}}

\DeclareMathOperator\diff{{d \!}}

\DeclareMathOperator\image{{im}}
\DeclareMathOperator*{\esssup}{ess\,sup}

\def\unif{{\mathsf{unif}}}

\newcommand{\pdf}{P_{\alpha}}
\newcommand{\pdfs}{{\cal P}}
\newcommand{\estpdf}{\hat{P}}
\newcommand{\est}{\widehat{\cal P}}
\newcommand{\ests}{\widehat{\cal P}^*}
\newcommand{\estsp}{\widehat{\cal P}^+}
\newcommand{\iotab}{\mathbf{i}}
\newcommand{\pib}{\boldsymbol{\pi}}

\newtheorem{theorem}{Theorem}
\newtheorem{lemma}{Lemma}
\newtheorem{proposition}{Proposition}

\newtheorem{claim}{Claim}

\theoremstyle{definition}

\newcommand{\figref}[1]{Figure~\ref{#1}}
\newcommand{\tabref}[1]{Table~\ref{#1}}
\newcommand{\thmref}[1]{Theorem~\ref{#1}}
\newcommand{\propref}[1]{Proposition~\ref{#1}}
\newcommand{\lemref}[1]{Lemma~\ref{#1}}
\newcommand{\clmref}[1]{Claim~\ref{#1}}

\title{Estimation of Skill Distributions}

\author{%
  Ali~Jadbabaie \hspace{0.25in} Anuran~Makur \hspace{0.25in} Devavrat~Shah\thanks{The author ordering is alphabetical.}\\
	Laboratory for Information and Decision Systems\\
  Massachusetts Institute of Technology\\
  Cambridge, MA 02139\\
  \texttt{\{jadbabai,a\_makur,devavrat\}@mit.edu}\\
}

\begin{document}

\maketitle

\begin{abstract} 
In this paper, we study the problem of learning the skill distribution of a population of agents from observations of pairwise games in a tournament. These games are played among randomly drawn agents from the population. The agents in our model can be individuals, sports teams, or Wall Street fund managers. Formally, we postulate that the likelihoods of outcomes of games are governed by the parametric Bradley-Terry-Luce (or multinomial logit) model, where the probability of an agent beating another is the ratio between its skill level and the pairwise sum of skill levels, and the skill parameters are drawn from an unknown, non-parametric skill density of interest. The problem is, in essence, to learn a distribution from noisy, quantized observations. We propose a surprisingly simple and tractable algorithm that learns the skill density with near-optimal minimax mean squared error scaling as $n^{-1+\varepsilon}$, for any $\varepsilon>0$, so long as the density is smooth. Our approach brings together prior work on learning skill parameters from pairwise comparisons with kernel density estimation from non-parametric statistics. Furthermore, we prove information theoretic lower bounds which establish minimax optimality of the skill parameter estimation technique used in our algorithm. These bounds utilize a continuum version of Fano's method along with a careful covering argument. We apply our algorithm to various soccer leagues and world cups, cricket world cups, and mutual funds. We find that the entropy of a learnt distribution provides a quantitative measure of skill, which in turn provides rigorous explanations for popular beliefs about perceived qualities of sporting events, e.g., soccer league rankings. Finally, we apply our method to assess the skill distributions of mutual funds. Our results shed light on the abundance of low quality funds prior to the Great Recession of 2008, and the domination of the industry by more skilled funds after the financial crisis.
\end{abstract}

\input{intro}

\input{setup}
\input{results}

\input{experiments}

\input{conc}

\section*{Broader Impact}

The analysis of our algorithm, which forms the main contribution of this work, is theoretical in nature, and therefore, does not have any foreseeable societal consequences. On other other hand, applications of our algorithm to real-world settings could have potential societal impacts. As outlined at the outset of this paper, our algorithm provides a data-driven approach to address questions about perceived qualities of sporting events or other competitive enterprises, e.g., financial markets. Hence, a potential positive impact of our work is that subjective beliefs of stakeholders regarding the distributions of relative skills in competitive events can be moderated by a rigorous statistical method. In particular, our method could assist sports teams, sports tournament organizers, or financial firms to corroborate existing trends in the skill levels of players, debunk erroneous myths, or even unveil entirely new trends based on available data. However, our work may also have negative consequences if utilized without paying heed to its limitations. Recall that Step 1 of Algorithm \ref{Algorithm: Density Estimation} estimates BTL skill parameters of agents that participate in a tournament. Since the BTL model is a well-known approach for ranking agents \cite{BradleyTerry1952,Luce1959}, it should be used with caution, as with any method that discriminates among agents. Indeed, the BTL model only takes into account wins or losses of pairwise games between agents, but does not consider the broader circumstances surrounding these outcomes. For example, in the context of soccer, the BTL model does not consider the goal difference in a game to gauge how significant a win really is, or take into account the injuries sustained by players. Yet, rankings of teams or players may be used by team managements to make important decisions such as assigning remunerations. Thus, users of algorithms such as ours must refrain from solely using rankings or skill distributions to make decisions that may adversely affect individuals.

\begin{ack}
This work was supported in part by the Vannevar Bush Fellowship, in part by the ARO Grant W911NF-18-S-0001, in part by the NSF CIMS 1634259, in part by the NSF CNS 1523546, and in part by the MIT-KACST project.
\end{ack}

\appendix

\input{propositions}

\input{upperbound}

\input{lowerbound}

\input{concentration}

\bibliographystyle{myIEEEtran}
\bibliography{SkillRefs}

\end{document}

%% file: intro.tex
\section{Introduction}
\label{Introduction}

It is a widely-held  belief among soccer enthusiasts that English Premier League (EPL) is the most competitive amongst professional leagues even though the likely eventual winner is often one of a handful of usual suspects \cite{McIntyre2019,Spacey2020}. Similarly, the Cricket World Cup in 2019 is believed to be the most exciting in the modern history of the sport, and ended with one of the greatest matches of all time \cite{Smythetal2019,Bull2019}. But is any of this backed up by data, or are they just common misconceptions? In this work, we answer this question by quantifying such observations, beyond mere sports punditry and subjective opinions, in a data-driven manner. We then illustrate that a similar approach can be used to quantify the evolution of the overall quality and relative skills of mutual funds over the years.

To this end, we posit that the population of agents in a tournament, e.g., EPL teams or mutual fund managers, has 
an associated distribution of skills with a probability density function (PDF) $\pdf$ over $\mathbb{R}_+$. 
Our goal is to learn this $\pdf$. Traditionally, in the non-parametric statistics literature, cf. \cite{Tsybakov2009}, one observes samples from the 
distribution directly to estimate $\pdf$. In our setting, however, we can only observe extremely noisy, quantized values. Specifically, given $n$ individuals, teams, or players participating in a tournament, indexed by $[n] \triangleq \{1,\dots,n\}$, let their skill levels be $\alpha_i$,~$i \in [n]$, which are sampled independently from $\pdf$. We observe the outcomes of pairwise games or comparisons between them. More precisely, for
each $i \neq j \in [n]$, with probability $p \in (0,1]$, we observe the outcomes of $k \geq 1$ games, and
with probability $1-p$, we observe nothing. Let $\G(n,p)$ denote the induced Erd{\H o}s-R\'{e}nyi random graph on $[n]$ with
edge $\{i,j\} \in \G(n,p)$ if games between $i$ and $j$ are observed. For $\{i,j\} \in \G(n,p)$, let $Z_m(i,j) \in \{0,1\}$ denote whether $j$ beats 
$i$, i.e., value $1$ if $j$ beats $i$ and $0$ otherwise, in game $m \in [k]$. By definition, $Z_m(i,j) + Z_m(j, i) = 1$. 
We assume the \emph{Bradley-Terry-Luce (BTL)} \cite{BradleyTerry1952,Luce1959} or \emph{multinomial logit model} \cite{McFadden1973} where: 
\begin{equation}
\label{Eq: BT model}
\P\!\left(Z_m(i,j) = 1 \, \middle| \, \alpha_1,\dots,\alpha_n\right) \triangleq \frac{\alpha_j}{\alpha_i + \alpha_j}, 
\end{equation}
independently of the outcomes of all other games. Our objective is to learn $\pdf$ from the observations $\{Z_m(i,j): \{i,j\} \in \G(n,p), \, m \in [k]\}$, instead of $\alpha_i, i\in [n]$ (as in traditional statistics \cite{Tsybakov2009}). 
\begin{table}[t]
\caption{Comparison of our contributions with prior works. The notation $\tilde{O}$ and $\tilde{\Omega}$ hide ${\sf poly}(\log(n))$ terms, and $\varepsilon > 0$ is any
arbitrarily small constant.} 
\vspace{.1in}
\centering
\begin{tabular}{llcc}
\toprule
\textbf{Estimation problem} & \textbf{Loss function} & \textbf{Upper bound} & \textbf{Lower bound} \\
\toprule
Smooth $\mathcal{C}^{\infty}$ skill PDF & MSE & \textcolor{blue}{$\tilde{O}(n^{-1 + \varepsilon})$} (\thmref{Thm: MSE Upper Bound}) & $\Omega(n^{-1})$ \cite{IbragimovKhasminskii1982,Tsybakov2009} \\
BTL skill parameters & relative $\ell^{\infty}$-norm & $\tilde{O}(n^{-1/2})$ \cite{Chenetal2019} &  \textcolor{blue}{$\tilde{\Omega}(n^{-1/2})$} (\thmref{Thm: Minimax Relative l^infty-Risk}) \\
BTL skill parameters & $\ell^{1}$-norm & $O(n^{-1/2})$ \cite{Chenetal2019} & \textcolor{blue}{$\tilde{\Omega}(n^{-1/2})$} (\thmref{Thm: Minimax l^1-Risk}) \\
\bottomrule
\end{tabular}
\label{Table: Technical Contributions}
\end{table}
For a given, fixed set of  $\alpha_i, i \in [n]$, learning them from pairwise comparison data $\{Z_m(i,j): \{i,j\} \in \G(n,p), \, m \in [k]\}$ has been extensively studied in the  recent literature \cite{NegahbanOhShah2012,NegahbanOhShah2017, Chenetal2019}. Nevertheless, this line of research does not provide any means to estimate the underlying skill distribution $\pdf$.

\noindent{\bf Contributions.} As the main contribution of this work, we develop a statistically near-optimal and computationally tractable method for estimating the skill distribution $\pdf$ from a subset of 
pairwise comparisons. Our estimation method is a two-stage algorithm that uses the (spectral) 
rank centrality estimator \cite{NegahbanOhShah2012,NegahbanOhShah2017} followed by the 
Parzen-Rosenblatt kernel density estimator \cite{Rosenblatt1956,Parzen1962} with carefully 
chosen bandwidth. We establish that the minimax \emph{mean squared error (MSE)} of our method scales as $\tilde{O}(n^{-\eta/(\eta+1)})$ for any $\pdf$ belonging to an $\eta$-H\"{o}lder class. Thus, if $\pdf$ is smooth ($\mathcal{C}^{\infty}$) with bounded derivatives, then the minimax MSE is $\tilde{O}(n^{-1 + \varepsilon})$ for any $\varepsilon > 0$; see \thmref{Thm: MSE Upper Bound} for details. Somewhat surprisingly, although we do not directly observe $\alpha_i, i \in [n]$, this
minimax MSE rate matches the minimax MSE lower bound of $\Omega(n^{-1})$ for smooth $P_{\alpha}$ even when $\alpha_i, i \in [n]$ are observed \cite{IbragimovKhasminskii1982,Tsybakov2009}. 

As a key step in our estimation method, we utilize the rank centrality algorithm \cite{NegahbanOhShah2012,NegahbanOhShah2017} for estimating $\alpha_i, i \in [n]$. While the optimal learning rate of the rank centrality algorithm with respect to relative $\ell^2$-loss is well-understood \cite{NegahbanOhShah2012,NegahbanOhShah2017,Chenetal2019}, the optimal learning rates with respect to relative $\ell^{\infty}$ and $\ell^1$-losses are not known since we only know upper bounds \cite{Chenetal2019}, but not matching minimax lower bounds. In Theorems \ref{Thm: Minimax Relative l^infty-Risk} and \ref{Thm: Minimax l^1-Risk}, we prove minimax lower bounds of $\tilde{\Omega}(n^{-1/2})$ with respect to both relative $\ell^{\infty}$ and $\ell^1$-losses. These bounds match the learning rates of the rank centrality algorithm obtained in \cite{Chenetal2019} with respect to both $\ell^{\infty}$ and $\ell^1$-losses, and hence, identify the optimal minimax rates. We derive these information theoretic lower bounds by employing a recent variant of the generalized Fano's method with covering arguments. (Our main technical results are all delineated in \tabref{Table: Technical Contributions}.)

Finally, we illustrate the utility of  our algorithm through four experiments on real-world data: cricket world cups, 
soccer world cups, European soccer leagues, and mutual funds. Intuitively, a concentrated skill distribution, i.e., one that is close to a Dirac delta measure, corresponds to a balanced tournament with players that are all equally skilled. Hence, the outcomes of games are random or unpredictable. On the other hand, a skill distribution that is close to uniform suggests a wider spread of players' skill levels. So, the outcomes of games are driven more by skill rather than luck (or random chance). We, therefore, propose to use the negative entropy of a learnt skill distribution as a way to measure 
the ``overall skill score,'' because negative entropy captures distance to the uniform distribution. 
For cricket world cups, we find that negative entropy decreases from 2003 to 2019. Indeed, this corroborates with fan experience, where in 2003, Australia and India dominated but all other teams were roughly equal, while in 2019, there was a healthy spread of skill levels making many teams potential contenders for the championship. In soccer, we observe that the EPL and World Cup have high negative entropy, which indicates that most teams are competitive, and thus, it is very difficult to predict outcomes up front. Lastly, the negative entropy of US mutual funds decreases significantly during the Great Recession of 2008, and we see flatter skill distributions post 2008. This reveals that mutual funds became more competent to avoid being weeded out of the market by the financial crisis.

\noindent{\bf Related work.} 
The problem of estimating distributions of skill levels from tournaments has received increased attention due to the recent advent of fantasy sports platforms, which give rise to new legal and policy making challenges concerned with regulating the accompanying rise of gambling on such platforms, cf. \cite{Gettyetal2018} and follow-up work. Indeed, when the distribution of skill levels of players is concentrated around one point, the associated game is essentially one of chance (or luck), and governments may understandably seek to place more betting regulations on such tournaments. While \cite{Gettyetal2018} provides an empirical study of an ad hoc measure of skill using fantasy sports data, we consider a rigorous statistical formulation of this problem where the objective is to estimate an unknown PDF of skill levels from partially observed win-loss data of tournaments. 

As mentioned earlier, we assume that all players in a tournament have latent ``skill'' or ``merit'' parameters that are drawn from an unknown prior skill PDF, and these skill parameters determine the likelihoods of wins and losses in games according to the BTL model. Our algorithm to estimate such skill distributions proceeds by first estimating skill parameters from the observed data, and then estimating the skill distribution based on these parameter estimates. To estimate the skill PDFs from (estimated) skill parameters in the second stage of our algorithm, we exploit kernel density estimation techniques that were originally developed in \cite{Rosenblatt1956,Parzen1962,Epanechnikov1969}. Moreover, as noted in \tabref{Table: Technical Contributions}, to evaluate the minimax MSE risk achieved by our algorithm, we compare our MSE risk scaling with well-known minimax lower bounds on density estimation for certain classes of analytic densities, cf. \cite{IbragimovKhasminskii1982} and the references therein. On a separate front, to establish the near-optimality of the skill parameter estimation technique (to be explained in due course) used in our algorithm, we exploit a variant of the generalized Fano's method. This method was also initially developed in the context of density estimation in \cite{IbragimovKhasminskii1977,Khasminskii1979}. For the sake of brevity, we do not review the extensive non-parametric density estimation literature any further, and instead refer readers to \cite[Chapters 1 and 2]{Tsybakov2009}, \cite{Wasserman2019}, and the references therein for thorough modern treatments.

Since we assume that the likelihoods of the outcomes of two-player games in a tournament follow the BTL model \cite{BradleyTerry1952,Luce1959}, and estimation of the skill parameters of this model forms the first stage of our proposed algorithm, we outline several relevant aspects of the vast literature concerning the BTL model in the remainder of this section. Indeed, while the BTL model was introduced in statistics to study pairwise comparisons \cite{BradleyTerry1952}, it has a long and diverse history. The model was initially proposed by Zermelo in \cite{Zermelo1929}, who also provided an iterative algorithm to compute the maximum likelihood (ML) estimators of the BTL skill parameters. Moreover, the BTL model is a special case of the \emph{Plackett-Luce (PL) model} \cite{Luce1959,Plackett1975}, which was originally developed in mathematical psychology. The PL model defines a probability distribution over rankings (or permutations) of players that is a natural consequence of \emph{Luce's choice axiom}. This axiom can be perceived as a formulation of the ``\emph{independence of irrelevant alternatives}'' in social choice theory and econometrics. In fact, McFadden's work on the multinomial logit model in economics is equivalent to the PL model \cite{McFadden1973}. The earliest known model that is related to the PL model is perhaps the \emph{Thurstonian model} from psychometrics, which provides a probability distribution over rankings using the so called \emph{law of comparative judgment} \cite{Thurstone1927}. Specifically, Thurstone models a ``discriminal process'' to rank $n$ items by first associating latent merit parameters $\alpha_1,\dots,\alpha_n$ to each of the $n$ items, and then ranking them by ranking the corresponding random variables $\alpha_1 + X_1,\dots,\alpha_n + X_n$, where the independent and identically distributed (i.i.d.) random variables $X_1,\dots,X_n$ represent ``noise'' in the discriminal process. As explained in \cite[Section 9D]{Diaconis1988}, the resulting distribution over rankings is equivalent to the PL model when the $X_i$'s have Gumbel (or generalized extreme value type-I) distribution, cf. \cite{Yellott1977}. We refer readers to \cite[Sections 9C and 9D]{Diaconis1988} for other models of rankings based on \emph{exponential families} and further equivalent formulations of the BTL and PL models, and to \cite{Shah2019} for a comprehensive discussion on other equivalent models from a modern machine learning perspective. For example, the celebrated \emph{Boltzmann-Gibbs distribution} in statistical physics and the \emph{softmax} model in machine learning are also versions of the PL model.

In order to estimate the skill parameters of the BTL model, two families of algorithms have been developed in the literature. The first of these is a class of \emph{minorization-maximization (MM) algorithms} that generalize Zermelo's iterative algorithm in \cite{Zermelo1929}. Much like how Zermelo's algorithm computes ML estimators of the parameters under a \emph{strong connectivity condition} \cite{Ford1957} (also see \cite[Assumption 1]{Hunter2004} for a graph theoretic interpretation), the more general MM algorithms can be utilized to perform ML estimation for ``generalized'' BTL models \cite{Hunter2004}. Moreover, although MM algorithms are typically seen as extending the better known \emph{expectation-maximization (EM) algorithms} for ML estimation of latent variable models (e.g., Gaussian mixture models) \cite{DempsterLairdRubin1977}, the MM algorithms for generalized BTL models can also be construed as special cases of EM algorithms (corresponding to certain choices of latent variables) \cite{CaronDoucet2012}. In contrast, in this paper, we utilize the second, more recently discovered, family of spectral algorithms based on the notion of rank centrality introduced in \cite{NegahbanOhShah2012,NegahbanOhShah2017}. The main innovation of such spectral algorithms is to construe (normalized) skill parameters as an invariant distribution of a reversible Markov chain, and armed with this perspective, estimate skill parameters by first estimating the stochastic kernel defining the Markov chain. 

Both MM and spectral algorithms have been analyzed extensively in the literature. For instance, \cite{SimonsYao1999} proves the consistency and asymptotic normality of ML estimators for skill parameters computed by Zermelo's algorithm, and \cite[Theorems 1 and 2]{NegahbanOhShah2017} establishes sample complexity bounds for the relative $\ell^2$-norm estimation error of (normalized) skill parameters. Furthermore, both families of algorithms are shown to be optimal for recovering the top ${\sf K}$ ranked players in \cite{Chenetal2019}, which presents non-asymptotic analysis for relative $\ell^{\infty}$ and $\ell^2$-norm losses. In particular, \cite{NegahbanOhShah2017,Chenetal2019} assume that a random Erd{\H o}s-R\'{e}nyi graph captures the subset of pairwise games that are observed in a tournament. Our analysis also considers this partial observation model, and exploits the relative $\ell^{\infty}$ and $\ell^2$-norm loss results of \cite{Chenetal2019}. In a different vein, \cite{Shahetal2016} establishes minimax estimation bounds for squared semi-norm losses defined by graph Laplacian matrices, where the fixed graphs encode the subsets of observed pairwise games (also see follow-up work), and \cite{Chatterjee2015} demonstrates that the \emph{universal singular value thresholding} algorithm can be used to estimate ``non-parametric'' BTL models. Finally, we refer readers to \cite{GuiverSnelson2009,CaronDoucet2012} and the references therein for other recent research on efficient Bayesian inference for BTL and PL models. As opposed to these works, we also analyze minimax estimation of skill parameters under a previously unexplored setting where parameters are drawn i.i.d. from a prior skill PDF.

\noindent{\bf Notation.} 
We briefly introduce some relevant notation. Let $\N \triangleq \{1,2,3,\dots\}$ denote the set of natural numbers. For any $n \in \N$, let $\S_n$ denote the probability simplex of row probability vectors in $\R^n$, and $\S_{n \times n}$ denote the set of all $n \times n$ row stochastic matrices in $\R^{n \times n}$. For any vector $x \in \R^n$ and any $q \in [1,\infty]$, let $\|x\|_{q}$ denote the $\ell^q$-norm of $x$. Moreover, $\log(\cdot)$ denotes the natural logarithm function with base $e$, $\I\{\cdot\}$ denotes the indicator function that equals $1$ if its input proposition is true and $0$ otherwise, and $\lceil\cdot \rceil$ denotes the ceiling function. Finally, we will use standard Bachmann-Landau asymptotic notation, e.g., $O(\cdot)$, $\Omega(\cdot)$, $\Theta(\cdot)$, where it is understood that $n \rightarrow \infty$, and tilde notation, e.g., $\tilde{O}(\cdot)$, $\tilde{\Omega}(\cdot)$, $\tilde{\Theta}(\cdot)$, when we neglect ${\sf poly}(\log(n))$ factors and problem parameters other than $n$.

%% file: setup.tex
\section{Estimation algorithm}
\label{Formal model}

\noindent{\bf Overview.} Our interest is in estimating the skill PDF $\pdf$ from noisy, discrete observations $\{Z_m(i,j) : \{i,j\}\in \G(n,p), \, m \in [k]\}$. Instead, if we had exact knowledge of the samples $\alpha_i, i \in [n]$ from $\pdf$, then we could utilize traditional methods from non-parametric statistics such as kernel density estimation. However, we do not have access to these samples. So, given pairwise comparisons 
$\{Z_m(i,j) : \{i,j\}\in \G(n,p), \, m \in [k]\}$ generated as per the BTL model with parameters $\alpha_i, i \in [n]$, we can use some recent developments from the BTL-related literature to estimate these skill parameters first. Therefore, a natural two-stage algorithm is to first estimate $\alpha_i, i \in [n]$ using the observations, and then use these estimated parameters to produce an estimate of $\pdf$. We do precisely this. The key challenge is to ensure that the PDF estimation method is robust to the estimation 
error in $\alpha_i, i \in [n]$. As our main contribution, we rigorously argue that carefully chosen methods for both steps produces as good an estimation of $\pdf$ as if we had access to the exact knowledge of $\alpha_i, i \in [n]$. 

\noindent{\bf Setup.} We formalize the setup here. For any given $\delta, \epsilon, b \in (0,1)$ and $\eta,L_1,B > 0$, let $\pdfs = \pdfs(\delta, \epsilon, b, \eta, L_1, B)$ be the set of all uniformly bounded PDFs with respect to the Lebesgue measure on $\R$ that have support in $[\delta,1]$, belong to the $\eta$-\emph{H\"{o}lder class} \cite[Definition 1.2]{Tsybakov2009}, and are lower bounded by $b$ in an $\epsilon$-neighborhood of $1$. More precisely, for every $f \in \pdfs$, $f$ is bounded (almost everywhere), i.e., $f(x) \leq B$ for all $x \in [\delta,1]$; $f$ is $s = \lceil \eta \rceil - 1$ times differentiable, 
and its $s$th derivative $f^{(s)} : [\delta,1] \rightarrow \R$ satisfies $|f^{(s)}(x) - f^{(s)}(y)| \leq L_1 |x - y|^{\eta - s}$ for all 
$x,y \in [\delta,1]$; and
 $f(x) \geq b$ for all $x \in [1 - \epsilon,1]$. As an example, when $\eta = 1$, $\pdfs$ denotes the set of all Lipschitz continuous PDFs on $[\delta,1]$ that are lower bounded near $1$. Furthermore, we define the observation matrix $Z \in [0,1]^{n\times n}$, whose $(i,j)$th entry is:
\begin{equation}
\label{Eq: Observation matrix}
\forall i,j \in [n], \enspace Z(i,j) \triangleq 
\begin{cases}
\I\!\left\{\{i,j\} \in \G(n,p)\right\} \frac{1}{k} \sum_{m = 1}^{k}{Z_m(i,j)} \, , & i \neq j \, , \\
0 \, , & i = j \, . 
\end{cases}
\end{equation}

\noindent{\bf Estimation error.} It turns out that $Z$ is a sufficient statistic for the purposes of estimating $\alpha_i, i \in n$ \cite[p.2208]{Chenetal2019}. For this reason, we shall restrict our attention to all possible estimators of $\pdf$ using $Z$. 
Specifically, let $\est$ be set of all possible measurable and potentially randomized estimators that map $Z$ to a Borel measurable function from $\R$ to $\R$. Then, the minimax MSE risk is defined as: 
\begin{equation}
\label{Eq: Minimax MSE Formulation}
R_{\mathsf{MSE}}(n) \triangleq \inf_{\estpdf \in \est}\,{ \sup_{\pdf \in \pdfs}{ \E\!\left[\int_{\R}{\left(\estpdf(x) - \pdf(x)\right)^2\diff{x}}\right] } } 
\end{equation}
where the expectation is with respect to the randomness in $Z$ as well as within the estimator. Our interest will be in understanding the scaling of $R_{\mathsf{MSE}}(n)$ as a function of $n$ and $\eta$. In the sequel, we will assume that the parameters
$k, p, \delta, \epsilon, b$ can depend on $n$, and all other parameters are constant.

\noindent{\bf Step 1: Estimate $\alpha_i, i \in [n]$.} Given the observation matrix $Z$, let $S \in \R^{n \times n}$ be the 
``empirical stochastic matrix'' whose $(i,j)$th element is given by:
\begin{equation}
\label{Eq: Stochastic Matrix Estimator}
\forall i,j \in [n], \enspace S(i,j) \triangleq 
\begin{cases}
\displaystyle{\frac{1}{2np} Z(i,j)} \, , & i \neq j \, ,  \\
\displaystyle{1 - \frac{1}{2np} \sum_{r = 1}^{n}{Z(i,r)}} \, , & i = j \, .
\end{cases} 
\end{equation}
As shown in \propref{Prop: Estimator Stochastic Matrix} in Appendix \ref{Proof of Prop Estimator Stochastic Matrix}, it is straightforward to verify that $S \in \S_{n \times n}$ (i.e., $S$ is row stochastic) with high probability when $p = \Omega(\log(n)/n)$. Next, inspired by the 
\emph{rank centrality} algorithm in \cite{NegahbanOhShah2012,NegahbanOhShah2017}, let $\hat{\pi}_* \in \S_n$ be the invariant 
distribution of $S$, given by: 
\begin{equation}
\label{Eq: Empirical Stat Dist}
\hat{\pi}_{*} \triangleq \begin{cases}
\text{invariant distribution of } S \text{ such that } \hat{\pi}_* = \hat{\pi}_* S \, , & S \in \S_{n \times n} \, , \\
\text{any randomly chosen distribution in } \S_n \, , & S \notin \S_{n \times n} \, ,
\end{cases}     
\end{equation}
where when $S \in \S_{n \times n}$, an invariant distribution always exists and we choose one arbitrarily 
when it is not unique. Then, we can define the following estimates of $\alpha_1,\dots,\alpha_n$ based on $Z$:
\begin{equation}
\label{Eq: Estimates of merit parameters}
\forall i \in [n], \enspace \hat{\alpha}_i \triangleq \frac{\hat{\pi}_*(i)}{\left\|\hat{\pi}_*\right\|_{\infty}}
\end{equation}
where $\hat{\pi}_*(i)$ denotes the $i$th entry of $\hat{\pi}_*$ for $i \in [n]$.

\noindent{\bf Step 2: Estimate $\pdf$.} Using \eqref{Eq: Estimates of merit parameters}, we construct the \emph{Parzen-Rosenblatt (PR) kernel density estimator} $\ests:\R \rightarrow \R$ for $\pdf$ based on $\hat{\alpha}_1,\dots,\hat{\alpha}_n$ (instead of $\alpha_1,\dots,\alpha_n$) \cite{Rosenblatt1956,Parzen1962}:
\begin{equation}
\label{Eq: Robust PR Estimator}
\forall x \in \R, \enspace \ests(x) \triangleq \frac{1}{n h} \sum_{i = 1}^{n}{K\!\left(\frac{\hat{\alpha}_i - x}{h}\right)}
\end{equation}
where $h > 0$ is a judiciously chosen bandwidth parameter (see the proof in Appendix \ref{Proof of MSE Upper Bound}):
\begin{equation}
\label{Eq: Precise Bandwidth}
h = \gamma \max\!\left\{\frac{1}{\delta^{\frac{1}{\eta + 1}} (p k)^{\frac{1}{2 \eta + 2}}},1\right\} \left(\frac{\log(n)}{n}\right)^{\frac{1}{2 \eta + 2}}
\end{equation}
for any (universal) constant $\gamma > 0$, and $K : [-1,1] \rightarrow \R$ is any fixed kernel function with certain properties that we explain below.

For any $s \in \N\cup\!\{0\}$, the function $K : [-1,1] \rightarrow \R$ is said to be a \emph{kernel of order $s$}, where we assume that $K(x) = 0$ for $|x| > 1$, if $K$ is (Lebesgue) square-integrable, $\int_{\R}{K(x) \diff{x}} = 1$, and $\int_{\R}{x^i K(x) \diff{x}} = 0$ for all $i \in [s]$ when $s \geq 1$. 
Such kernels of order $s$ can be constructed using orthogonal polynomials as expounded in \cite[Section 1.2.2]{Tsybakov2009}. 
We will additionally assume that there exists a constant $L_2 > 0$ such that our kernel $K : [-1,1] \rightarrow \R$ 
is \emph{$L_2$-Lipschitz continuous}, i.e., $|K(x) - K(y)| \leq L_2 |x - y|$ for all $x,y \in \R$. This is a mild assumption 
since several well-known kernels satisfy it. For instance, the (parabolic) \emph{Epanechnikov kernel} 
$K_{\mathsf{E}}(x) \triangleq \frac{3}{4}(1 - x^2) \I\{|x| \leq 1\}$ has order $s = 1$, and is Lipschitz continuous with 
$L_2 = \frac{3}{2}$ \cite{Epanechnikov1969}. Other examples of valid kernels can be found 
in \cite[p.3 and Section 1.2.2]{Tsybakov2009}.

\noindent{\bf Algorithm, in summary.} 
Here, we provide the `pseudo-code' summary of our algorithm. 

\begin{algorithm}[H]
\begin{algorithmic}[1]
\renewcommand{\algorithmicrequire}{\textbf{Input:}}
\Require Observation matrix $Z \in [0,1]^{n \times n}$ (as defined in \eqref{Eq: Observation matrix})
\renewcommand{\algorithmicensure}{\textbf{Output:}}
\Ensure Estimator $\ests : \R \rightarrow \R$ of the unknown PDF $P_{\alpha}$
\Statex \textbf{\emph{Step 1:} Skill parameter estimation using rank centrality algorithm}
\State Construct $S \in \S_{n \times n}$ according to \eqref{Eq: Stochastic Matrix Estimator} using $Z$ (and $p$ and $n$)
\State Compute leading left eigenvector $\hat{\pi}_* \in \S_n$ of $S$ in \eqref{Eq: Empirical Stat Dist} \Comment{$\hat{\pi}_*$ is the invariant distribution of $S$}
\State Compute estimates $\hat{\alpha}_i = \hat{\pi}_*(i)/\|\hat{\pi}_*\|_{\infty}$ for $i = 1,\dots,n$ via \eqref{Eq: Estimates of merit parameters}
\Statex \textbf{\emph{Step 2:} Kernel density estimation using Parzen-Rosenblatt  method} 
\State Compute bandwidth $h$ via \eqref{Eq: Precise Bandwidth} (using $p$, $k$, $\delta$, $\eta$, and $n$)
\State Construct $\ests$ according to \eqref{Eq: Robust PR Estimator} using $\hat{\alpha}_1,\dots,\hat{\alpha}_n$, $h$, and a valid kernel $K : [-1,1] \rightarrow \R$ 
\State \Return $\ests$
\end{algorithmic}
\caption{Estimating skill PDF $\pdf$ using $Z$.}
\label{Algorithm: Density Estimation}
\end{algorithm}

With fixed $\delta \in (0,1)$, $\eta > 0$, and a valid kernel $K : [-1,1] \rightarrow \R$, and given knowledge of $k \in \N$ and $p \in (0,1]$ (which can also be easily estimated), Algorithm \ref{Algorithm: Density Estimation} constructs the estimator \eqref{Eq: Robust PR Estimator} for $\pdf$ based on $Z$. In Algorithm \ref{Algorithm: Density Estimation}, we assume that $S \in \S_{n \times n}$, because this is almost always the case in practice. Furthermore, if $k$ varies between players so that $i$ and $j$ play $k_{i,j} = k_{j,i}$ games for $i \neq j$, we can re-define the data $Z(i,j)$ to use $k_{i,j}$ instead of $k$ in \eqref{Eq: Observation matrix}, and utilize an appropriately altered bandwidth $h$. The computational complexity of Algorithm \ref{Algorithm: Density Estimation} is determined by the running time of rank centrality: if the spectral gap of $S$ is $\Theta(1)$ 
and we use  \emph{power iteration} (cf. \cite[Section 7.3.1]{GolubVanLoan1996}, \cite[Section 4.4.1]{Demmel1997}) to obtain an $O(n^{-5})$ $\ell_2$-approximation of $\hat{\pi}_*$, then Algorithm \ref{Algorithm: Density Estimation} runs in $O(n^2 \log(n))$ time. We refer readers to Appendix \ref{App: Intuition} for further intuition regarding Algorithm \ref{Algorithm: Density Estimation}.

%% file: results.tex
\section{Main results}
\label{Main results and discussion}

We now present our main results: an achievable minimax MSE for the $\pdf$ estimation method in Algorithm \ref{Algorithm: Density Estimation}, and minimax lower bounds on estimation of the skill parameters $\alpha_i, i \in [n]$ from $Z$ (i.e., Step 1 of Algorithm \ref{Algorithm: Density Estimation}) for {\em any} method. This collectively establishes the near-optimality of our proposed method as $\eta \to \infty$, i.e., as the density becomes smooth ($\mathcal{C}^{\infty}$). To this end, we first establish minimax rates for skill parameter estimation, and then derive minimax rates for PDF estimation. 

\noindent{\bf Tight minimax bounds on skill parameter estimation.} To obtain tight $\pdf$ estimation, it is 
essential that we have tight skill parameter estimation. Hence, we show that the parameter estimation step performed in \eqref{Eq: Empirical Stat Dist} has minimax optimal rate. Specifically, we define the ``canonically scaled'' skill parameters $\pi \in \S_n$ with $i$th entry given by:
\begin{equation}
\label{Eq: Canonically scaled merit parameters}
\forall i \in [n], \enspace \pi(i) \triangleq \frac{\alpha_i}{\alpha_1 + \cdots + \alpha_n} \, .
\end{equation}
Building upon \cite[Theorem 3.1]{Chenetal2019}, the ensuing theorem portrays that the \emph{minimax relative $\ell^{\infty}$-risk} of 
estimating \eqref{Eq: Canonically scaled merit parameters} based on $Z$ is 
$\tilde{\Theta}(n^{-1/2})$ (see \tabref{Table: Technical Contributions}). For simplicity, 
we will assume throughout this subsection on skill parameter estimation that $\delta$, $p$, and $k$ are $\Theta(1)$.

\begin{theorem}[Minimax Relative $\ell^{\infty}$-Risk]
\label{Thm: Minimax Relative l^infty-Risk} 
For sufficiently large constants $c_{14},c_{15} > 0$ (which depend on $\delta$, $p$, and $k$), and for all sufficiently large $n \in \N$:
$$ \frac{c_{14}}{\log(n)\sqrt{n}} \leq \inf_{\hat{\pi}}{\sup_{\pdf \in \pdfs}{\E\!\left[\frac{\left\|\hat{\pi} - \pi\right\|_{\infty}}{\left\|\pi\right\|_{\infty}}\right]}} \leq \sup_{\pdf \in \pdfs}{\E\!\left[\frac{\left\|\hat{\pi}_* - \pi\right\|_{\infty}}{\left\|\pi\right\|_{\infty}}\right]} \leq c_{15} \sqrt{\frac{\log(n)}{n}} $$
where the infimum is over all estimators $\hat{\pi} \in \S_n$ of $\pi$ based on $Z$, and $\hat{\pi}_* \in \S_n$ is defined in \eqref{Eq: Empirical Stat Dist}. 
\end{theorem}
The proof of \thmref{Thm: Minimax Relative l^infty-Risk} can be found in Appendix \ref{Proof of Thm Minimax Relative l^infty-Risk}. \thmref{Thm: Minimax Relative l^infty-Risk} states that the rank centrality estimator $\hat{\pi}_*$ achieves an extremal Bayes relative $\ell^{\infty}$-risk of $\tilde{O}(n^{-1/2})$, and no other estimator can achieve a risk that decays faster than $\tilde{\Omega}(n^{-1/2})$. In the same vein, we show that the \emph{minimax (relative) $\ell^{1}$-risk} (or total variation distance risk) of estimating 
\eqref{Eq: Canonically scaled merit parameters} based on $Z$ is also $\tilde{\Theta}(n^{-1/2})$ (see \tabref{Table: Technical Contributions}).
\begin{theorem}[Minimax $\ell^1$-Risk]
\label{Thm: Minimax l^1-Risk} 
For sufficiently large constants $c_{17},c_{18} > 0$ (which depend on $\delta$, $p$, and $k$), and for all sufficiently large $n \in \N$:
$$ \frac{c_{17}}{\log(n)\sqrt{n}} \leq \inf_{\hat{\pi}}{\sup_{\pdf \in \pdfs}{\E\!\left[\left\|\hat{\pi} - \pi\right\|_{1}\right]}} \leq \sup_{\pdf \in \pdfs}{\E\!\left[\left\|\hat{\pi}_* - \pi\right\|_{1}\right]} \leq \frac{c_{18}}{\sqrt{n}} \, . $$
\end{theorem}

\thmref{Thm: Minimax l^1-Risk} is established in Appendix \ref{Proof of Thm Minimax l^1-Risk}. The upper bounds in Theorems \ref{Thm: Minimax Relative l^infty-Risk} and \ref{Thm: Minimax l^1-Risk} follow from \cite[Theorems 3.1 and 5.2]{Chenetal2019} after some calculations, but the \emph{lower bounds are novel contributions}. We prove them by first lower bounding the minimax risks in terms of Bayes risks in order to circumvent an involved analysis of the infinite-dimensional parameter space $\pdfs$. In particular, we \emph{set $\pdf \in \pdfs$ to be the uniform PDF} over $[\delta,1]$, denoted $\unif([\delta,1]) \in \pdfs$. We then lower bound the Bayes risks using a recent generalization of \emph{Fano's method} \cite{Khasminskii1979,IbragimovKhasminskii1977} (cf. \cite{Yu1997,Tsybakov2009}), which was specifically 
developed to produce such lower bounds in the setting where the parameter space is a continuum, e.g., $[\delta,1]$, 
instead of a finite set \cite{Zhang2006,DuchiWainwright2013,ChenGuntuboyinaZhang2016,XuRaginsky2017}; see Appendices \ref{Generalized Fano's method} and \ref{Auxiliary lemmata 2}.

The principal analytical difficulty in executing the generalized Fano's method is in deriving a tight upper bound on the 
\emph{mutual information} between $\pi$ and $Z$, denoted $I(\pi;Z)$ (see \eqref{Eq: MI Def} in Appendix \ref{App: Covering numbers} for a formal definition), 
where the probability law of $\pi$ is defined using $\pdf = \unif([\delta,1])$. The ensuing proposition presents 
our upper bound on $I(\pi,Z)$. 

\begin{proposition}[Covering Number Bound on Mutual Information]
\label{Prop: Upper Bound on MI}
For all $n \geq 2$, we have:
$$ I(\pi;Z) \leq \frac{1}{2} n \log(n) + \frac{(1-\delta)^2}{8 \delta^2} \left( 2 + \delta + \frac{1}{\delta}\right) k p n \, . $$
\end{proposition}


\propref{Prop: Upper Bound on MI} is proved in Appendix \ref{App: Covering numbers}. We note that although standard information inequalities, e.g. \cite[Equation (44)]{ChenGuntuboyinaZhang2016}, typically suffice to obtain minimax rates for various estimation problems, they only produce a sub-optimal estimate $I(\pi;Z) = O(n^2)$ in our problem, as explained at the end of Appendix \ref{App: Covering numbers}, cf. \eqref{Eq: Last Quadratic Upper Bound}. So, to derive the sharper estimate $I(\pi;Z) = O(n \log(n))$ 
in \propref{Prop: Upper Bound on MI}, we execute a careful covering number argument that is inspired by the 
techniques of \cite{YangBarron1999} (also see the distillation in \cite[Lemma 16.1]{Wu2019}).

We make two further remarks. Firstly, it is worth juxtaposing our results with \cite[Theorem 5.2]{Chenetal2019} and \cite[Theorems 2 and 3]{NegahbanOhShah2017}, which state that the minimax relative $\ell^2$-risk of estimating $\pi$ is $\Theta(n^{-1/2})$. This result holds under a worst-case merit parameter value model as opposed to the worst-case prior distribution model of this paper. Secondly, both Theorems \ref{Thm: Minimax Relative l^infty-Risk} and \ref{Thm: Minimax l^1-Risk} hold verbatim if $\pdfs$ is replaced by any set of probability measures with support in $[\delta,1]$ that contains $\unif([\delta,1])$.

\noindent{\bf Tight minimax bound on skill PDF $\pdf$ estimation.} We now state our main result concerning the estimation error for $\pdf$. In particular, we
argue that the MSE risk of our estimation algorithm (see \eqref{Eq: Robust PR Estimator}) scales as $\tilde{O}(n^{-\eta/(\eta+1)})$ for any $\pdf \in \pdfs$. 

\begin{theorem}[MSE Upper Bound]
\label{Thm: MSE Upper Bound}
Fix any sufficiently large constants $c_2,c_3 > 0$ and suppose that $p \geq c_2 \log(n)/(\delta^5 n)$, $b \geq c_3 \sqrt{\log(n)/n}$, $\epsilon \geq 5\log(n)/(bn)$, and $\lim_{n \rightarrow \infty}{\delta^{-1} (n p k)^{-1/2} \log(n)^{1/2}} = 0$. Then, for any $L_2$-Lipschitz continuous kernel $K:[-1,1] \rightarrow \R$ of order $\lceil \eta \rceil - 1$, there exists a sufficiently large constant $c_{12} > 0$ (that depends on $\gamma$, $\eta$, $B$, $L_1$, $L_2$, and $K$) such that for all sufficiently large $n \in \N$:
$$ R_{\mathsf{MSE}}(n) \leq \sup_{\pdf \in \pdfs} \E\!\left[\int_{\R}{\left(\ests(x) - \pdf(x)\right)^{ 2}\diff{x}}\right] \leq c_{12} \max\!\left\{ \left(\frac{1}{\delta^2 p k}\right)^{\!\frac{\eta}{\eta + 1}},1\right\} \! \left(\frac{\log(n)}{n}\right)^{\!\frac{\eta}{\eta + 1}} . $$
\end{theorem}

\thmref{Thm: MSE Upper Bound} is established in Appendix \ref{Proof of MSE Upper Bound}. We next make several pertinent remarks. Firstly, the condition $p \geq c_2 \log(n)/(\delta^5 n)$ is precisely the critical scaling that ensures that $\G(n,p)$ 
is connected almost surely, cf. \cite[Theorem 8.11]{BlumHopcroftKannan2020}, \cite[Section 7.1]{Bollobas2001}. 
This is essential to estimate $\alpha_1,\dots,\alpha_n$ in Step 1 of Algorithm \ref{Algorithm: Density Estimation}, 
since we cannot reasonably compare the skill levels of disconnected players. Secondly, while $\ests$ can be negative, 
the non-negative truncated estimator $\estsp(x) = \max\{\ests(x),0\}$ achieves smaller MSE risk than $\ests$, cf. \cite[p.10]{Tsybakov2009}. So it is easy to construct good non-negative estimators. Thirdly, there exists a constant $c_{13} > 0$ (depending on $\eta,L_1$) such that for all sufficiently large $n \in \N$, the following minimax lower bound holds, cf. \cite[Theorem 6]{Wasserman2019}, \cite[Exercise 2.10]{Tsybakov2009}:
\begin{equation}
\label{Eq: Density Estimation Minimax Lower Bound}
R_{\mathsf{MSE}}(n) \geq \inf_{\hat{P}_{\alpha^n}(\cdot)}\,{ \sup_{\pdf \in \pdfs}{ \E\!\left[\int_{\R}{\left(\hat{P}_{\alpha^n}(x) - \pdf(x) \right)^2 \diff{x}} \right] } } \geq c_{13} \left(\frac{1}{n}\right)^{\frac{2\eta}{2\eta + 1}}
\end{equation}
where the infimum is over all estimators $\hat{P}_{\alpha^n} : \R \rightarrow \R$ of $P_{\alpha}$ based on 
$\alpha_1,\dots,\alpha_n$, and the first inequality holds because the infimum in \eqref{Eq: Minimax MSE Formulation} 
is over a subset of the class of estimators used in the infimum in \eqref{Eq: Density Estimation Minimax Lower Bound}; 
indeed, given $\alpha_1,\dots,\alpha_n$, one can simulate $Z$ via \eqref{Eq: BT model} and estimate $P_{\alpha}$ 
from $Z$. Thus, when $\eta = 1$, \thmref{Thm: MSE Upper Bound} and \eqref{Eq: Density Estimation Minimax Lower Bound} 
show that $R_{\mathsf{MSE}}(n) = \tilde{O}(n^{-1/2})$ and $R_{\mathsf{MSE}}(n) = \Omega(n^{-2/3})$. Likewise, when ($\eta \rightarrow \infty$ and) $P_{\alpha}$ is smooth, i.e., \emph{infinitely differentiable} with all derivatives bounded by $L_1$, \thmref{Thm: MSE Upper Bound} holds for all 
$\eta > 0$, and an $\Omega(n^{-1})$ lower bound analogous to \eqref{Eq: Density Estimation Minimax Lower Bound} holds \cite{IbragimovKhasminskii1982}. Letting $\varepsilon = (\eta + 1)^{-1}$, these results yield the first row of \tabref{Table: Technical Contributions}. Lastly, 
we note that similar analyses to \thmref{Thm: MSE Upper Bound} can be carried out for, e.g., \emph{Nikol'ski} 
and \emph{Sobolev classes} of PDFs, cf. \cite[Section 1.2.3]{Tsybakov2009}. 

We emphasize that the key technical 
step in the proof of \thmref{Thm: MSE Upper Bound} is the ensuing intermediate result. 

\begin{proposition}[MSE Decomposition]
\label{Prop: MSE Decomposition}
Fix any sufficiently large constants $c_2,c_3,c_8,c_9 > 0$ and suppose that $p \geq c_2 \log(n)/(\delta^5 n)$, $b \geq c_3 \sqrt{\log(n)/n}$, $\epsilon \geq 5\log(n)/(bn)$, and $\lim_{n \rightarrow \infty} \delta^{-1} \allowbreak (n p k)^{-1/2} \log(n)^{1/2} = 0$. Then, for any $\pdf \in \pdfs$, any $L_2$-Lipschitz continuous kernel $K:[-1,1] \rightarrow \R$, any bandwidth $h \in (0,1]$ with $h = \Omega\big(\!\max\{1/(\delta \sqrt{p k}),1\} \sqrt{\log(n)/n}\big)$, and any sufficiently large $n \in \N$:
$$ \E\!\left[\int_{\R}{\!\!\left(\!\ests(x) -\pdf(x)\!\right)^{\! 2}\!\!\diff{x}}\right] \!\leq 2 \E\!\left[\int_{\R}{\!\!\left(\!\hat{P}_{\alpha^n}^*(x) - \pdf(x)\!\right)^{\! 2}\!\!\diff{x}}\right] + \frac{c_8 B^2 L_2^2}{h^2} \E\!\left[\max_{i \in [n]}{\left|\hat{\alpha}_i - \alpha_i\right|^{2}}\right] + \frac{c_9 L_2^2}{n^5 h^4} $$
where $\hat{P}_{\alpha^n}^*:\R \rightarrow \R$ denotes the classical PR kernel density estimator of $P_{\alpha}$ based on the true samples $\alpha_1,\dots,\alpha_n$ (if they were made available by an oracle) \cite{Rosenblatt1956,Parzen1962}:
\begin{equation}
\label{Eq: True PR estimator}
\forall x \in \R, \enspace \hat{P}_{\alpha^n}^*(x) \triangleq \frac{1}{n h} \sum_{i = 1}^{n}{K\!\left(\frac{\alpha_i - x}{h}\right)} \, .
\end{equation}
\end{proposition}

The proof of \propref{Prop: MSE Decomposition} can be found in Appendix \ref{App: MSE decomposition}. This result decomposes the MSE between $\ests$ (with general $h$) and $\pdf$ into two dominant terms: the MSE of estimating $\pdf$ 
using \eqref{Eq: True PR estimator}, which can be analyzed using a standard bias-variance tradeoff 
(see \lemref{Lemma: Bias-Variance Tradeoff} in Appendix \ref{Proof of MSE Upper Bound} \cite{Tsybakov2009,Wasserman2019}), 
and the squared $\ell^{\infty}$-risk of estimating $\alpha_1,\dots,\alpha_n$ using \eqref{Eq: Estimates of merit parameters}. 
To analyze the second term, we use a relative $\ell^{\infty}$-norm bound from \cite[Theorem 3.1]{Chenetal2019} 
(see \lemref{Lemma: Relative ell^infty-Loss Bound} in Appendix \ref{App: MSE decomposition}); the same bound was also used to obtain the upper bound in \thmref{Thm: Minimax Relative l^infty-Risk}.

%% file: experiments.tex
\section{Experiments}
\label{Experiments}

We apply our method to several real-world datasets to exhibit its utility. Specifically, Algorithm \ref{Algorithm: Density Estimation} produces estimates of skill distributions. In order to compare skill distributions across different scenarios as well as capture their essence, it is desirable to compute a single \emph{score} that holistically measures the levels of skill in a tournament. 

\noindent{\bf Skill score of $\pdf$.} Intuitively, a delta measure (i.e., all skills are equal) represents a setting where all game outcomes are completely random; there is no role of skill. On the other hand, the uniform PDF $\unif([0,1])$ (assuming $\delta$ is very small) typifies a setting of maximal skill since players are endowed with the broadest variety of skill parameters. We refer readers to \cite{Gettyetal2018} for a related discussion. Propelled by this intuition, 
any distance between $\pdf$ and $\unif([0,1])$ serves as a valid score that is larger when luck plays a greater role in determining the outcomes of games. Therefore, we propose to use the negative \emph{differential entropy} of $\pdf$ as a score to measure skill in a tournament \cite{CoverThomas2006,PolyanskiyWu2017Notes}:
\begin{equation}
\label{Eq: Entropy Score}
- h(\pdf) \triangleq \int_{\R}{\pdf(t) \log(\pdf(t)) \diff{t}} = D(\pdf||\unif{([0,1])}) \, .
\end{equation}
This is a well-defined and finite quantity that is equal to the \emph{Kullback-Leibler (KL) divergence} between $\pdf$ and $\unif{([0,1])}$ (cf. \eqref{Eq: KL definition} in Appendix \ref{App: Covering numbers}). To estimate $-h(\pdf)$ from data, we will use the simple \emph{resubstitution estimator} based on 
$\ests$ and $\hat{\alpha}_1,\dots,\hat{\alpha}_n$ \cite{AhmadLin1976,Beirlantetal1997}.

\noindent{\bf Algorithmic choices.} In all our simulations, we assume that $\eta = 1$, use the Epanechnikov kernel $K_{\mathsf{E}}$, and set the bandwidth to $h = 0.3 n^{-1/4}$; 
indeed, $h$ is typically chosen using ad hoc data-driven techniques in practice \cite[Section 1.4]{Tsybakov2009}. 

\noindent{\bf Data processing.} The data is available in the form of wins, losses, and draws in tournaments. For simplicity, we ignore draws and only utilize wins and losses. To allow for `regularization' in the small data regime,  we apply Laplace smoothing so that between any pair of players, each observed game 
is counted as $20$ games, and $1$ additional win is added for each player; this effectively means that $p = 1$.

\begin{figure}[t]
\begin{subfigure}{.33\linewidth}
\centering
  
\includegraphics[trim = 40mm 85mm 40mm 85mm, width=\linewidth]{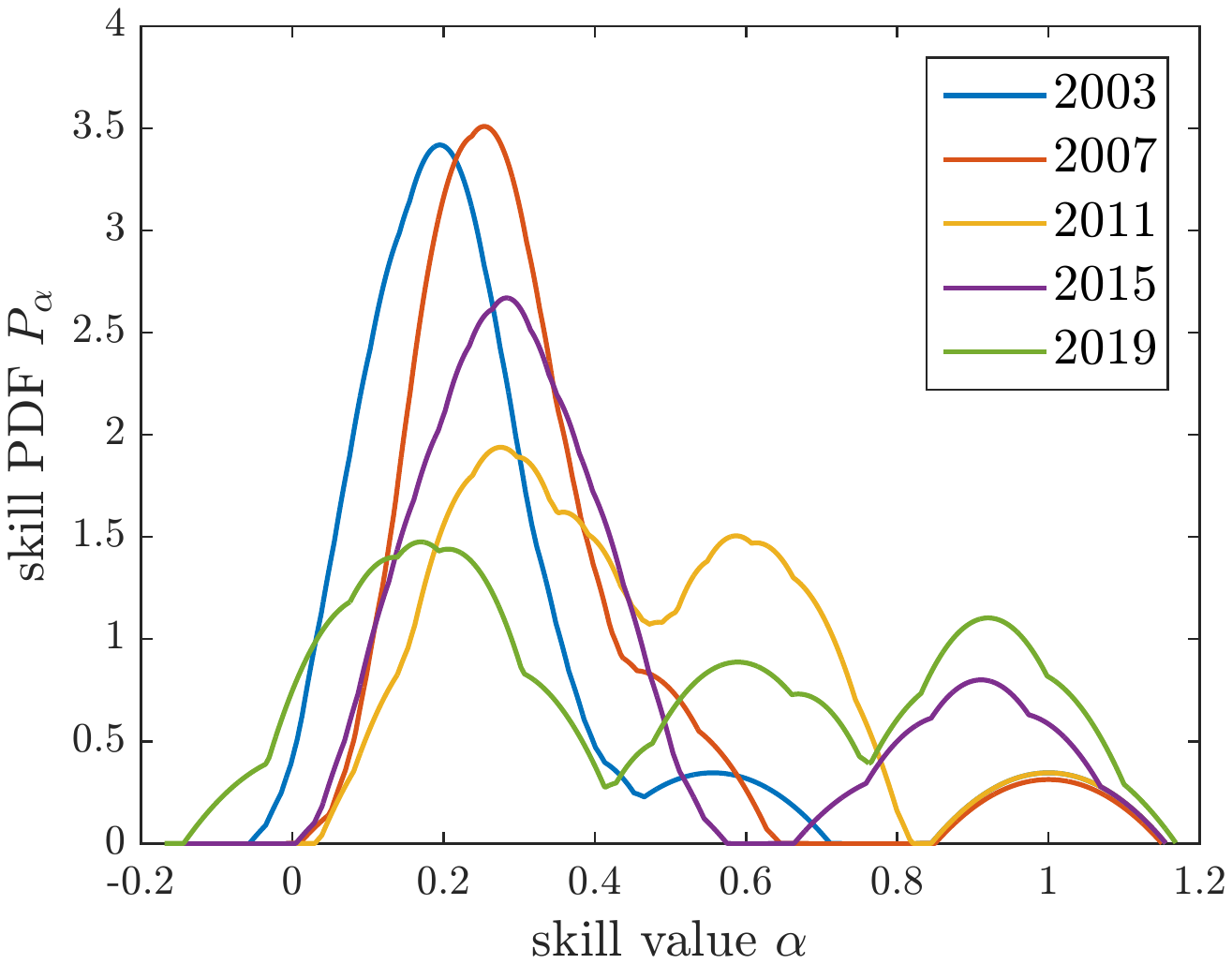}

\caption{ICC Cricket World Cups}  
\label{Fig: ICC Densities}
\end{subfigure}
\begin{subfigure}{.33\linewidth}
\centering
  
\includegraphics[trim = 40mm 85mm 40mm 85mm, width=\linewidth]{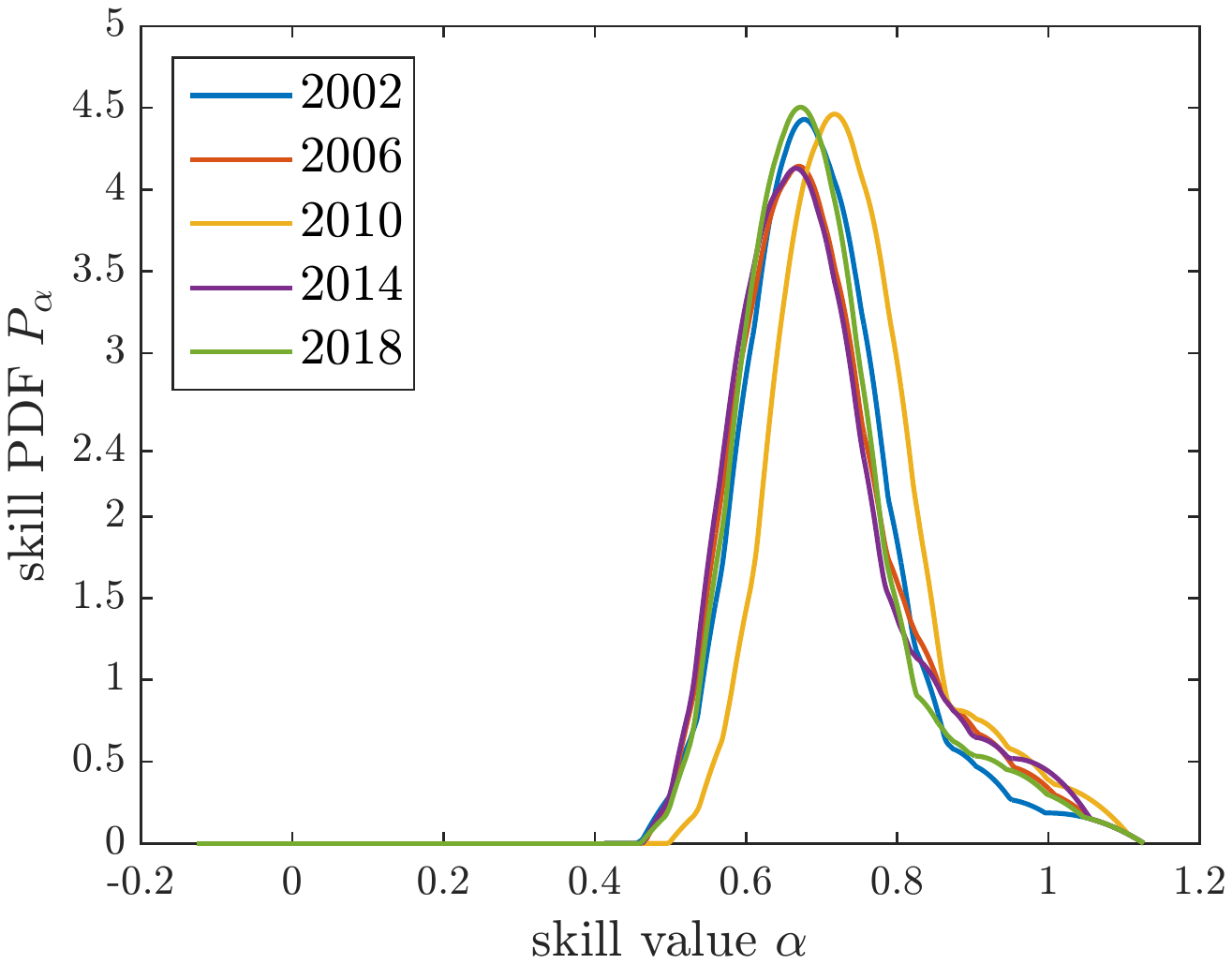}

\caption{FIFA Soccer World Cups}  
\label{Fig: FIFA Densities}
\end{subfigure}
\begin{subfigure}{.33\linewidth}
\centering
  
\includegraphics[trim = 40mm 85mm 40mm 85mm, width=\linewidth]{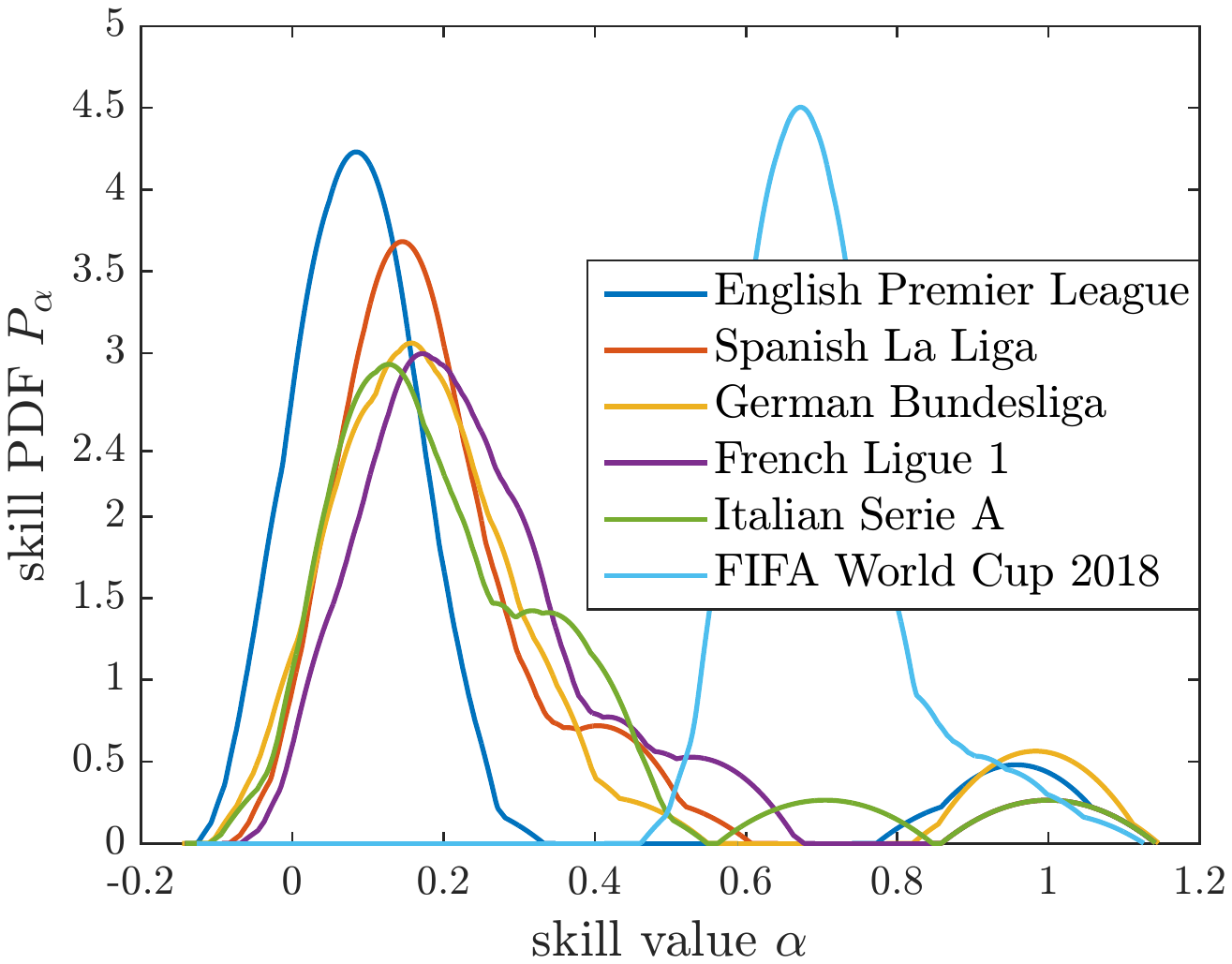}

\caption{Soccer leagues in 2018-2019}  
\label{Fig: EL Densities}
\end{subfigure}
\newline
\begin{subfigure}{.33\linewidth}
\centering
  
\includegraphics[trim = 40mm 85mm 40mm 85mm, width=\linewidth]{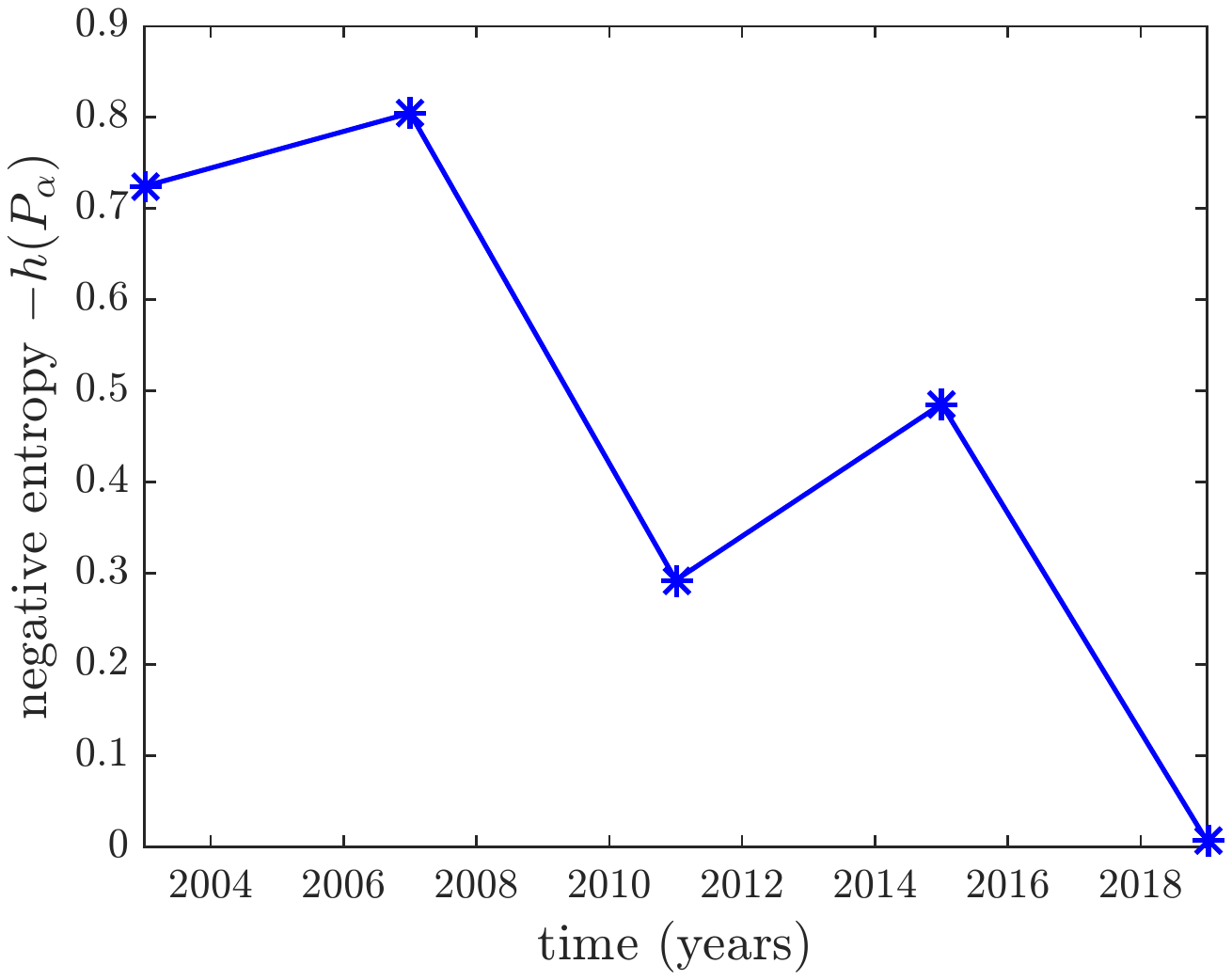}

\caption{ICC Cricket World Cups}  
\label{Fig: ICC Entropies}
\end{subfigure}
\begin{subfigure}{.33\linewidth}
\centering
  
\includegraphics[trim = 40mm 85mm 40mm 85mm, width=\linewidth]{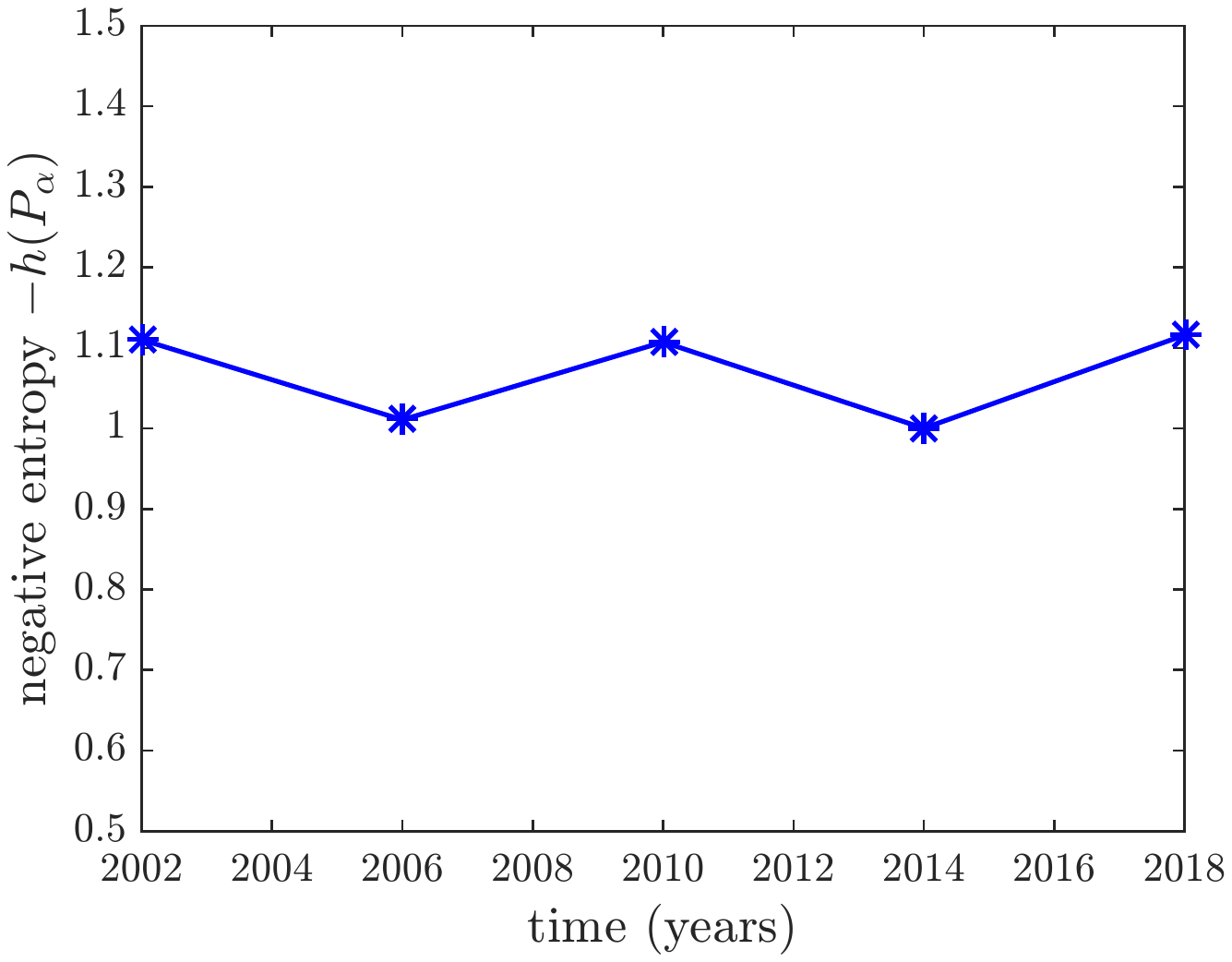}

\caption{FIFA Soccer World Cups}  
\label{Fig: FIFA Entropies}
\end{subfigure}
\begin{subfigure}{.33\linewidth}
\centering
  
\includegraphics[trim = 40mm 85mm 40mm 85mm, width=\linewidth]{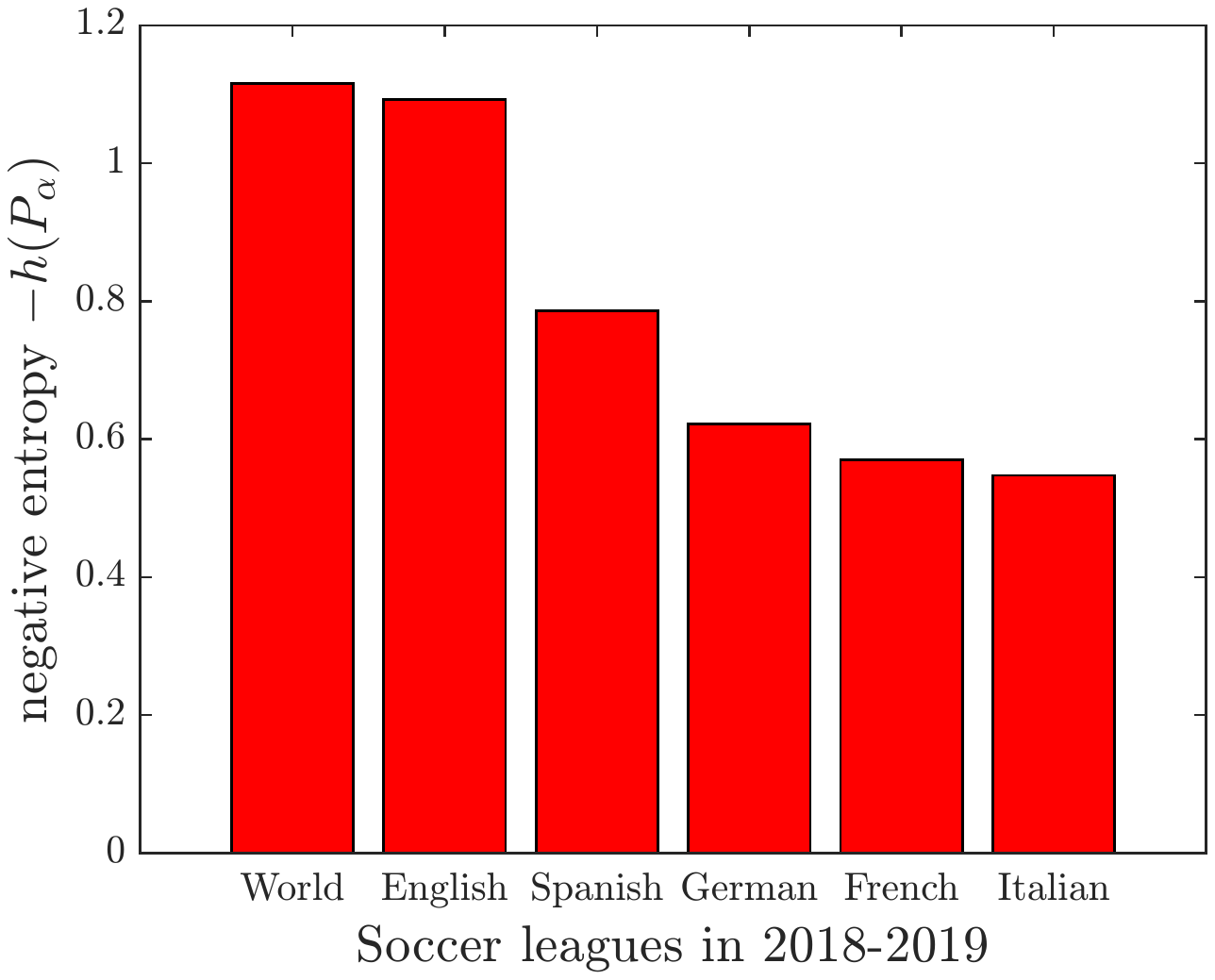}

\caption{Soccer leagues in 2018-2019}  
\label{Fig: EL Entropies}
\end{subfigure}
\newline
\begin{subfigure}{.33\linewidth}
\centering
  
\includegraphics[trim = 40mm 85mm 40mm 85mm, width=\linewidth]{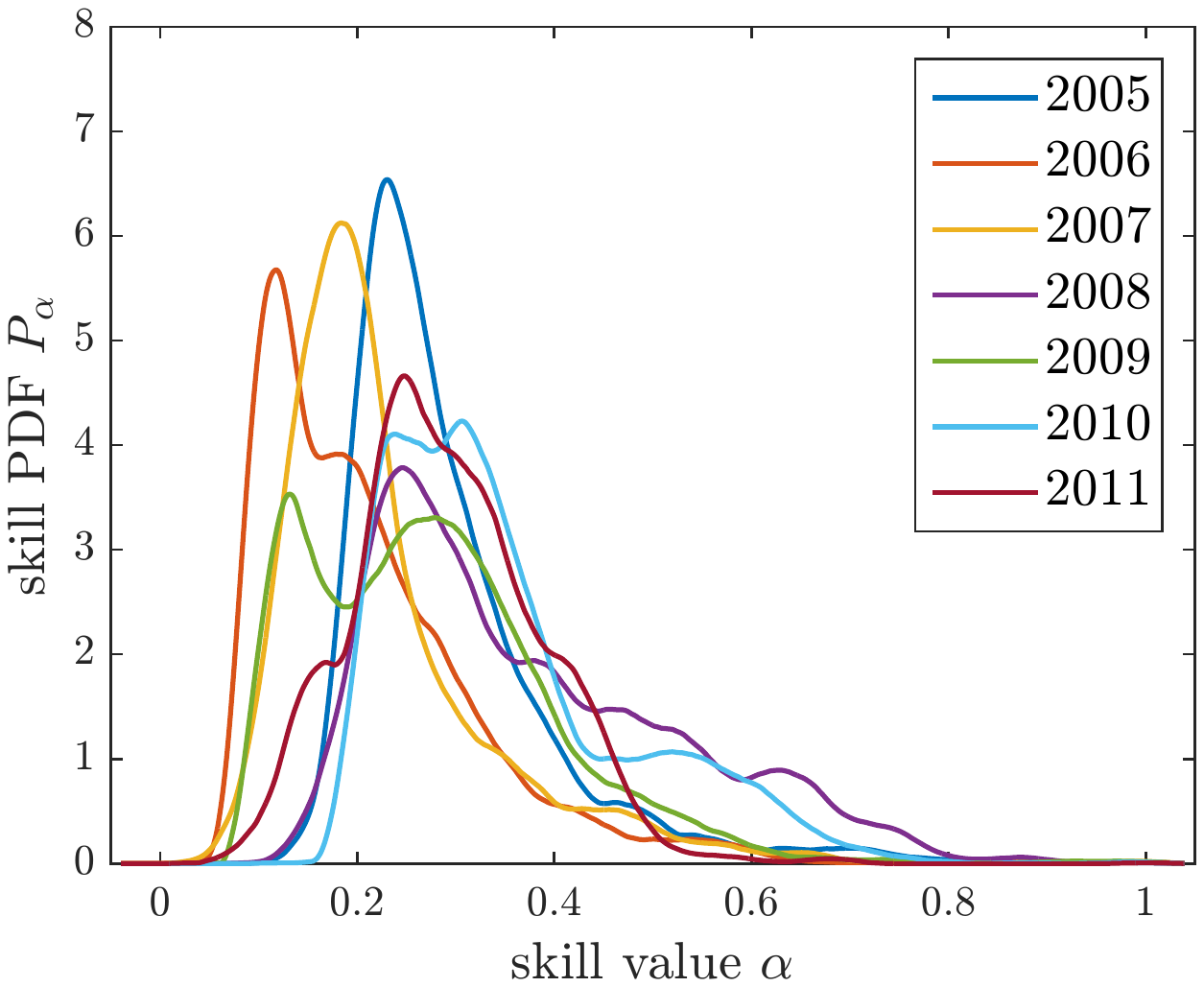}

\caption{US mutual funds (2005-2011)}  
\label{Fig: MF Densities 1}
\end{subfigure}
\begin{subfigure}{.33\linewidth}
\centering
  
\includegraphics[trim = 40mm 85mm 40mm 85mm, width=\linewidth]{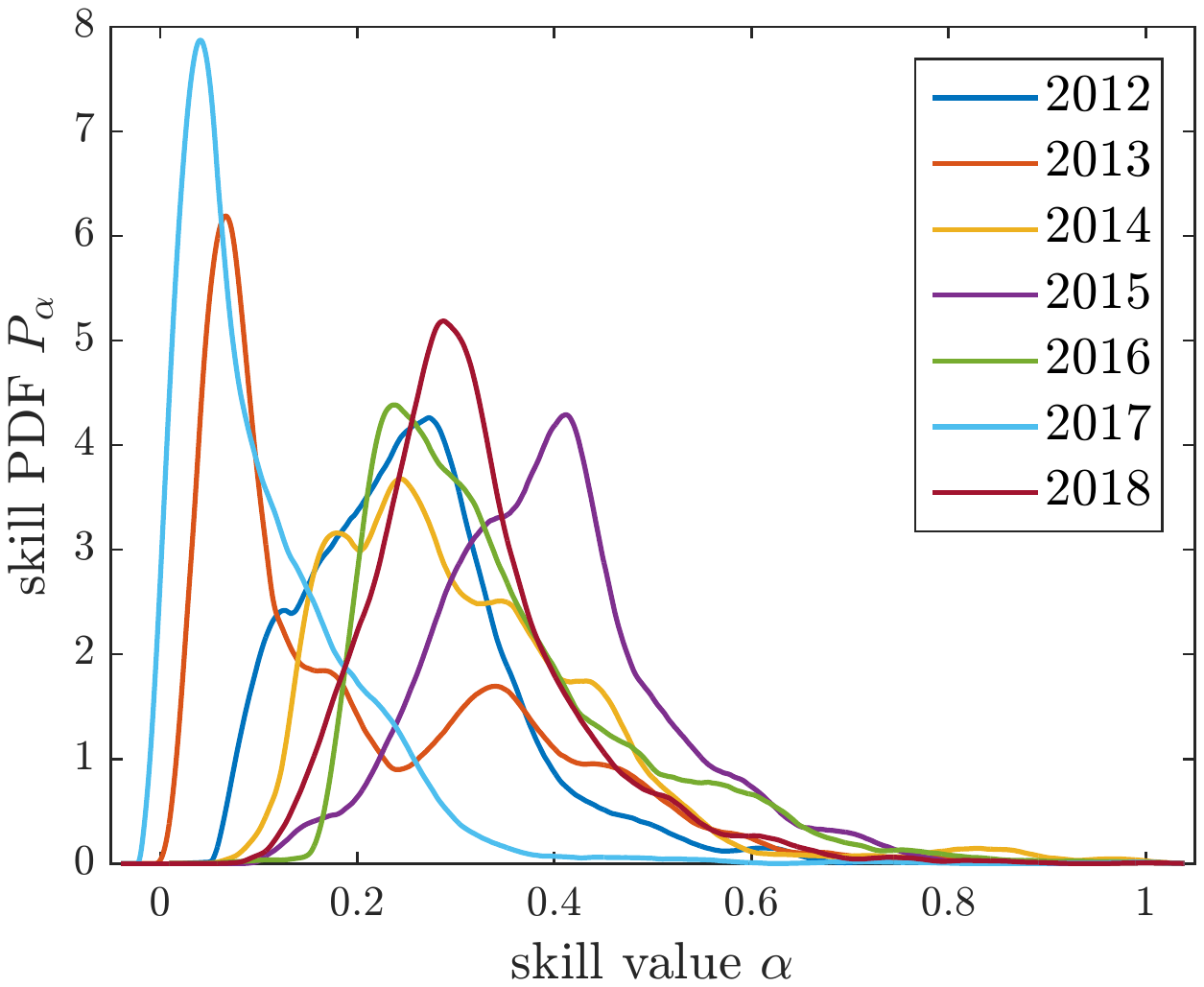}

\caption{US mutual funds (2012-2018)}  
\label{Fig: MF Densities 2}
\end{subfigure}
\begin{subfigure}{.33\linewidth}
\centering
  
\includegraphics[trim = 40mm 85mm 40mm 85mm, width=\linewidth]{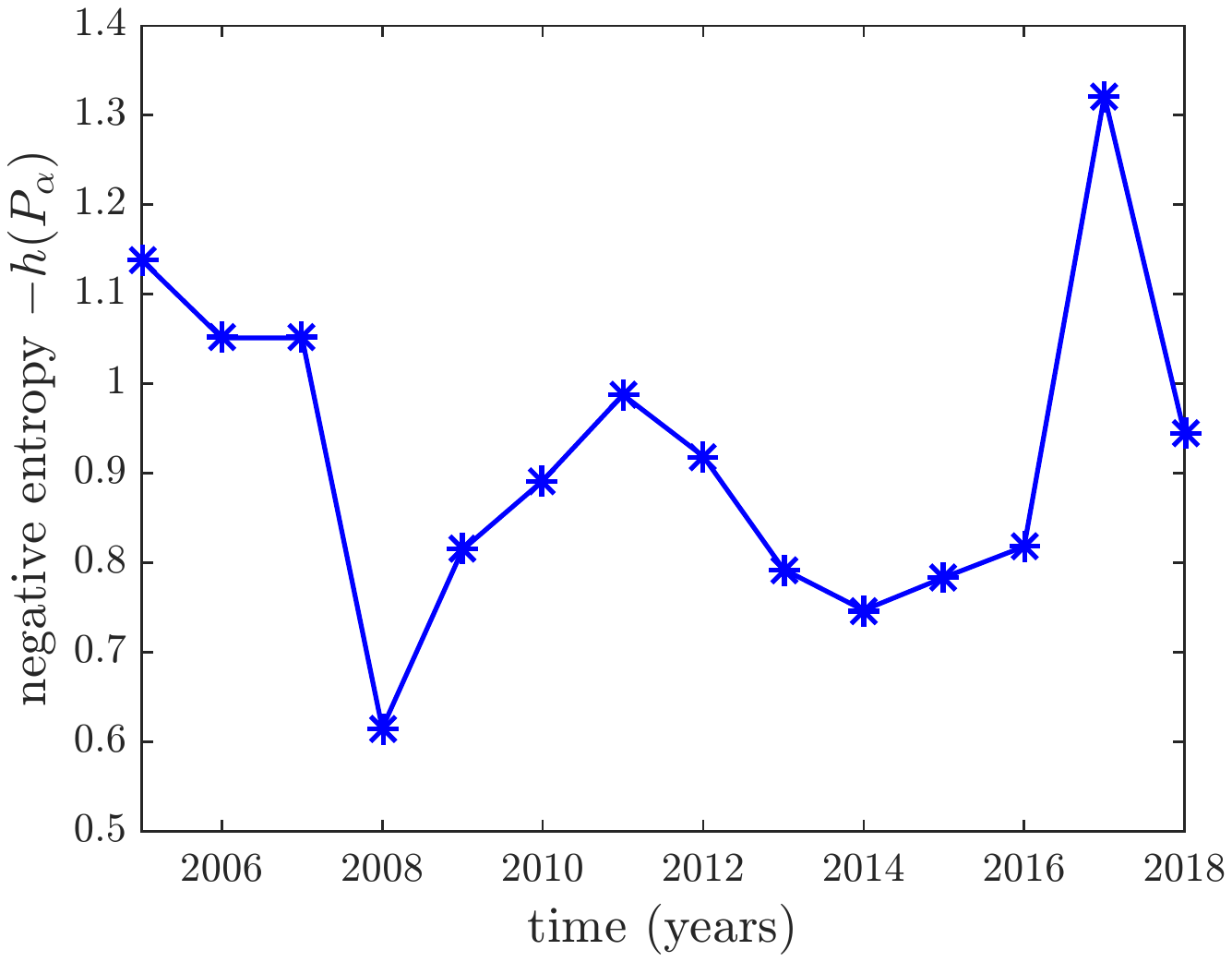}

\caption{US mutual funds (2005-2018)}  
\label{Fig: MF Entropies}
\end{subfigure}

\caption{Plots \ref{Fig: ICC Densities}, \ref{Fig: FIFA Densities}, \ref{Fig: EL Densities}, \ref{Fig: MF Densities 1}, and \ref{Fig: MF Densities 2} illustrate the \emph{estimated PDFs} of skill levels of cricket world cups, soccer world cups, European soccer leagues, and US mutual funds, respectively. Plots \ref{Fig: ICC Entropies}, \ref{Fig: FIFA Entropies}, \ref{Fig: EL Entropies}, and \ref{Fig: MF Entropies} illustrate the corresponding \emph{estimated negative differential entropies} of these PDFs.}
\label{Fig: Plots}
\end{figure}

\noindent\textbf{Cricket world cups.} We utilize publicly available data from Wikipedia for international (ICC) Cricket World Cups held in 2003, 2007, 2011, 2015, and 2019. Each world cup has between $n = 10$ to $n = 16$
teams, with each pair of teams playing $0$, $1$, or (rarely) $2$ matches against one another. We learn the skill distributions for each world cup separately as portrayed in \figref{Fig: ICC Densities}. 
The corresponding negative entropies are reported in \figref{Fig: ICC Entropies}. As can be seen, there is a
clear decrease in negative entropy reaching close to $0$ in 2019. This elegantly quantifies sports intuition about the 2019 World Cup having some of the most thrilling matches in the modern history of cricket \cite{Smythetal2019,Bull2019}.

\noindent\textbf{Soccer world cups.} Again, we use publicly available data from Wikipedia for FIFA Soccer World Cups in 2002, 2006, 2010, 2014, and 2018. Each world cup has $n = 32$ teams, with each pair of teams playing $0$, $1$, or (rarely) $2$ matches. Figures \ref{Fig: FIFA Densities}  and \ref{Fig: FIFA Entropies} depict the skill distributions and associated negative
entropies of soccer world cups over the years. It is evident that the negative entropies have remained roughly constant and away from $0$. This suggests that game outcomes in world cups have remained unpredictable over the years\textemdash very consistent with soccer fan experience.

\noindent\textbf{European soccer leagues.} Yet again,  we use publicly available data from Wikipedia
for the English Premier League (EPL), Spanish La Liga, German Bundesliga, French Ligue 1, and 
Italian Serie A in the 2018-2019 season. Each league has between $n = 18$ to $n = 20$ teams, with every pair of teams playing $0$, $1$, or $2$ times against each other (excluding ties). \figref{Fig: EL Densities} 
illustrates the skill PDFs of these leagues and the 2018 FIFA World Cup. As expected, we observe 
that the skill levels of world cup teams are concentrated in a smaller interval closer to $1$. 
\figref{Fig: EL Entropies} sorts the negative entropies of the skill PDFs and recovers an intuitively 
sound ranking of these leagues. Indeed, many fans believe that EPL has better ``quality'' teams 
than other leagues \cite{McIntyre2019,Spacey2020}, and this observation is confirmed by \figref{Fig: EL Densities}. \figref{Fig: EL Densities} reveals that EPL has higher negative entropy 
than other leagues since its skill PDF has the tallest and narrowest peak, presumably because 
EPL only contains high quality teams with little variation among them. This example shows how our algorithm can be used to compare 
different leagues within the same sport (or even different sports).

\noindent\textbf{US mutual funds.} Our final experiments are calculated based on data obtained through
\cite{WRDS} from \emph{CRSP US Survivor-Bias-Free Mutual Funds Database} that is made
available by the Center for Research in Security Prices (CRSP), The University of Chicago Booth School of Business. 
We consider $n = 3260$ mutual funds in this dataset that have monthly net asset values recorded from January 2005 to December 2018. These values are pre-processed by computing monthly returns (i.e., change in net asset value normalized by the previous month's value) for all funds, which provide a fair measure of monthly performance. Then, we perceive each year as a tournament where each fund plays $k = 12$ monthly games against every other fund, and one fund beats another in a month if it has a larger monthly return. Figures \ref{Fig: MF Densities 1} and \ref{Fig: MF Densities 2} depict the skill PDFs obtained by applying our algorithm to the win-loss data (produced by the method above) every year, and \figref{Fig: MF Entropies} presents the associated negative entropies. Clearly, 2017 and the Great Recession in 2008 were the times where negative entropy was maximized and minimized, respectively, in \figref{Fig: MF Entropies}. Figures \ref{Fig: MF Densities 1} and \ref{Fig: MF Densities 2} unveil that the skill PDF is much more spread out in 2008 compared to 2017, which contains a large peak near $0$. So, as expected, far fewer lowly skilled funds existed during the economic recession in 2008. These observations elucidate the utility of our algorithm in identifying and explaining trends in other kinds of data, such as financial data.

%% file: conc.tex
\section{Conclusion}

In this paper, we proposed an efficient and minimax near-optimal algorithm to learn skill distributions from win-loss data of tournaments. Then, using negative entropy of a learnt distribution as a skill score, we demonstrated the utility of our algorithm in rigorously discerning trends in sports and other data. In closing, we suggest that a worthwhile future direction would be to develop minimax optimal algorithms that directly estimate entropy, or other meaningful skill scores, from tournament data.

%% file: propositions.tex
\section{Proofs of propositions}
\label{Proofs of propositions}

In this appendix, we first present an auxiliary proposition mentioned in the main discussion, and then prove Propositions  \ref{Prop: Upper Bound on MI} and \ref{Prop: MSE Decomposition}.

\subsection{Auxiliary proposition}
\label{Proof of Prop Estimator Stochastic Matrix}

The ensuing proposition shows that the ``empirical stochastic matrix'' $S$ defined in \eqref{Eq: Stochastic Matrix Estimator} (and used in Algorithm \ref{Algorithm: Density Estimation}) is indeed row stochastic with high probability. 

\begin{proposition}[Empirical Stochastic Matrix]
\label{Prop: Estimator Stochastic Matrix}
If $n \geq 2$ and $p \geq 16(c_1 + 1) \log(n) / (3 n)$ for any fixed (universal) constant $c_1 > 0$, then we have:
$$ \P\!\left(S \in \S_{n \times n}\right) \geq 1 - \frac{1}{n^{c_1}} \, . $$
\end{proposition}

\begin{proof}
First, define the random variables:
$$ \forall m \in [n], \enspace M_m \triangleq \sum_{r \in [n]\backslash\!\{m\}}{\I\!\left\{\{m,r\} \in \G(n,p)\right\}} $$
which count the numbers of outcomes of games observed for each player. Furthermore, define the event:
$$ A \triangleq \left\{\forall m \in [n], \, M_m \leq 2 (n-1) p \right\} . $$
Then, we can verify that $A \subseteq \{S \in \S_{n \times n}\}$. Indeed, if $A$ occurs, then we have:
\begin{align*}
& \forall m \in [n], \enspace \frac{1}{2 n p} \sum_{r \in [n]\backslash\!\{m\}}{\I\!\left\{\{m,r\} \in \G(n,p)\right\}} \leq \frac{M_m}{2 (n-1) p} \leq 1 \\
\Rightarrow \quad & \forall m \in [n], \enspace \frac{1}{2 n p} \sum_{r \in [n]\backslash\!\{m\}}{Z(m,r)} \leq 1 \\
\Leftrightarrow \quad & \forall m \in [n], \enspace S(m,m) \geq 0 \\
\Leftrightarrow \quad & S \in \S_{n \times n}
\end{align*}
where the second line follows from \eqref{Eq: Observation matrix}, and the third and fourth lines follow from \eqref{Eq: Stochastic Matrix Estimator}. Hence, it suffices to prove that $\P(A^{c}) \leq n^{-c_1}$. To this end, notice that:
\begin{align*}
\P\!\left(A^{c}\right) & = \P\!\left(\exists \, m \in [n], \, M_m > 2 (n-1)p \right) \\
& \leq n \, \P\!\left(\frac{1}{n-1} \sum_{r = 2}^{n}{\I\!\left\{\{1,r\} \in \G(n,p)\right\} - p} \geq p\right) \\
& \leq n \exp\!\left(-\frac{3 (n-1) p}{8}\right) \\
& \leq n \exp\!\left(-\frac{2 (c_1 + 1) (n-1) \log(n)}{n}\right) \\
& \leq \frac{1}{n^{c_1}}
\end{align*} 
where the second inequality follows from the union bound and the fact that $\{M_m : m \in [n]\}$ are identically distributed random variables, the third inequality follows from \lemref{Lemma: Bernstein's Inequality} in Appendix \ref{App: Concentration of Measure Inequalities} since each indicator random variable $\I\{\{1,r\} \in \G(n,p)\}$ has mean $p$, variance $p(1-p) \leq p$, and satisfies $\big|\I\{\{1,r\} \in \G(n,p)\} - p\big| \leq 1$, the fourth inequality follows from our assumption that $p \geq 16(c_1 + 1) \log(n) / (3 n)$, and the fifth inequality holds because $n \geq 2$. This completes the proof.
\end{proof}

\subsection{Mutual information, covering numbers, and the proof of \propref{Prop: Upper Bound on MI}}
\label{App: Covering numbers}

The goal of this subsection is to establish the upper bound on $I(\pi;Z)$ in \propref{Prop: Upper Bound on MI}. To this end, we commence by presenting some basic definitions and properties of mutual information from information theory. Recall that for any two probability measures $\mu$ and $\nu$ over the same measurable space $(\Omega,\F)$, the KL divergence (or relative entropy) of $\nu$ from $\mu$ is defined as (cf. \cite[Definition 1.4]{PolyanskiyWu2017Notes}):
\begin{equation}
\label{Eq: KL definition}
D(\mu||\nu) \triangleq 
\begin{cases}
\displaystyle{\int_{\Omega}{\log\!\left(\frac{\diff \mu}{\diff \nu}\right) \diff\mu}} \, , & \mu \text{ is absolutely continuous with respect to } \nu \\
+\infty \, , & \text{otherwise}
\end{cases}
\end{equation}
where $\frac{\diff \mu}{\diff \nu}$ is the Radon-Nikodym derivative (or density) of $\mu$ with respect to $\nu$. Using \eqref{Eq: KL definition}, for any pair of jointly distributed random variables $X$ and $Y$, we define the \emph{mutual information} between $X$ and $Y$ as (cf. \cite[Definition 2.3]{PolyanskiyWu2017Notes}):
\begin{equation}
\label{Eq: MI Def}
I(X;Y) \triangleq D(P_{X,Y}||P_X \otimes P_Y)
\end{equation}
where $P_{X,Y}$ denotes the joint probability law of $X$ and $Y$, and $P_X \otimes P_Y$ denotes the product measure of the corresponding marginal probability laws of $X$ and $Y$, respectively. Note that in \eqref{Eq: MI Def}, the random variables $X$ or $Y$ can be compound variables. So, for example, if $Y = (Y_1,Y_2)$, i.e., $Y$ is actually a pair of random variables $Y_1,Y_2$ (where $X,Y_1,Y_2$ are jointly distributed), then we can write $I(X;Y) = I(X;Y_1,Y_2)$. Furthermore, given three jointly distributed random variables $X,Y,W$, where $W \in \mathcal{W}$ is a discrete random variable whose probability distribution $P_W$ has support $\mathcal{W}$, we define the mutual information between $X$ and $Y$ given $W = w$ as:
\begin{equation}
\label{Eq: Cond MI Pre-Def}
\forall w \in \mathcal{W}, \enspace I(X;Y|W = w) \triangleq D(P_{X,Y|W = w}||P_{X|W = w} \otimes P_{Y|W = w})
\end{equation}
where $P_{X,Y|W = w}$ denotes the conditional (joint) probability law of $X$ and $Y$ given $W = w$, and $P_{X|W = w} \otimes P_{Y|W = w}$ denotes the product measure of the conditional (marginal) probability laws of $X$ and $Y$ given $W = w$, respectively. Then, using \eqref{Eq: Cond MI Pre-Def}, the \emph{conditional mutual information} between $X$ and $Y$ given $W$ is defined as (cf. \cite[Definition 2.4]{PolyanskiyWu2017Notes}):
\begin{equation}
\label{Eq: Cond MI Def}
I(X;Y|W) \triangleq \sum_{w \in \mathcal{W}}{P_W(w) I(X;Y|W = w)}
\end{equation}
which the expected value of $I(X;Y|W = w)$ with respect to $P_W$. We will utilize the following well-known properties of mutual information in the sequel.

\begin{lemma}[Properties of Mutual Information {\cite{PolyanskiyWu2017Notes,CoverThomas2006}}]
\label{Lemma: Properties of Mutual Information}
For any three jointly distributed random variables $X,Y,W$, the following results hold:
\begin{enumerate}
\item \emph{(Chain rule} \cite[Theorem 2.5]{PolyanskiyWu2017Notes}, \cite[Theorem 2.5.2]{CoverThomas2006}\emph{)} If $W \in \mathcal{W}$ is a discrete random variable whose probability distribution $P_W$ has support $\mathcal{W}$, then:
$$ I(X;Y,W) = I(X;W) + I(X;Y|W) \, . $$
\item \emph{(Data processing inequality} \cite[Theorem 2.5]{PolyanskiyWu2017Notes}, \cite[Theorem 2.8.1]{CoverThomas2006}\emph{)} If $X \rightarrow Y \rightarrow W$ forms a Markov chain, i.e., $X$ and $W$ are conditionally independent given $Y$, then: 
$$ I(X;W) \leq I(Y;W) \, . $$
\end{enumerate}
\end{lemma}

Another set of ideas we will exploit to prove \propref{Prop: Upper Bound on MI} concerns a powerful and general approach to upper bound mutual information (or more generally, Shannon capacity, cf. \cite[Sections 4.4 and 4.5]{PolyanskiyWu2017Notes}, \cite[Chapter 7]{CoverThomas2006}) via covering arguments. While there are several variants of such arguments in the literature, in this paper, we will resort to the classical covering argument of \cite{YangBarron1999}. (We refer readers to \cite[Section 5]{ChenGuntuboyinaZhang2016} for generalizations of such covering arguments for a class of $f$-informativities \cite{Csiszar1972}.)

To present the technique in \cite{YangBarron1999}, let us condition on any fixed realization of the underlying Erd{\H o}s-R\'{e}nyi random graph $\G(n,p) = G$. Then, the partial observations $Z$, defined in \eqref{Eq: Observation matrix}, can be equivalently represented using the (compound) random variable:
\begin{equation}
\label{Eq: Compound rv}
Z_G \triangleq \left\{Z(i,j) : \{i,j\} \in G, \, i < j \right\}
\end{equation}
where the notation $\{i,j\} \in G$ shows that the undirected edge $\{i,j\}$ exists in the graph $G$, and each $Z(i,j) = \frac{1}{k}\sum_{m = 1}^{k}{Z_m(i,j)}$ given $\{i,j\} \in G$ (where the $Z_m(i,j)$'s are defined via \eqref{Eq: BT model}). Moreover, using \eqref{Eq: Compound rv}, we have the relation:
\begin{equation}
\label{Eq: Simplified MI Relation}
I(\alpha_1,\dots,\alpha_n;Z|\G(n,p) = G) = I(\alpha_1,\dots,\alpha_n;Z_G)
\end{equation}
where we use \eqref{Eq: Cond MI Pre-Def}, and the fact that the random variables $\alpha_1,\dots,\alpha_n,Z_G$ are independent of $\G(n,p)$. (Note that if $G$ contains no edges, then $Z_G$ is a deterministic quantity and $I(\alpha_1,\dots,\alpha_n;Z_G) = 0$.) Now consider the $n$-dimensional hypercube $[\delta,1]^n$ in which the skill parameter random variables $(\alpha_1,\dots,\alpha_n)$ take values. For any parameter vector (realization) $\beta = (\beta_1,\dots,\beta_n) \in [\delta,1]^n$, let $P_{Z_G|\beta}$ denote the conditional probability distribution of $Z_G$ given $\alpha_i = \beta_i$ for all $i \in [n]$ (with abuse of notation); see \eqref{Eq: BT model} and \eqref{Eq: Observation matrix}. Then, for any $\varepsilon > 0$, we define an \emph{$\varepsilon$-covering} of $[\delta,1]^n$ with finite cardinality $M \in \N$ to be a subset of parameter vectors $\big\{\beta^{(1)},\dots,\beta^{(M)}\big\} \subset [\delta,1]^n$ that satisfies:
\begin{equation}
\forall \beta \in [\delta,1]^n, \, \exists \, i \in [M], \enspace D\!\left(P_{Z_G|\beta}\big|\big| P_{Z_G|\beta^{(i)}}\right) \leq \varepsilon
\end{equation}
where we use KL divergence as our ``distance'' measure. Furthermore, for every $\varepsilon > 0$, we define the \emph{$\varepsilon$-covering number} as:
\begin{equation}
M^{*}(\varepsilon) \triangleq \min\!\left\{M \in \N : \exists \, \varepsilon\text{-covering of } [\delta,1]^n \text{ with cardinality } M \right\} .
\end{equation}
The next lemma distills the upper bound on mutual information via covering numbers presented in \cite[Equation (2), p.1571]{YangBarron1999}, and specializes it to our setting of \eqref{Eq: Simplified MI Relation}.

\begin{lemma}[Covering Number Bound {\cite[Lemma 16.1]{Wu2019}}]
\label{Lemma: Covering Number Bound}
Let $(\alpha_1,\dots,\alpha_n)$ be i.i.d. with
distribution $\pdf= \unif([\delta,1])$, and recall that the conditional probability distribution of $Z_G$ 
given $(\alpha_1,\dots,\alpha_n)$ is defined by \eqref{Eq: BT model} and 
\eqref{Eq: Observation matrix}. Then, the mutual information between 
$(\alpha_1,\dots,\alpha_n)$ and $Z_G$ is upper bounded by:
$$ I(\alpha_1,\dots,\alpha_n;Z_G) \leq \inf_{\varepsilon > 0}{\, \varepsilon + \log(M^{*}(\varepsilon))} \, . $$
\end{lemma}  
Using Lemmata \ref{Lemma: Properties of Mutual Information} and \ref{Lemma: Covering Number Bound}, we can finally prove the upper bound on $I(\pi,Z)$ in \propref{Prop: Upper Bound on MI}.
\begin{proof}[Proof of \propref{Prop: Upper Bound on MI}]
First, notice that $\pi \rightarrow (\alpha_1,\dots,\alpha_n) \rightarrow Z$ forms a Markov chain, because $\pi$ is a deterministic function of $(\alpha_1,\dots,\alpha_n)$ (according to \eqref{Eq: Canonically scaled merit parameters}). Hence, by part 2 of \lemref{Lemma: Properties of Mutual Information}, we get:
\begin{equation}
\label{Eq: DPI Consequence 1}
I(\pi;Z) \leq I(\alpha_1,\dots,\alpha_n;Z) \, .
\end{equation}
Furthermore, notice that $Z \rightarrow (Z,\G(n,p)) \rightarrow (\alpha_1,\dots,\alpha_n)$ forms a Markov chain, because $Z$ is a deterministic (projection) function of $(Z,\G(n,p))$. Thus, by part 2 of \lemref{Lemma: Properties of Mutual Information}, we also get:
\begin{align}
I(\alpha_1,\dots,\alpha_n;Z) & \leq I(\alpha_1,\dots,\alpha_n;Z,\G(n,p)) \nonumber \\
& = I(\alpha_1,\dots,\alpha_n;\G(n,p)) + I(\alpha_1,\dots,\alpha_n;Z|\G(n,p)) \nonumber \\
& = I(\alpha_1,\dots,\alpha_n;Z|\G(n,p))
\label{Eq: DPI Consequence 2}
\end{align}
where the second equality utilizes part 1 of \lemref{Lemma: Properties of Mutual Information}, and the third equality holds because $(\alpha_1,\dots,\alpha_n)$ and $\G(n,p)$ are independent, which implies that $I(\alpha_1,\dots,\alpha_n;\G(n,p)) = 0$ (see \eqref{Eq: MI Def}). Combining \eqref{Eq: DPI Consequence 1} and \eqref{Eq: DPI Consequence 2}, we obtain:
\begin{equation}
\label{Eq: Intermediate MI Bound}
I(\pi;Z) \leq I(\alpha_1,\dots,\alpha_n;Z|\G(n,p)) \, .
\end{equation}
So, it suffices to upper bound the conditional mutual information $I(\alpha_1,\dots,\alpha_n;Z|\G(n,p))$.

To this end, we condition on any fixed realization of the underlying Erd{\H o}s-R\'{e}nyi random graph $\G(n,p) = G$ (as in the discussion preceding this proof), and proceed to establishing an upper bound on $I(\alpha_1,\dots,\alpha_n;Z_G)$ by employing \lemref{Lemma: Covering Number Bound}. Specifically, we next evaluate the right hand side of the inequality in \lemref{Lemma: Covering Number Bound} above for a judiciously chosen $\varepsilon$-covering of $[\delta,1]^n$. Fix any $q > 0$ (to be chosen later), and quantize the interval $[\delta,1]$ using the set of values:
$$ \Q \triangleq \left\{\delta + \frac{(1-\delta) m}{n^{q}} : m \in \big[\lfloor n^q \rfloor\big] \right\} $$
which has cardinality $|\Q| = \lfloor n^q \rfloor \leq n^q$, and satisfies the condition:
\begin{equation}
\label{Eq: Quantization}
\forall t \in [\delta,1], \enspace \min_{s \in \Q}{|t - s|} \leq \frac{1-\delta}{n^q} 
\end{equation}
where the right hand side can be improved to $(1-\delta)/(2 n^q)$ when $t$ is not located at the edges of the interval $[\delta,1]$, and $\lfloor\cdot \rfloor$ denotes the floor function. The next claim shows that $\Q^n$ is actually an $\varepsilon$-covering of $[\delta,1]^n$ with $\varepsilon = O(n^{2-2q})$ (neglecting the dependence of $\varepsilon$ on $\delta$ and $k$). 

\begin{claim}[$\varepsilon$-Covering]
\label{Claim: Covering}
$\Q^n$ is an $\varepsilon$-covering of $[\delta,1]^n$ with cardinality $|\Q^n| = (\lfloor n^q \rfloor)^n \leq n^{qn}$ and:
$$ \varepsilon = \frac{(1-\delta)^2}{4 \delta^2} \left( 2 + \delta + \frac{1}{\delta}\right) \frac{k |G|}{n^{2 q}} $$
where $|G|$ denotes the number of edges in the graph $G$ with abuse of notation. (Since $|G| \leq \frac{n(n-1)}{2}$, $\varepsilon = O(n^{2-2q})$.)
\end{claim}

\begin{proof}
The cardinality of $\Q^n$ follows since it is a product set. So, we focus on verifying the value of $\varepsilon$. Fix any parameter vector $\beta = (\beta_1,\dots,\beta_n) \in [\delta,1]^n$. Due to \eqref{Eq: Quantization}, we have that:
\begin{equation}
\label{Eq: Closeness of Parameters}
\forall i \in [n], \, \exists \gamma_i \in \Q, \enspace |\beta_i - \gamma_i| \leq \frac{1-\delta}{n^q} \, .
\end{equation}
Letting $\gamma = (\gamma_1,\dots,\gamma_n) \in \Q^n$, observe that:
\begin{align}
D\!\left(P_{Z_G|\beta} \big|\big| P_{Z_G|\gamma}\right) & = \sum_{\substack{i,j \in [n] : \\ \{i,j\} \in G, \, i < j}}{D\big(P_{Z(i,j)|\beta,G} \big|\big| P_{Z(i,j)|\gamma,G}\big)} \nonumber \\
& = \sum_{\substack{i,j \in [n] : \\ \{i,j\} \in G, \, i < j}}{\! \sum_{i = 0}^{k}{ \binom{k}{i} \! \left(\frac{\beta_j}{\beta_i + \beta_j}\right)^{\!\! i} \! \left(\frac{\beta_i}{\beta_i + \beta_j}\right)^{\!\! k-i} \!\log\!\left(\frac{\left(\frac{\beta_j}{\beta_i + \beta_j}\right)^{\! i} \! \left(\frac{\beta_i}{\beta_i + \beta_j}\right)^{\! k-i}}{\left(\frac{\gamma_j}{\gamma_i + \gamma_j}\right)^{\! i} \! \left(\frac{\gamma_i}{\gamma_i + \gamma_j}\right)^{\! k-i}}\right) } } \nonumber \\
& = \sum_{\substack{i,j \in [n] : \\ \{i,j\} \in G, \, i < j}} \log\!\left(\frac{\left(\frac{\beta_j}{\beta_i + \beta_j}\right)}{\left(\frac{\gamma_j}{\gamma_i + \gamma_j}\right)}\right) \sum_{i = 0}^{k}{i \binom{k}{i} \! \left(\frac{\beta_j}{\beta_i + \beta_j}\right)^{\!\! i} \! \left(\frac{\beta_i}{\beta_i + \beta_j}\right)^{\!\! k-i}} \nonumber \\
& \qquad \qquad \quad \enspace \,\, + \log\!\left(\frac{\left(\frac{\beta_i}{\beta_i + \beta_j}\right)}{\left(\frac{\gamma_i}{\gamma_i + \gamma_j}\right)}\right) \sum_{i = 0}^{k}{(k - i) \binom{k}{i} \! \left(\frac{\beta_j}{\beta_i + \beta_j}\right)^{\!\! i} \! \left(\frac{\beta_i}{\beta_i + \beta_j}\right)^{\!\! k-i}} \nonumber \\
& = k \sum_{\substack{i,j \in [n] : \\ \{i,j\} \in G, \, i < j}}{ \left(\frac{\beta_j}{\beta_i + \beta_j}\right) \log\!\left(\frac{\left(\frac{\beta_j}{\beta_i + \beta_j}\right)}{\left(\frac{\gamma_j}{\gamma_i + \gamma_j}\right)}\right) + \left(\frac{\beta_i}{\beta_i + \beta_j}\right) \log\!\left(\frac{\left(\frac{\beta_i}{\beta_i + \beta_j}\right)}{\left(\frac{\gamma_i}{\gamma_i + \gamma_j}\right)}\right)}  \nonumber \\
& = k \sum_{\substack{i,j \in [n] : \\ \{i,j\} \in G, \, i < j}}{D\!\left(\frac{\beta_j}{\beta_i + \beta_j} \bigg|\bigg| \frac{\gamma_j}{\gamma_i + \gamma_j}\right)} \label{Eq: Tensorization} \\
& \leq k \sum_{\substack{i,j \in [n] : \\ \{i,j\} \in G, \, i < j}}{\left(\frac{\beta_j}{\beta_i + \beta_j} - \frac{\gamma_j}{\gamma_i + \gamma_j}\right)^{\!2} \left( 2 + \frac{\gamma_i}{\gamma_j} + \frac{\gamma_j}{\gamma_i} \right)} \nonumber \\
& \leq k \left( 2 + \delta + \frac{1}{\delta}\right) \! \sum_{\substack{i,j \in [n] : \\ \{i,j\} \in G, \, i < j}}{\! \left(\frac{\beta_j}{\beta_i + \beta_j} - \frac{\gamma_j}{\gamma_i + \gamma_j}\right)^{\!2}} \nonumber \\
& \leq k \left( 2 + \delta + \frac{1}{\delta}\right) \!\sum_{\substack{i,j \in [n] : \\ \{i,j\} \in G, \, i < j}}{\!\left(\left|\frac{\beta_j}{\beta_i + \beta_j} - \frac{\beta_j}{\gamma_i + \beta_j}\right| + \left|\frac{\beta_j}{\gamma_i + \beta_j} - \frac{\gamma_j}{\gamma_i + \gamma_j}\right|\right)^{\!2}} \nonumber \\
& \leq \frac{k}{16 \delta^2} \left( 2 + \delta + \frac{1}{\delta}\right) \! \sum_{\substack{i,j \in [n] : \\ \{i,j\} \in G, \, i < j}}{\!\left( \left|\beta_i - \gamma_i\right| + \left|\beta_j - \gamma_j\right|\right)^{2}} \nonumber \\
& \leq \frac{(1-\delta)^2}{4 \delta^2} \left( 2 + \delta + \frac{1}{\delta}\right) \frac{k |G|}{n^{2 q}} \nonumber \\
& = \varepsilon \nonumber 
\end{align}
where the first equality uses \eqref{Eq: KL definition}, \eqref{Eq: Compound rv}, and the fact that the $Z(i,j)$'s in \eqref{Eq: Compound rv} are conditionally independent given the skill parameters and the random graph realization $G$, and $P_{Z(i,j)|\beta,G}$ denotes the conditional probability distribution of $Z(i,j)$ given $(\alpha_1,\dots,\alpha_n) = \beta$ and $\G(n,p) = G$, the second equality holds because each $k Z(i,j) = \sum_{m = 1}^{k}{Z_m(i,j)}$ has binomial distribution given $\{i,j\} \in G$ (as indicated earlier, where the $Z_m(i,j)$'s are defined via \eqref{Eq: BT model}), the fourth equality follows from applying the formula for expected values of binomial random variables, the fifth equality uses the binary KL divergence function, which is defined as $(0,1) \ni (a,b) \mapsto D(a||b) \triangleq a \log\!\big(\frac{a}{b}\big) + (1-a) \log\!\big(\frac{1-a}{1-b}\big)$, the sixth inequality follows from a simple upper bound on KL divergence in terms of $\chi^2$-divergence \cite{Su1995} (alternatively see, e.g., \cite[Equation (4), Lemma 3]{MakurZheng2020} and the references therein for a detailed exposition):
$$ \forall a,b \in (0,1), \enspace D(a||b) \leq \underbrace{\frac{(a-b)^2}{b} + \frac{(a-b)^2}{1-b}}_{\text{binary } \chi^2\text{-divergence}} = \frac{(a-b)^2}{b(1-b)} \, , $$
the seventh inequality holds because the function $g : \big[\delta,\frac{1}{\delta}\big] \rightarrow \R$, $g(t) = t + \frac{1}{t}$ is maximized at $g(\delta) = g\big(\frac{1}{\delta}\big) = \delta + \frac{1}{\delta}$ (where $\frac{\gamma_i}{\gamma_j} \in \big[\delta,\frac{1}{\delta}\big]$), the eighth inequality follows from the triangle inequality, the ninth inequality follows from the coordinate-wise Lipschitz continuity shown in \clmref{Claim: Coordinate-wise Lipschitz Continuity} below, and the tenth inequality follows from \eqref{Eq: Closeness of Parameters}. This establishes the claim.
\end{proof} 

Before proceeding with the proof of \propref{Prop: Upper Bound on MI}, we quickly prove the coordinate-wise Lipschitz continuity claim used in the proof of \clmref{Claim: Covering} for completeness. 

\begin{claim}[Coordinate-wise Lipschitz Continuity]
\label{Claim: Coordinate-wise Lipschitz Continuity}
Consider the map $F : [\delta,\infty)^2 \rightarrow (0,\infty)$:
$$ \forall x,y \geq \delta, \enspace F(x,y) \triangleq \frac{x}{x + y} $$
which is used to define the likelihoods of the BTL model in \eqref{Eq: BT model}. This map is coordinate-wise Lipschitz continuous:
\begin{enumerate}
\item For any fixed $x \in [\delta,\infty)$:
$$ \forall y_1,y_2 \in [\delta,\infty), \enspace \left|F(x,y_1) - F(x,y_2) \right| \leq \frac{1}{4\delta} \left|y_1 - y_2\right| . $$ 
\item For any fixed $y \in [\delta,\infty)$:
$$ \forall x_1,x_2 \in [\delta,\infty), \enspace \left|F(x_1,y) - F(x_2,y) \right| \leq \frac{1}{4\delta} \left|x_1 - x_2\right| . $$
\end{enumerate}
\end{claim}

\begin{proof}
This is a straightforward exercise in calculus. Indeed, we have for all $x,y \geq \delta$:
$$ \left|\frac{\partial F}{\partial y}(x,y)\right| = \frac{x}{(x + y)^2} \quad \text{and} \quad \left|\frac{\partial F}{\partial x}(x,y)\right| = \frac{y}{(x + y)^2} \, , $$
which implies that:
$$ \max_{x,y \geq \delta}{\left|\frac{\partial F}{\partial y}(x,y)\right|} = \max_{x,y \geq \delta}{\left|\frac{\partial F}{\partial x}(x,y)\right|} = \max_{t \geq \delta}{\frac{t}{(t + \delta)^2}} = \frac{1}{4\delta} \, , $$
where the final equality holds because it is easy to verify that the map $\delta \leq t \mapsto t/(t + \delta)^2$ is globally maximized at $t = \delta$. This establishes the Lipschitz constants in parts 1 and 2 of \clmref{Claim: Coordinate-wise Lipschitz Continuity}. 
\end{proof}

Finally, using \lemref{Lemma: Covering Number Bound} and \clmref{Claim: Covering}, we get:
\begin{align*}
I(\alpha_1,\dots,\alpha_n;Z_G) & \leq \varepsilon + \log(M^{*}(\varepsilon)) \\
& \leq \frac{(1-\delta)^2}{4 \delta^2} \left( 2 + \delta + \frac{1}{\delta}\right) \frac{k |G|}{n^{2 q}} + \log(|\Q^n|) \\
& \leq \frac{(1-\delta)^2}{4 \delta^2} \left( 2 + \delta + \frac{1}{\delta}\right) \frac{k |G|}{n^{2 q}} + q n \log(n) \, .
\end{align*}
Then, by taking expectations on both sides of this inequality with respect to the law of $\G(n,p)$, we obtain, using \eqref{Eq: Cond MI Def}, \eqref{Eq: Simplified MI Relation}, and \eqref{Eq: Intermediate MI Bound}, that:
\begin{align*}
I(\pi;Z) & \leq \frac{(1-\delta)^2}{8 \delta^2} \left( 2 + \delta + \frac{1}{\delta}\right) \frac{k p n(n-1)}{n^{2 q}} + q n \log(n) \\
& \leq q n \log(n) + \frac{(1-\delta)^2}{8 \delta^2} \left( 2 + \delta + \frac{1}{\delta}\right) k p n^{2 - 2 q} \, .
\end{align*}
Neglecting $\delta$, $k$, and $p$, it is clear that setting $q \geq \frac{1}{2}$ ensures that this upper bound is $O(n \log(n))$. As the proofs of Theorems \ref{Thm: Minimax Relative l^infty-Risk} and \ref{Thm: Minimax l^1-Risk} in subsections \ref{Proof of Thm Minimax Relative l^infty-Risk} and \ref{Proof of Thm Minimax l^1-Risk} illustrate, choosing the smallest $q$ satisfying $q \geq \frac{1}{2}$ yields the tightest possible minimax lower bound using this approach. Thus, we set $q = \frac{1}{2}$ in the above bound and get:
$$ I(\pi;Z) \leq \frac{1}{2} n \log(n) + \frac{(1-\delta)^2}{8 \delta^2} \left( 2 + \delta + \frac{1}{\delta}\right) k p n $$
as desired. 
\end{proof}

In the literature, various information inequalities are often used to upper bound mutual information terms like $I(\alpha_1,\dots,\alpha_n ; Z_G)$, for any realization $\G(n,p) = G$, in a simpler manner (cf. \cite[Equation (44)]{ChenGuntuboyinaZhang2016}). However, these approaches tend to yield poorer scaling in the upper bound with $n$ compared to the covering number argument we utilize (in \lemref{Lemma: Covering Number Bound}). For example, the convexity of KL divergence immediately yields the bound (cf. \cite[Equation (44)]{ChenGuntuboyinaZhang2016} or \cite[p.1319]{Zhang2006}):
\begin{align}
I(\alpha_1,\dots,\alpha_n ; Z_G) & \leq \frac{1}{(1-\delta)^{2n}} \int_{[\delta,1]^n}{\int_{[\delta,1]^n}{D\!\left(P_{Z_G|\beta} \big|\big| P_{Z_G|\gamma}\right) \diff \gamma} \diff \beta} \nonumber \\
& \leq \sup_{\beta,\gamma \in [\delta,1]^n}{D\!\left(P_{Z_G|\beta} \big|\big| P_{Z_G|\gamma}\right)} \nonumber \\
& = k \sup_{\beta,\gamma \in [\delta,1]^n}{\sum_{\substack{i,j \in [n] : \\ \{i,j\} \in G, \, i < j}}{D\!\left(\frac{\beta_j}{\beta_i + \beta_j} \bigg|\bigg| \frac{\gamma_j}{\gamma_i + \gamma_j}\right)}} \nonumber \\
& \leq k |G| \max_{a,b \in \left[\frac{\delta}{1+\delta},\frac{1}{1 + \delta}\right]}{D(a||b)} \nonumber \\
& = k |G| \, D\!\left(\frac{1}{1 + \delta} \bigg|\bigg| \frac{\delta}{1 + \delta}\right)
\label{Eq: Quadratic Upper Bound}
\end{align}
where the third equality follows from \eqref{Eq: Tensorization} (and we use the notation $\beta = (\beta_1,\dots,\beta_n)$ and $\gamma = (\gamma_1,\dots,\gamma_n)$ as before), the maximization in the fourth inequality is over $a,b \in \big[\frac{\delta}{1+\delta},\frac{1}{1 + \delta}\big]$ because the map $0 < x \mapsto x/(c + x)$ is monotone increasing for every fixed $c > 0$, and the last equality follows from basic properties of binary KL divergence (see, e.g., \cite[p.20]{PolyanskiyWu2017Notes}). As before, by taking expectations on both sides of \eqref{Eq: Quadratic Upper Bound} with respect to the law of $\G(n,p)$, we obtain, using \eqref{Eq: Cond MI Def}, \eqref{Eq: Simplified MI Relation}, and \eqref{Eq: Intermediate MI Bound}, that:
\begin{equation}
\label{Eq: Last Quadratic Upper Bound}
I(\pi;Z) \leq \frac{k p n(n-1)}{2} \, D\!\left(\frac{1}{1 + \delta} \bigg|\bigg| \frac{\delta}{1 + \delta}\right) .
\end{equation}
Neglecting $\delta$, $k$, and $p$, it is clear that $I(\pi ; Z) = O(n^2)$ in \eqref{Eq: Last Quadratic Upper Bound}, but the proof of \propref{Prop: Upper Bound on MI} gives the sharper estimate $I(\pi ; Z) = O(n \log(n))$.

\subsection{Relative $\ell^{\infty}$-norm bound and the proof of \propref{Prop: MSE Decomposition}}
\label{App: MSE decomposition}

In order to prove \propref{Prop: MSE Decomposition}, we require the following known result from the literature \cite{Chenetal2019}, which upper bounds the relative $\ell^{\infty}$-norm loss between $\hat{\pi}_*$ and the canonically scaled skill parameters in \eqref{Eq: Canonically scaled merit parameters}. (As explained later in subsection \ref{App: Intuition}, $\hat{\pi}_*$ is intuitively a good estimator of $\pi$.)

\begin{lemma}[Relative $\ell^{\infty}$-Loss Bound {\cite[Theorem 3.1]{Chenetal2019}}]
\label{Lemma: Relative ell^infty-Loss Bound}
Suppose that $p \geq c_2 \log(n)/(\delta^5 n)$ for some sufficiently large constant $c_2 > 0$. Then, there exist (universal) constants $c_4,c_5 > 0$ such that for all sufficiently large $n \in \N$, we have:
$$ \P\!\left(\frac{\left\|\hat{\pi}_* - \pi\right\|_{\infty}}{\left\|\pi\right\|_{\infty}} \leq \frac{c_4}{\delta}\sqrt{\frac{\log(n)}{n p k}} \, \middle| \, \alpha_1,\dots,\alpha_n\right) \geq 1 - \frac{c_5}{n^{5}} $$
where the probability is computed with respect to the conditional distribution of the observation matrix $Z$ and the random graph $\G(n,p)$ given any realizations of the skill parameters $\alpha_1,\dots,\alpha_n$. 
\end{lemma}

We remark that the proof of \lemref{Lemma: Relative ell^infty-Loss Bound} in \cite{Chenetal2019} crucially uses the assumption that $\alpha_1,\dots,\alpha_n \in [\delta,1]$. We also note that the conditioning on $\alpha_1,\dots,\alpha_n$ in \lemref{Lemma: Relative ell^infty-Loss Bound} reflects the fact that the authors of \cite{Chenetal2019} consider a non-Bayesian scenario where $\alpha_1,\dots,\alpha_n$ are deterministic (and unknown). In contrast, this work assumes that $\alpha_1,\dots,\alpha_n$ are drawn i.i.d. from a prior PDF $\pdf$.

We now proceed to establishing \propref{Prop: MSE Decomposition} using \lemref{Lemma: Relative ell^infty-Loss Bound}.

\begin{proof}[Proof of \propref{Prop: MSE Decomposition}]
We commence this proof with the following bound on the MSE between $\ests$ and $\pdf$:
\begin{align}
& \E\!\left[\int_{\R}{\left(\ests(x) - \pdf(x)\right)^2\diff{x}}\right] \nonumber \\
& = \E\!\left[\int_{\R}{\left(\ests(x) - \hat{P}_{\alpha^n}^*(x) + \hat{P}_{\alpha^n}^*(x) - \pdf(x)\right)^2\diff{x}}\right] \nonumber \\
& \leq 2 \, \E\!\left[\int_{\R}{\left(\hat{P}_{\alpha^n}^*(x) - \pdf(x)\right)^2\diff{x}}\right] 
+ 2 \, \E\!\left[\int_{\R}{\left(\ests(x) - \hat{P}_{\alpha^n}^*(x)\right)^2\diff{x}}\right] \nonumber \\
& = 2 \, \E\!\left[\int_{\R}{\left(\hat{P}_{\alpha^n}^*(x) - \pdf(x)\right)^2\diff{x}}\right] 
+ 2 \int_{\R}{\E\!\left[\left(\ests(x) - \hat{P}_{\alpha^n}^*(x)\right)^2\right] \diff{x}} \nonumber \\
& = 2 \, \E\!\left[\int_{\R}{\left(\hat{P}_{\alpha^n}^*(x) - \pdf(x)\right)^2\diff{x}}\right] 
+ \frac{2}{n^2 h^2} \int_{\R}{\E\!\left[\left(\sum_{i = 1}^{n}{K\!\left(\frac{\hat{\alpha}_i - x}{h}\right) - K\!\left(\frac{\alpha_i - x}{h}\right)}\right)^{2}\right] \diff{x}} \nonumber \\
& = 2 \, \E\!\left[\int_{\R}{\left(\hat{P}_{\alpha^n}^*(x) - \pdf(x)\right)^2\diff{x}}\right] 
+ \frac{2}{n^2 h^2} \int_{-1}^{2}{\E\!\left[\left(\sum_{i = 1}^{n}{K\!\left(\frac{\hat{\alpha}_i - x}{h}\right) - K\!\left(\frac{\alpha_i - x}{h}\right)}\right)^{ 2}\right] \diff{x}}
\label{Eq: First bound on MSE}
\end{align}
where the second inequality follows from the inequality $(a + b)^2 \leq 2 a^2 + 2 b^2$ for all $a,b \in \R$, the third equality follows from Tonelli's theorem, the fourth equality follows from \eqref{Eq: Robust PR Estimator} and \eqref{Eq: True PR estimator} (where $h$ is general, i.e., not necessarily given by \eqref{Eq: Precise Bandwidth}), and the fifth equality holds because $h \leq 1$, $\alpha_i,\hat{\alpha}_i \in [0,1]$ for all $i \in [n]$ (see \eqref{Eq: Estimates of merit parameters}), and the maps $\R \ni t \mapsto K((\alpha_i - t)/h)$ and $\R \ni t \mapsto K((\hat{\alpha}_i - t)/h)$ have supports contained inside the intervals $[\alpha_i - h,\alpha_i + h]$ and $[\hat{\alpha}_i - h,\hat{\alpha}_i + h]$, respectively, for all $i \in [n]$. We proceed by bounding the second term in \eqref{Eq: First bound on MSE}.

To this end, we fix any $x \in \R$, and define $S_x \subseteq [n]$ to be the set of players $i \in [n]$, whose skill parameters $\alpha_i$, or their estimates $\hat{\alpha}_i$, fall into the small (and diminishing) neighborhood $[x - h,x + h]$:
$$ S_x \triangleq \left\{i \in [n] : \alpha_i \in [x - h,x + h] \enspace \text{or} \enspace \hat{\alpha}_i \in [x - h,x + h] \right\} . $$
Then, we may bound the integrand in the second term of \eqref{Eq: First bound on MSE} as follows:
\begin{align}
\E\!\left[\left(\sum_{i = 1}^{n}{K\!\left(\frac{\hat{\alpha}_i - x}{h}\right) - K\!\left(\frac{\alpha_i - x}{h}\right)}\right)^{ 2}\right] & = \E\!\left[\left(\sum_{i \in S_x}{K\!\left(\frac{\hat{\alpha}_i - x}{h}\right) - K\!\left(\frac{\alpha_i - x}{h}\right)}\right)^{ 2}\right] \nonumber \\
& \leq \E\!\left[\left(\sum_{i \in S_x}{\left|K\!\left(\frac{\hat{\alpha}_i - x}{h}\right) - K\!\left(\frac{\alpha_i - x}{h}\right)\right|}\right)^{ 2}\right] \nonumber \\
& \leq \frac{L_2^2}{h^2} \, \E\!\left[\left(\sum_{i \in S_x}{ \left|\hat{\alpha}_i - \alpha_i\right|}\right)^{ 2}\right] \nonumber \\
& \leq \frac{L_2^2}{h^2} \, \E\!\left[|S_x|^2 \max_{i \in [n]}{\left|\hat{\alpha}_i - \alpha_i\right|^{2}}\right] 
\label{Eq: Intermediate Integrand Bound}
\end{align}
where the first equality holds because the map $\R \ni t \mapsto K((t - x)/h)$ has support contained inside the interval $[x - h,x + h]$, the second inequality follows from the triangle inequality, and the third inequality follows from the fact that the kernel $K$ is $L_2$-Lipschitz continuous. To further upper bound \eqref{Eq: Intermediate Integrand Bound}, we now present three claims.

The first claim is a well-understood auxiliary result that is frequently used in the high-dimensional and non-parametric statistics and theoretical machine learning literature. It says that the intersection of high probability events is itself a high probability event.

\begin{claim}[Intersection of High Probability Events]
\label{claim: Intersection of High Probability Events}
Consider any two events $A_1$ and $A_2$ with probabilities satisfying $\P(A_1) \geq 1 - \varepsilon_1$ and $\P(A_2) \geq 1 - \varepsilon_2$ for any constants $\varepsilon_1,\varepsilon_2 \in [0,1]$. Then, we have:
$$ \P\!\left(A_1 \cap A_2\right) \geq 1 - \varepsilon_1 - \varepsilon_2 \, . $$
\end{claim}

\begin{proof}
This is an immediate corollary of the inclusion-exclusion principle. 
\end{proof}

The second claim utilizes \lemref{Lemma: Relative ell^infty-Loss Bound} to show that $\left|\hat{\alpha}_i - \alpha_i\right| = O\big(\!\max\{1/(\delta \sqrt{pk}),1\} \sqrt{\log(n)/n}\big)$ for every $i \in [n]$ with high probability for all sufficiently large $n$.

\begin{claim}[$\ell^{\infty}$-Norm Bound on Skill Parameter Estimation]
\label{Claim: Bound on Merit Parameter Estimation}
There exists a (universal) constant $c_6 > 0$ such that for all sufficiently large $n \in \N$:
$$ \P\!\left(\max_{i \in [n]}{\left|\hat{\alpha}_i - \alpha_i\right|} \leq c_6 \max\!\left\{\frac{1}{\delta \sqrt{p k}},1\right\} \! \sqrt{\frac{\log(n)}{n}}\right) \geq 1 - \frac{c_5 + 1}{n^5} $$
where $c_5 > 0$ is the fixed constant from \lemref{Lemma: Relative ell^infty-Loss Bound}.
\end{claim}

\begin{proof}
We prove this claim in four steps. Firstly, we establish that for all sufficiently large $n \in \N$:
\begin{equation}
\label{Eq: Close to 1}
\P\!\left(\max_{i \in [n]}{\alpha_i} \geq 1 - \frac{5 \log(n)}{b n}\right) \geq 1 - \frac{1}{n^{5}} \, .
\end{equation}
To establish \eqref{Eq: Close to 1}, note that:
\begin{align*}
\P\!\left(\forall i \in [n], \, \alpha_i < 1 - \frac{5 \log(n)}{b n} \right) & = \P\!\left(\alpha_1 < 1 - \frac{5 \log(n)}{b n} \right)^n \\
& = \left(1 - \int_{1 - 5 \log(n)/(b n)}^{1}{\pdf(t) \diff{t}}\right)^n \\
& \leq \left(1 - \frac{5 \log(n)}{n}\right)^n \\
& = \exp\!\left(n\log\!\left(1 - \frac{5 \log(n)}{n}\right)\right) \\
& \leq \frac{1}{n^{5}}
\end{align*}
where the first equality holds because $\alpha_1,\dots,\alpha_n$ are i.i.d., the third inequality holds because we have assumed that $\epsilon \geq 5 \log(n)/(b n)$ and $\pdf$ satisfies the lower bound $\pdf(t) \geq b$ for all $t \in [1-\epsilon,1]$, and for every large enough $n$, the fifth inequality follows from the well-known bound $\log(1 + x) \leq x$ for all $x > -1$. This produces the desired bound \eqref{Eq: Close to 1}.

Secondly, we define the normalized skill parameter random variables:
$$ \forall i \in [n] , \enspace \tilde{\alpha}_i \triangleq \frac{\alpha_i}{\max_{j \in [n]}{\alpha_j}} = \frac{\pi(i)}{\left\|\pi\right\|_{\infty}} $$
where $\pi$ denotes the probability vector of canonically scaled skill parameters in \eqref{Eq: Canonically scaled merit parameters}. It turns out that for all sufficiently large $n \in \N$:
\begin{equation}
\label{Eq: Good approximation of normalized params}
\P\!\left(\forall i \in [n], \, 0 \leq \tilde{\alpha}_i - \alpha_i \leq \frac{10}{c_3} \sqrt{\frac{\log(n)}{n}} \right) \geq 1 - \frac{1}{n^{5}} 
\end{equation}
i.e., the normalized skill parameters are close to the true skill parameters with high probability. To derive \eqref{Eq: Good approximation of normalized params}, suppose that the event in \eqref{Eq: Close to 1} occurs. Then, we have:
$$ 1 - \frac{5 \log(n)}{b n} \leq \max_{i \in [n]}{\alpha_i} \leq 1 $$
which implies that for every $i \in [n]$:
$$ \alpha_i \leq \tilde{\alpha}_i \leq \alpha_i \left(1 - \frac{5 \log(n)}{b n} \right)^{\! -1} \leq \alpha_i \left(1 - \frac{5}{c_3} \sqrt{\frac{\log(n)}{n}} \right)^{\! -1} \leq \alpha_i \left(1 + \frac{10}{c_3} \sqrt{\frac{\log(n)}{n}} \right) $$
where the third inequality follows from the lower bound we have assumed on $b$ in the proposition statement, and the fourth inequality holds for all sufficiently large $n \in \N$, because the bound $(1 - x)^{-1} = 1 + x + x^2(1 - x)^{-1} \leq 1 + x + 2 x^2 \leq 1 + 2x$ holds for any $x \in \big(0,\frac{1}{2}\big]$. Thus, since $\alpha_i \leq 1$ for all $i \in [n]$, we get \eqref{Eq: Good approximation of normalized params}.

Thirdly, we approximate the normalized skill parameters $\tilde{\alpha}_i$ using the estimates $\hat{\alpha}_i$. Recall from \lemref{Lemma: Relative ell^infty-Loss Bound} that for all sufficiently large $n \in \N$:
\begin{equation}
\label{Eq: Chen event} 
\P\!\left(\frac{\left\|\hat{\pi}_* - \pi\right\|_{\infty}}{\left\|\pi\right\|_{\infty}} \leq \frac{c_4}{\delta}\sqrt{\frac{\log(n)}{n p k}}\right) \geq  1 - \frac{c_5}{n^5} 
\end{equation}
where we take expectations with respect to the law of $\alpha_1,\dots,\alpha_n$. Suppose that the event in \eqref{Eq: Chen event} happens. Then, for any $i \in [n]$, notice that for all sufficiently large $n \in \N$:
\begin{align*}
\hat{\alpha}_i & = \frac{\hat{\pi}_*(i)}{\left\|\hat{\pi}_*\right\|_{\infty}} \\
& = \frac{\pi(i) + (\hat{\pi}_*(i) - \pi(i))}{\left\|\pi\right\|_{\infty} + (\left\|\hat{\pi}_*\right\|_{\infty} - \left\|\pi\right\|_{\infty})} \\
& \leq \frac{\pi(i) + \left\|\hat{\pi}_* - \pi\right\|_{\infty}}{\left\|\pi\right\|_{\infty} - \left|\left\|\hat{\pi}_*\right\|_{\infty} - \left\|\pi\right\|_{\infty}\right|} \\
& \leq \frac{\pi(i) + \left\|\hat{\pi}_* - \pi\right\|_{\infty}}{\left\|\pi\right\|_{\infty} - \left\|\hat{\pi}_* - \pi\right\|_{\infty}} \\
& = \left(\tilde{\alpha}_i + \frac{\left\|\hat{\pi}_* - \pi\right\|_{\infty}}{\left\|\pi\right\|_{\infty}}\right) \! \left(1 - \frac{\left\|\hat{\pi}_* - \pi\right\|_{\infty}}{\left\|\pi\right\|_{\infty}}\right)^{\! -1} \\
& \leq \left(\tilde{\alpha}_i + \frac{\left\|\hat{\pi}_* - \pi\right\|_{\infty}}{\left\|\pi\right\|_{\infty}}\right) \! \left(1 + 2 \, \frac{\left\|\hat{\pi}_* - \pi\right\|_{\infty}}{\left\|\pi\right\|_{\infty}}\right) \\
& \leq \tilde{\alpha}_i + 4 \, \frac{\left\|\hat{\pi}_* - \pi\right\|_{\infty}}{\left\|\pi\right\|_{\infty}}
\end{align*}
where the first equality follows from \eqref{Eq: Estimates of merit parameters}, the fourth inequality follows from the (reverse) Minkowski inequality, the sixth inequality holds for all sufficiently large $n \in \N$ because: 1) the event in \eqref{Eq: Chen event} occurs, 2) we have assumed that $\lim_{n \rightarrow \infty}{\delta^{-1} (n p k)^{-1/2} \log(n)^{1/2}} = 0$ in the proposition statement, and 3) we use the bound $(1 - x)^{-1} \leq 1 + 2x$ for $x \in \big(0,\frac{1}{2}\big]$, and the seventh inequality holds because $\tilde{\alpha}_i \leq 1$ for all $i \in [n]$. Likewise, using analogous reasoning, observe that for all sufficiently large $n \in \N$:
\begin{align*}
\hat{\alpha}_i & \geq \frac{\pi(i) - \left\|\hat{\pi}_* - \pi\right\|_{\infty}}{\left\|\pi\right\|_{\infty} + \left\|\hat{\pi}_* - \pi\right\|_{\infty}} \\
& = \left(\tilde{\alpha}_i - \frac{\left\|\hat{\pi}_* - \pi\right\|_{\infty}}{\left\|\pi\right\|_{\infty}}\right) \! \left(1 + \frac{\left\|\hat{\pi}_* - \pi\right\|_{\infty}}{\left\|\pi\right\|_{\infty}}\right)^{\! -1} \\
& \geq \left(\tilde{\alpha}_i - \frac{\left\|\hat{\pi}_* - \pi\right\|_{\infty}}{\left\|\pi\right\|_{\infty}}\right) \! \left(1 - \frac{\left\|\hat{\pi}_* - \pi\right\|_{\infty}}{\left\|\pi\right\|_{\infty}}\right) \\
& \geq \tilde{\alpha}_i - 2 \, \frac{\left\|\hat{\pi}_* - \pi\right\|_{\infty}}{\left\|\pi\right\|_{\infty}} \, .
\end{align*}
Therefore, using \lemref{Lemma: Relative ell^infty-Loss Bound}, we obtain that for all sufficiently large $n \in \N$:
\begin{equation}
\label{Eq: Good approximation of normalized params 2}
\P\!\left(\max_{i \in [n]}{\left|\hat{\alpha}_i - \tilde{\alpha}_i\right|} \leq 4 \, \frac{\left\|\hat{\pi}_* - \pi\right\|_{\infty}}{\left\|\pi\right\|_{\infty}} \right) \geq 1 - \frac{c_5}{n^5} \, .    
\end{equation}

Finally, we combine \eqref{Eq: Good approximation of normalized params} and \eqref{Eq: Good approximation of normalized params 2} together. Suppose that the events in both \eqref{Eq: Good approximation of normalized params} and \eqref{Eq: Chen event} occur. Then, we get:
\begin{align*}
\max_{i \in [n]}{\left|\hat{\alpha}_i - \alpha_i\right|} & \leq \max_{i \in [n]}{\left|\hat{\alpha}_i - \tilde{\alpha}_i\right|} + \max_{i \in [n]}{\left|\tilde{\alpha}_i - \alpha_i\right|} \\
& \leq 4 \, \frac{\left\|\hat{\pi}_* - \pi\right\|_{\infty}}{\left\|\pi\right\|_{\infty}} + \frac{10}{c_3} \sqrt{\frac{\log(n)}{n}} \\
& \leq \frac{4 c_4}{\delta} \sqrt{\frac{\log(n)}{n p k}} + \frac{10}{c_3} \sqrt{\frac{\log(n)}{n}} \\
& = O\!\left(\!\max\!\left\{\frac{1}{\delta \sqrt{p k}},1\right\} \! \sqrt{\frac{\log(n)}{n}}\right)
\end{align*}
where the first inequality follows from the triangle inequality, the second inequality follows from \eqref{Eq: Good approximation of normalized params} and \eqref{Eq: Good approximation of normalized params 2}, and the third inequality follows from \eqref{Eq: Chen event} (or \lemref{Lemma: Relative ell^infty-Loss Bound}). Using \clmref{claim: Intersection of High Probability Events}, this produces the desired bound:
$$ \P\!\left(\max_{i \in [n]}{\left|\hat{\alpha}_i - \alpha_i\right|} =  O\!\left(\max\!\left\{\frac{1}{\delta \sqrt{p k}},1\right\} \! \sqrt{\frac{\log(n)}{n}}\right)\right) \geq 1 - \frac{c_5 + 1}{n^5} $$
for all sufficiently large $n$.
\end{proof}

Finally, our third claim uses \clmref{Claim: Bound on Merit Parameter Estimation} to argue that with high probability, the cardinality of $S_x$ is \emph{not} too large.

\begin{claim}[Cardinality Bound for $S_x$]
\label{Claim: Cardinality Bound}
There exists a sufficiently large (universal) constant $c_7 > 0$ such that for every sufficiently large $n \in \N$, we have:
$$ \P\!\left(|S_x| \leq c_7 B h n \right) \geq 1 - \frac{c_5 + 2}{n^{5}} \, . $$
\end{claim}  

\begin{proof}
First, for any constant $\tau > 1$ (to be chosen later), we define $N_h \in [n]\cup\!\{0\}$ to be the discrete random variable representing the number of players $i \in [n]$ for whom $\alpha_i$ belongs to the interval $[x - \tau h,x + \tau h]$:
$$ N_h \triangleq \sum_{i = 1}^{n}{\I\!\left\{\alpha_i \in [x - \tau h,x + \tau h]\right\}} \, . $$
Then, fix any $\varepsilon \geq 0$, and observe using \lemref{Lemma: Hoeffding's Inequality} in Appendix \ref{App: Concentration of Measure Inequalities} that:
$$ \P\!\left(\frac{1}{n} N_h - \P\!\left(\alpha_1 \in [x - \tau h,x + \tau h]\right) > \varepsilon\right) \leq \exp\!\left(- 2 n \varepsilon^2\right) $$
which uses the fact that $\big\{\I\{\alpha_i \in [x - \tau h,x + \tau h]\} : i \in [n]\big\}$ are i.i.d. Bernoulli random variables with mean $\P(\alpha_1 \in [x - \tau h,x + \tau h])$, since $\alpha_1,\dots,\alpha_n$ are drawn i.i.d. from $\pdf$. At this point, letting $\varepsilon = \sqrt{5 \log(n) / (2 n)}$ yields:
\begin{equation}
\label{Eq: Intermediate Hoeffding Bound}
\P\!\left(\frac{1}{n} N_h - \P\!\left(\alpha_1 \in [x - \tau h,x + \tau h]\right) > \sqrt{\frac{5 \log(n)}{2 n}} \right) \leq \exp\!\left(-\frac{5 n \log(n)}{n} \right) = \frac{1}{n^{5}} \, .
\end{equation}
Next, recall that $\pdf \in \pdfs$ is uniformly bounded (almost everywhere) by $B > 0$, i.e., $\pdf(t) \leq B$ for all $t \in \R$.
Using this bound, we obtain:
\begin{equation}
\label{Eq: Upper Bound on Expectation}
\P\!\left(\alpha_1 \in [x - \tau h,x + \tau h]\right) = \int_{x - \tau h}^{x + \tau h}{\pdf(t) \diff{t}} \leq 2 B \tau h \, .
\end{equation}
Hence, combining \eqref{Eq: Intermediate Hoeffding Bound} and \eqref{Eq: Upper Bound on Expectation}, we have that with probability at least $1 - n^{-5}$:
\begin{align*}
\frac{1}{n} N_h & \leq \P\!\left(\alpha_1 \in [x - \tau h,x + \tau h]\right) + \sqrt{\frac{5 \log(n)}{2 n}} \\
& \leq 2 B \tau h + \sqrt{\frac{5 \log(n)}{2 n}} \, .
\end{align*}
Equivalently, we have derived the following bound: 
\begin{equation}
\label{Eq: Proportion of Players in Small Neighborhood}
\P\!\left(N_h \leq 2 B \tau h n + \sqrt{\frac{5}{2}\, n \log(n)} \right) \geq 1 - \frac{1}{n^{5}} \, .
\end{equation}

Now, recall from the proposition statement that $h = \Omega\big(\!\max\{1/(\delta \sqrt{pk}),1\} \sqrt{\log(n)/n}\big)$. Hence, we may choose $\tau > 1$ large enough so that for all sufficiently large $n$, we have:
\begin{equation}
\label{Eq: Domination of h}
(\tau - 1) h \geq c_6 \max\!\left\{\frac{1}{\delta \sqrt{pk}},1\right\} \! \sqrt{\frac{\log(n)}{n}}
\end{equation}
where $c_6 > 0$ is the constant from \clmref{Claim: Bound on Merit Parameter Estimation}. Assume that both the events in \eqref{Eq: Proportion of Players in Small Neighborhood} and \clmref{Claim: Bound on Merit Parameter Estimation} occur. Then, for any $i \in [n]$, if $\alpha_i \in [x - h,x + h]$, then we trivially have $\alpha_i \in [x - \tau h,x + \tau h]$ since $\tau > 1$. On the other hand, if $\hat{\alpha}_i \in [x - h,x + h]$, then we have $\alpha_i \in [x - \tau h,x + \tau h]$, because $|\hat{\alpha}_i - \alpha_i| \leq c_6 \max\{1/(\delta \sqrt{pk}),1\} \sqrt{\log(n)/n} \leq (\tau - 1) h$ by \clmref{Claim: Bound on Merit Parameter Estimation} and \eqref{Eq: Domination of h}. Thus, we get:
$$ |S_x| \leq N_{h} $$
with $\tau$ chosen according to \eqref{Eq: Domination of h}. Applying \clmref{claim: Intersection of High Probability Events}, \clmref{Claim: Bound on Merit Parameter Estimation}, and \eqref{Eq: Proportion of Players in Small Neighborhood} together yields:
$$ \P\!\left(|S_x| \leq 2 B \tau h n + \sqrt{\frac{5}{2}\, n \log(n)} \right) \geq 1 - \frac{c_5 + 2}{n^{5}} $$
for every sufficiently large $n$. Lastly, let $c_7 > 2 \tau$ be a sufficiently large constant so that for all sufficiently large $n$:
$$ (c_7 - 2 \tau) B h \geq \sqrt{\frac{5 \log(n)}{2 n}} $$
which is consistent with our assumption that $h = \Omega\big(\!\max\{1/(\delta \sqrt{pk}),1\} \sqrt{\log(n)/n}\big)$. Therefore, we may write:
$$ \P\!\left(|S_x| \leq c_7 B h n \right) \geq 1 - \frac{c_5 + 2}{n^{5}} $$
for every sufficiently large $n$, which completes the proof.
\end{proof}

Having developed \clmref{Claim: Cardinality Bound}, we are now in a position to upper bound \eqref{Eq: Intermediate Integrand Bound}. For any sufficiently large $n$, define the event in \clmref{Claim: Cardinality Bound} as:
$$ A_n \triangleq \left\{|S_x| \leq c_7 B h n\right\} . $$
Then, observe that:
\begin{align}
& \E\!\left[|S_x|^2 \max_{i \in [n]}{\left|\hat{\alpha}_i - \alpha_i\right|^{2}}\right] \nonumber \\
& \qquad = \E\!\left[|S_x|^2 \max_{i \in [n]}{\left|\hat{\alpha}_i - \alpha_i\right|^{2}}\middle|A_n\right] \P(A_n) 
+ \E\!\left[|S_x|^2 \max_{i \in [n]}{\left|\hat{\alpha}_i - \alpha_i\right|^{2}}\middle|A_n^c\right] \P(A_n^c) \nonumber \\
 & \qquad \leq c_7^2 B^2 h^2 n^2 \, \E\!\left[\max_{i \in [n]}{\left|\hat{\alpha}_i - \alpha_i\right|^{2}}\middle|A_n\right] \P(A_n) + \frac{c_5 + 2}{n^3} \nonumber \\
& \qquad\leq c_7^2 B^2 h^2 n^2 \bigg(\E\!\left[\max_{i \in [n]}{\left|\hat{\alpha}_i - \alpha_i\right|^{2}}\middle|A_n\right] \P(A_n) 
+ \underbrace{\E\!\left[\max_{i \in [n]}{\left|\hat{\alpha}_i - \alpha_i\right|^{2}}\middle|A_n^c\right] \P(A_n^c)}_{\geq \, 0}\bigg) + \frac{c_5 + 2}{n^3}\nonumber \\
&\qquad = c_7^2 B^2 h^2 n^2 \, \E\!\left[\max_{i \in [n]}{\left|\hat{\alpha}_i - \alpha_i\right|^{2}}\right] + \frac{c_5 + 2}{n^3} 
\label{Eq: Intermediate Integrand Bound 2}
\end{align}
where the first and fourth equalities follow from the tower property, and the second inequality uses \clmref{Claim: Cardinality Bound} and the facts that $|S_x| \leq n$ and $|\hat{\alpha}_i - \alpha_i| \leq 1$ for all $i \in [n]$ (see \eqref{Eq: Estimates of merit parameters}). Plugging \eqref{Eq: Intermediate Integrand Bound 2} into \eqref{Eq: Intermediate Integrand Bound} produces:
$$ \E\!\left[\left(\sum_{i = 1}^{n}{K\!\left(\frac{\hat{\alpha}_i - x}{h}\right) - K\!\left(\frac{\alpha_i - x}{h}\right)}\right)^{ 2}\right] \leq \frac{L_2^2}{h^2} \left(c_7^2 B^2 h^2 n^2 \, \E\!\left[\max_{i \in [n]}{\left|\hat{\alpha}_i - \alpha_i\right|^{2}}\right] + \frac{c_5 + 2}{n^3}\right) . $$
We can then substitute this bound into \eqref{Eq: First bound on MSE} to obtain:
\begin{align*}
& \E\!\left[\int_{\R}{\left(\ests(x) - \pdf(x)\right)^2\diff{x}}\right]  \\
& \quad \leq 2 \, \E\!\left[\int_{\R}{\left(\hat{P}_{\alpha^n}^*(x) - \pdf(x)\right)^2\diff{x}}\right] 
+ \frac{2 L_2^2}{n^2 h^4} \int_{-1}^{2}{c_7^2 B^2 h^2 n^2 \, \E\!\left[\max_{i \in [n]}{\left|\hat{\alpha}_i - \alpha_i\right|^{2}}\right] + \frac{c_5 + 2}{n^3} \, \diff{x}} \\
& \quad = 2 \, \E\!\left[\int_{\R}{\left(\hat{P}_{\alpha^n}^*(x) - \pdf(x)\right)^2\diff{x}}\right] 
+ \frac{6 L_2^2}{n^2 h^4} \left(c_7^2 B^2 h^2 n^2 \, \E\!\left[\max_{i \in [n]}{\left|\hat{\alpha}_i - \alpha_i\right|^{2}}\right] + \frac{c_5 + 2}{n^3} \right)  \\
& \quad = 2 \, \E\!\left[\int_{\R}{\left(\hat{P}_{\alpha^n}^*(x) - \pdf(x)\right)^2\diff{x}}\right] 
+ \frac{6 c_7^2 B^2 L_2^2}{h^2} \, \E\!\left[\max_{i \in [n]}{\left|\hat{\alpha}_i - \alpha_i\right|^{2}}\right] + \frac{6 (c_5 + 2) L_2^2}{n^5 h^4} 
\end{align*}
for all sufficiently large $n$. This completes the proof after letting $c_8 = 6  c_7^2$ and $c_9 = 6 (c_5 + 2)$.
\end{proof}

%% file: upperbound.tex
\section{MSE upper bound for skill density estimation}
\label{MSE Upper Bound for Skill Density Estimation}

In this appendix, we will provide some intuition for Algorithm \ref{Algorithm: Density Estimation} in subsection \ref{App: Intuition}, and then prove \thmref{Thm: MSE Upper Bound} in subsection \ref{Proof of MSE Upper Bound}. 

\subsection{Intuition for Algorithm \ref{Algorithm: Density Estimation}}
\label{App: Intuition}

We briefly explain the intuition behind each of the two steps of Algorithm \ref{Algorithm: Density Estimation}. As we mentioned, step 1 of Algorithm \ref{Algorithm: Density Estimation} is inspired by the (spectral) rank centrality algorithm of \cite{NegahbanOhShah2012,NegahbanOhShah2017}. To understand this stage, define the row stochastic matrix $D \in \S_{n \times n}$, whose $(i,j)$th element is given by the BTL skill  parameters:
\begin{equation}
\label{Eq: BT matrix Markov kernel}
\forall i,j \in [n], \enspace D(i,j) \triangleq 
\begin{cases}
\displaystyle{\frac{1}{2 n}\left(\frac{\alpha_j}{\alpha_i + \alpha_j}\right)} \, , & i \neq j  \\
\displaystyle{1 - \frac{1}{2 n} \sum_{r \in [n]\backslash\!\{i\}}{\frac{\alpha_r}{\alpha_i + \alpha_r}}} \, , & i = j
\end{cases}
\end{equation}
where $D(i,j) + D(j,i) = (2n)^{-1}$ for all $i,j \in [n]$ such that $i \neq j$. Note that it is straightforward to verify from \eqref{Eq: BT matrix Markov kernel} that $D(i,i) \geq 0$ for all $i \in [n]$, because:
\begin{align}
\forall i \in [n], \enspace D(i,i) & = 1 - \frac{1}{2 n} \sum_{r \in [n]\backslash\!\{i\}}{\frac{\alpha_r}{\alpha_i + \alpha_r}} \nonumber \\
& = \frac{1}{2n} \sum_{r \in [n]\backslash\!\{i\}}{\frac{2n}{n-1} - \frac{\alpha_r}{\alpha_i + \alpha_r}} \nonumber \\
& = \frac{1}{2n} \sum_{r \in [n]\backslash\!\{i\}}{\frac{n + 1}{n - 1} + \frac{\alpha_i}{\alpha_i + \alpha_r}} \nonumber \\
& = \frac{n+1}{2n} + \frac{1}{2n}\sum_{r \in [n]\backslash\!\{i\}}{\frac{\alpha_i}{\alpha_i + \alpha_r}} \geq 0 \, .
\label{Eq: Diagonal of S}
\end{align}
We will construe $D$ as the transition probability matrix of a (time-homogeneous) discrete-time Markov chain on the state space $[n]$ of players. 
Next, following the crucial observation of \cite{NegahbanOhShah2017}, notice using \eqref{Eq: Canonically scaled merit parameters} and \eqref{Eq: BT matrix Markov kernel} that the ensuing \emph{detailed balance conditions} are satisfied:
\begin{equation}
\forall i,j \in [n], \enspace \pi(i) D(i,j) = \pi(j) D(j,i) 
\end{equation}
where $\pi \in \S_n$ are the canonically scaled skill parameters in \eqref{Eq: Canonically scaled merit parameters}. This implies that $D$ defines a \emph{reversible} Markov chain with invariant distribution $\pi = \pi D$ (see, e.g., \cite[Proposition 1.19]{LevinPeresWilmer2009}). Moreover, this Markov chain is ergodic (i.e., irreducible and aperiodic) because $D > 0$ entry-wise, which means that $\pi$ is the unique invariant distribution of $D$. This general idea that canonically scaled skill parameters of the BTL model form an invariant distribution of a reversible Markov chain is known as ``rank centrality'' \cite{NegahbanOhShah2012,NegahbanOhShah2017}. 

Step 1 of Algorithm \ref{Algorithm: Density Estimation} estimates the BTL skill parameters $\alpha_1,\dots,\alpha_n$ using $\hat{\alpha}_1,\dots,\hat{\alpha}_n$ in \eqref{Eq: Estimates of merit parameters} by first estimating the canonically scaled skill parameters $\pi \in \S_n$ in \eqref{Eq: Canonically scaled merit parameters}. To estimate $\pi$, it is reasonable to first produce an estimate of $D$ that is itself a row stochastic matrix (with high probability), and then utilize the corresponding invariant distribution as our estimate of $\pi$. Notice that:
\begin{equation}
\E\!\left[S \middle| \alpha_1,\dots,\alpha_n\right] = D
\end{equation}
where $S \in \R^{n \times n}$ is defined in \eqref{Eq: Stochastic Matrix Estimator}. Hence, $S$ (which is row stochastic with high probability, as shown in subsection \ref{Proof of Prop Estimator Stochastic Matrix}) can be construed as our estimator of the Markov kernel $D$. As a consequence, the invariant distribution $\hat{\pi}_* \in \S_n$ of $S$ in \eqref{Eq: Empirical Stat Dist} can be perceived as our estimator of $\pi$.

Step 2 of Algorithm \ref{Algorithm: Density Estimation} constructs an estimator for the unknown PDF $\pdf$ of interest using the skill parameter estimates $\hat{\alpha}_1,\dots,\hat{\alpha}_n$ obtained from Step 1. Clearly, if we had access to the true i.i.d. samples  $\alpha_1,\dots,\alpha_n$ from $\pdf$, then we could use the vanilla PR kernel density estimator in \eqref{Eq: True PR estimator} to estimate $\pdf$, because it is known to be minimax optimal for appropriate choices of bandwidth $h$ and kernel function $K$, cf. \cite{Tsybakov2009,Wasserman2019}. However, we do not have access to these true samples. Thus, we utilize the estimates $\hat{\alpha}_1,\dots,\hat{\alpha}_n$ to construct an analogous estimator in \eqref{Eq: Robust PR Estimator}, which is the output of (Step 2 of) Algorithm \ref{Algorithm: Density Estimation}. Intuitively, we expect this estimator to perform well, because $\hat{\alpha}_1,\dots,\hat{\alpha}_n$ should be ``close'' to $\alpha_1,\dots,\alpha_n$ when $n$ is large.

Finally, we also briefly explain our reasons behind the assumptions we imposed on the non-parametric class of skill densities $\pdfs$. By restricting $\pdf$ to $\pdfs$, we are able to perform tractable non-asymptotic analysis of Algorithm \ref{Algorithm: Density Estimation}. Indeed, the H\"{o}lder class and boundedness assumptions of $\pdfs$ are standard in the non-parametric density estimation literature (see, e.g., \cite[Section 1.2]{Tsybakov2009}). Moreover, since the BTL likelihoods in \eqref{Eq: BT model} are invariant to scaling the skill parameters, we may assume without loss of generality that $\alpha_1,\dots,\alpha_n \leq 1$. 
On the other hand, assuming that $\alpha_1,\dots,\alpha_n \geq \delta$ is equivalent to the \emph{condition number} (or dynamic range) bound:
 \begin{equation}
\max_{i,j \in [n]}{\frac{\alpha_i}{\alpha_j}} \leq \frac{1}{\delta} \, ,
 \end{equation}
which is very often exploited in the BTL-related literature, cf. \cite[Equations (1.5) and (1.6)]{SimonsYao1999}, \cite[Theorems 1 and 2]{NegahbanOhShah2017}, \cite[Equation (2.4)]{Chenetal2019}. Note that $\pdf$ having support in $[\delta,1]$ corresponds precisely to the condition that $\alpha_1,\dots,\alpha_n \in [\delta,1]$. 

\subsection{Proof of \thmref{Thm: MSE Upper Bound}}
\label{Proof of MSE Upper Bound}

To establish \thmref{Thm: MSE Upper Bound}, we first present a well-known lemma which conveys the tradeoff between the bias and variance of the classical PR kernel density estimator $\hat{P}_{\alpha^n}^*$ defined in \eqref{Eq: True PR estimator}, cf. \cite[Propositions 1.2 and 1.4]{Tsybakov2009}, \cite[Lemmata 3 and 4]{Wasserman2019}.

\begin{lemma}[Bias-Variance Tradeoff {\cite{Tsybakov2009,Wasserman2019}}]
\label{Lemma: Bias-Variance Tradeoff}
For any $\pdf \in \pdfs$, any kernel $K : [-1,1] \rightarrow \R$ of order $s = \lceil \eta \rceil - 1$, any $n \in \N$, and any bandwidth $h \in (0,1]$, we have:
$$ \E\!\left[\int_{\R}{\left(\hat{P}_{\alpha^n}^*(x) - \pdf(x)\right)^2 \diff{x}}\right] \leq \frac{1}{nh} \left( \int_{-1}^{1}{K(x)^2 \diff{x}}\right) + 3 h^{2 \eta} \left(\frac{L_1}{s!} \int_{-1}^{1}{|x|^{\eta} |K(x)| \diff{x}} \right)^{\! 2} $$
where $\alpha_1,\dots,\alpha_n$ are i.i.d. with distribution $\pdf$.
\end{lemma}

In \lemref{Lemma: Bias-Variance Tradeoff}, it is well-known that the first term captures the variance of $\hat{P}_{\alpha^n}^*$, and the second term bounds the squared bias of $\hat{P}_{\alpha^n}^*$. Specifically, as shown in \cite{Tsybakov2009,Wasserman2019}, the bound on the variance term uses the property 
that the kernel is square-integrable, and the bound on the bias term uses the other properties 
in the definition of a kernel as well as the H\"{o}lder class assumption on $\pdf$ (outlined earlier). Furthermore, we remark that the bound on the bias term in \lemref{Lemma: Bias-Variance Tradeoff} follows from \cite[Proposition 1.2]{Tsybakov2009} by noting that $\pdf$ has its support in the interval $[0,1]$ and $\hat{P}_{\alpha^n}^*$ has its support in the interval $[-1,2]$ (because the kernel $K$ has its support in $[-1,1]$ and the bandwidth $h \leq 1$). In fact, the length of the interval $[0,1] \cup [-1,2] = [-1,2]$ is what gives rise to the constant $3$ in \lemref{Lemma: Bias-Variance Tradeoff}.

We next prove \thmref{Thm: MSE Upper Bound} using Lemmata \ref{Lemma: Relative ell^infty-Loss Bound} and \ref{Lemma: Bias-Variance Tradeoff} and \propref{Prop: MSE Decomposition}.

\begin{proof}[Proof of \thmref{Thm: MSE Upper Bound}]
We begin by recalling the result of \propref{Prop: MSE Decomposition}. There exist sufficiently large constants $c_8,c_9 > 0$ such that for all sufficiently large $n \in \N$:
$$ \E\!\left[\int_{\R}{\!\!\left(\!\ests(x) - \pdf(x)\!\right)^{\! 2}\!\!\diff{x}}\right] \!\leq 2 \E\!\left[\int_{\R}{\!\!\left(\!\hat{P}_{\alpha^n}^*(x) - \pdf(x)\!\right)^{\! 2}\!\!\diff{x}}\right] + \frac{c_8 B^2 L_2^2}{h^2} \E\!\left[\max_{i \in [n]}{\left|\hat{\alpha}_i - \alpha_i\right|^{2}}\right] + \frac{c_9 L_2^2}{n^5 h^4} $$
where we assume that the bandwidth $h \in (0,1]$ satisfies $h = \Omega\big(\!\max\{1/(\delta \sqrt{p k}),1\} \sqrt{\log(n)/n}\big)$ and that $\lim_{n \rightarrow \infty}{\delta^{-1} (n p k)^{-1/2} \log(n)^{1/2}} = 0$. Now define the constants:
\begin{align*}
c_{10} & = 2 \int_{-1}^{1}{K(x)^2 \diff{x}} \, , \\
c_{11} & = 6 \left(\frac{L_1}{s!} \int_{-1}^{1}{|x|^{\eta} |K(x)| \diff{x}} \right)^{\! 2} ,
\end{align*}
which depend on the parameters $\eta$ and $L_1$ (that define the non-parametric class of PDFs 
$\pdfs$)
and on the kernel $K:[-1,1] \rightarrow \R$. Then, applying \lemref{Lemma: Bias-Variance Tradeoff} to the first term of the inequality in \propref{Prop: MSE Decomposition}, we obtain:
\begin{equation}
\label{Eq: MSE Decomp plus BV tradeoff}
\E\!\left[\int_{\R}{\left(\ests(x) - \pdf(x)\right)^{\! 2} \diff{x}}\right] \leq \frac{c_{10}}{nh} + c_{11} h^{2 \eta} + \frac{c_8 B^2 L_2^2}{h^2} \, \E\!\left[\max_{i \in [n]}{\left|\hat{\alpha}_i - \alpha_i\right|^{2}}\right] + \frac{c_9 L_2^2}{n^5 h^4}  
\end{equation}
for all sufficiently large $n$. We next upper bound the $\E\big[\!\max_{i \in [n]}{|\hat{\alpha}_i - \alpha_i|^{2}}\big]$ term in \eqref{Eq: MSE Decomp plus BV tradeoff}. To this end, define the event:
$$ A \triangleq \left\{\max_{i \in [n]}{\left|\hat{\alpha}_i - \alpha_i\right|} \leq c_6 \max\!\left\{\frac{1}{\delta \sqrt{p k}},1\right\} \! \sqrt{\frac{\log(n)}{n}}\right\} $$
using the constant $c_6 > 0$ from \clmref{Claim: Bound on Merit Parameter Estimation} in the proof of \propref{Prop: MSE Decomposition} in subsection \ref{App: MSE decomposition}, and recall from \clmref{Claim: Bound on Merit Parameter Estimation} that there exist constants $c_5,c_6 > 0$ such that for all sufficiently large $n \in \N$:
\begin{equation}
\label{Eq: Main claim restatement}
\P\!\left(A\right) \geq 1 - \frac{c_5 + 1}{n^5} \, . 
\end{equation}
Now observe that for all sufficiently large $n$:
\begin{align}
\E\!\left[\max_{i \in [n]}{\left|\hat{\alpha}_i - \alpha_i\right|^{2}}\right] & = \E\!\left[\max_{i \in [n]}{\left|\hat{\alpha}_i - \alpha_i\right|^{2}}\middle|A\right] \P(A) + \E\!\left[\max_{i \in [n]}{\left|\hat{\alpha}_i - \alpha_i\right|^{2}}\middle|A^{c}\right] \P(A^{c}) \nonumber \\
& \leq c_6^2 \max\!\left\{\frac{1}{\delta^2 p k},1\right\} \frac{\log(n)}{n} + \frac{c_5 + 1}{n^5} \nonumber \\
& \leq 2 c_6^2 \max\!\left\{\frac{1}{\delta^2 p k},1\right\} \frac{\log(n)}{n}
\label{Eq: Intermediate Infinity Norm of Merit Estimate Bound}
\end{align} 
where the first equality uses the law of total expectation, the second inequality follows from \eqref{Eq: Main claim restatement} and the fact that $|\hat{\alpha}_i - \alpha_i| \leq 1$ for all $i \in [n]$ (see \eqref{Eq: Estimates of merit parameters}), and the final inequality holds for all sufficiently large $n$.

Substituting \eqref{Eq: Intermediate Infinity Norm of Merit Estimate Bound} into \eqref{Eq: MSE Decomp plus BV tradeoff} produces:
\begin{equation}
\label{Eq: Pre-exttremization bound}
\E\!\left[\int_{\R}{\left(\ests(x) - \pdf(x)\right)^{\! 2} \diff{x}}\right] \leq \frac{c_{10}}{nh} + c_{11} h^{2 \eta} + \frac{2 c_8 c_6^2 B^2 L_2^2}{h^2} \, \max\!\left\{\frac{1}{\delta^2 p k},1\right\} \frac{\log(n)}{n} + \frac{c_9 L_2^2}{n^5 h^4}
\end{equation} 
for all sufficiently large $n$. All that remains is to minimize this bound over the choice of $h$ and show that \eqref{Eq: Precise Bandwidth} provides the optimal bound. Since the first three terms on the right hand side of \eqref{Eq: Pre-exttremization bound} will dominate the fourth term, we focus on optimizing these three terms. Notice that the second term is monotone increasing in $h$, while the first and third terms are monotone decreasing in $h$. So, the optimal scaling of $h$ with $n$ can be obtained by either balancing the first and second terms, or balancing the second and third terms. To decide which pair of terms to balance, let us temporarily neglect $\delta$, $p$, and $k$. Then, if we balance the first and second terms, we get (see, e.g., \cite[Section 1.2.3]{Tsybakov2009}):
$$ \frac{c_{10}}{nh} = c_{11} h^{2 \eta} \quad \Leftrightarrow \quad h = \Theta\!\left(n^{-\frac{1}{2 \eta + 1}}\right) $$
which implies that the right hand side of \eqref{Eq: Pre-exttremization bound} is $\tilde{\Theta}\big(n^{-(2 \eta - 1)/(2 \eta + 1)}\big)$. On the other hand, if we balance the second and third terms, we get:
\begin{align*}
& c_{11} h^{2 \eta} = \frac{2 c_8 c_6^2 B^2 L_2^2}{h^2} \, \max\!\left\{\frac{1}{\delta^2 p k},1\right\} \frac{\log(n)}{n} \\
& \qquad \qquad \qquad \qquad \Leftrightarrow \quad h = \Theta\!\left(\max\!\left\{\frac{1}{\delta^{\frac{1}{\eta + 1}} (p k)^{\frac{1}{2 \eta + 2}}},1\right\} \left(\frac{\log(n)}{n}\right)^{\frac{1}{2 \eta + 2}} \right) 
\end{align*}
which implies that the right hand side of \eqref{Eq: Pre-exttremization bound} is $\tilde{\Theta}\big(n^{- \eta/(\eta + 1)}\big)$. Since $\frac{\eta}{\eta + 1} > \frac{2 \eta - 1}{2 \eta + 1}$ for all $\eta > 0$, balancing the second and third terms yields the tighter bound on the right hand side of \eqref{Eq: Pre-exttremization bound}.

We conclude this proof by explicitly balancing the second and third terms, and computing the precise resulting expression on the right hand side of \eqref{Eq: Pre-exttremization bound}. For any constant $\gamma > 0$, let the bandwidth be as defined in \eqref{Eq: Precise Bandwidth}:
$$ h = \gamma \max\!\left\{\frac{1}{\delta^{\frac{1}{\eta + 1}} (p k)^{\frac{1}{2 \eta + 2}}},1\right\} \left(\frac{\log(n)}{n}\right)^{\frac{1}{2 \eta + 2}} . $$
(Note that it is straightforward to verify that the condition $h = \Omega\big(\!\max\{1/(\delta \sqrt{p k}),1\} \sqrt{\log(n)/n}\big)$ is satisfied by \eqref{Eq: Precise Bandwidth}. Indeed, since we know that $\lim_{n \rightarrow \infty}{\delta^{-1} (n p k)^{-1/2} \log(n)^{1/2}} = 0$, we also have $\lim_{n \rightarrow \infty}{\max\{1/(\delta \sqrt{p k}),1\} \sqrt{\log(n)/n}} = 0$. So, we must have $\max\{1/(\delta \sqrt{p k}),1\} \sqrt{\log(n)/n} = O\big((\max\{1/(\delta \sqrt{p k}),1\} \sqrt{\log(n)/n})^{1/(\eta + 1)}\big)$.) Then, the terms on the right hand side of \eqref{Eq: Pre-exttremization bound} can be written as:
\begin{align*}
\frac{c_{10}}{nh} & = \frac{c_{10}}{\gamma} \min\!\left\{\delta^{\frac{1}{\eta + 1}} (p k)^{\frac{1}{2 \eta + 2}},1\right\} \log(n)^{-\frac{1}{2 \eta + 2}} n^{-\frac{2 \eta + 1}{2 \eta + 2}} \, , \\
c_{11} h^{2 \eta} & = c_{11} \gamma^{2 \eta} \max\!\left\{\delta^{-\frac{2\eta}{\eta + 1}} (p k)^{-\frac{\eta}{\eta + 1}},1\right\} \log(n)^{\frac{\eta}{\eta + 1}} n^{-\frac{\eta}{\eta + 1}} \, , \\
\frac{2 c_8 c_6^2 B^2 L_2^2}{h^2} \, \max\!\left\{\frac{1}{\delta^2 p k},1\right\} \frac{\log(n)}{n} & = \frac{2 c_8 c_6^2 B^2 L_2^2}{\gamma^2}  \max\!\left\{\delta^{-\frac{2 \eta}{\eta + 1}} (pk)^{-\frac{\eta}{\eta + 1}},1\right\} \log(n)^{\frac{\eta}{\eta + 1}} n^{-\frac{\eta}{\eta + 1}} \, , \\
\frac{c_9 L_2^2}{n^5 h^4} & = \frac{c_9 L_2^2}{\gamma^4}  \min\!\left\{\delta^{\frac{4}{\eta + 1}} (p k)^{\frac{2}{\eta + 1}},1\right\} \log(n)^{-\frac{2}{\eta + 1}}  n^{-\frac{5\eta + 3}{\eta + 1}} \, .
\end{align*}
Clearly, the second and third terms are balanced, and dominate the first and fourth terms on the right hand side of \eqref{Eq: Pre-exttremization bound} as $n$ grows. Since $\gamma$, $\eta$, $B$, and $L_2$ are constant parameters that do not depend on $n$, we have that for all sufficiently large $n \in \N$:
$$ \E\!\left[\int_{\R}{\left(\ests(x) - \pdf(x)\right)^{\! 2} \diff{x}}\right] \leq c_{12} \max\!\left\{ \left(\frac{1}{\delta^2 p k}\right)^{\frac{\eta}{\eta + 1}},1\right\} \left(\frac{\log(n)}{n}\right)^{\frac{\eta}{\eta + 1}} $$
where $c_{12} > 0$ is a sufficiently large constant that depends on $\gamma$, $\eta$, $B$, $L_1$, $L_2$, and the kernel $K$. This completes the proof.
\end{proof}

%% file: lowerbound.tex
\section{Minimax lower bounds via generalized Fano's method}
\label{App: Minimax lower bound via generalized Fano's method}

In this appendix, we establish the minimax bounds in Theorems \ref{Thm: Minimax Relative l^infty-Risk} and \ref{Thm: Minimax l^1-Risk}. In order to simplify the exposition, we first present the generalized Fano's method in subsection \ref{Generalized Fano's method}, derive three useful auxiliary lemmata in subsection \ref{Auxiliary lemmata 2}, and then present the proofs of Theorems \ref{Thm: Minimax Relative l^infty-Risk} and \ref{Thm: Minimax l^1-Risk} in subsections \ref{Proof of Thm Minimax Relative l^infty-Risk} and \ref{Proof of Thm Minimax l^1-Risk}, respectively. 

\subsection{Generalized Fano's method}
\label{Generalized Fano's method}

A canonical approach to obtaining minimax lower bounds in non-parametric estimation problems is the so called \emph{Fano's method}, which was introduced in \cite{Khasminskii1979,IbragimovKhasminskii1977} (also see, e.g., \cite{Yu1997} and \cite[Section 2.7.1]{Tsybakov2009} for modern treatments). Fano's method proceeds by first lower bounding minimax risk by a Bayes risk, where all the prior probability mass is placed over a suitably chosen (and large) finite set of parameters in the (non-parametric or infinite-dimensional) parameter space, then lower bounding this Bayes risk using the probability of error of a multiple hypothesis testing problem, and finally, lower bounding this probability of error using the well-known \emph{Fano's inequality} from information theory (cf. \cite[Theorem 2.10.1]{CoverThomas2006}). In the problem of estimating \eqref{Eq: Canonically scaled merit parameters} based on $Z$, the parameter space of the minimax risks in Theorems \ref{Thm: Minimax Relative l^infty-Risk} and \ref{Thm: Minimax l^1-Risk} is the infinite-dimensional family of PDFs $\pdfs$. For simplicity and analytical tractability, instead of directly applying Fano's method to this large parameter space $\pdfs$, which would involve constructing a prior distribution over some judiciously chosen finite subset of $\pdfs$, we first obtain lower bounds on the minimax risks in Theorems \ref{Thm: Minimax Relative l^infty-Risk} and \ref{Thm: Minimax l^1-Risk} in terms of Bayes risks. In particular, as discussed earlier, we set $\pdf = \unif([\delta,1]) \in \pdfs$ throughout subsections \ref{Auxiliary lemmata 2}, \ref{Proof of Thm Minimax Relative l^infty-Risk}, and \ref{Proof of Thm Minimax l^1-Risk}, so that $\alpha_1,\dots,\alpha_n$ are i.i.d. $\pdf = \unif([\delta,1])$. Hence, $\P(\cdot)$ denotes the joint probability law of $\alpha_1,\dots,\alpha_n$, $\{Z_m(i,j) : i,j \in [n], \, i < j, \, m \in  [k]\}$, and $\G(n,p)$ with $\pdf = \unif([\delta,1])$ in the sequel, and $\E[\cdot]$ denotes the corresponding expectation operator. This yields the following lower bound on the minimax relative $\ell^{q}$-norm risk for any $q \in [1,+\infty]$:
\begin{equation}
\label{Eq: Bayes Risk Lower Bound}
\inf_{\hat{\pi}}{\sup_{\pdf \in \pdfs}{\E_{\pdf}\!\left[\frac{\left\|\hat{\pi} - \pi\right\|_{q}}{\left\|\pi\right\|_{q}}\right]}} \geq
\inf_{\hat{\pi}}{\E\!\left[\frac{\left\|\hat{\pi} - \pi\right\|_{q}}{\left\|\pi\right\|_{q}}\right]}
\end{equation}
where the infima are over all (measurable) randomized estimators $\hat{\pi} \in \S_n$ of the canonically scaled skill parameters $\pi$ based on the observation matrix $Z$, and $\E_{P_{\alpha}}[\cdot]$ denotes the expectation operator with respect to general (not necessarily uniform) $\pdf$. Clearly, letting $q = +\infty$ and $q = 1$ yield the minimax problems in Theorems \ref{Thm: Minimax Relative l^infty-Risk} and \ref{Thm: Minimax l^1-Risk}, respectively. Therefore, we can focus on the simpler problem of lower bounding the Bayes risks on the right hand side of \eqref{Eq: Bayes Risk Lower Bound} for $q \in \{1,+\infty\}$.

Unfortunately, while Fano's method is very effective at lower bounding non-parametric minimax risks, it cannot lower bound Bayes risks where the parameter space is not a discrete and finite set, because the classical Fano's inequality only holds for discrete and finite parameter sets, cf. \cite[Theorem 2.10.1]{CoverThomas2006}. To remedy this dearth of Fano-based techniques to lower bound Bayes risks where the parameter space is a continuum, the so called \emph{generalized Fano's method} has been developed in the recent literature \cite{Zhang2006,ChenGuntuboyinaZhang2016,XuRaginsky2017}. One of the first results in this line of work was a generalization of Fano's inequality to the \textit{continuum Fano inequality} in \cite[Proposition 2]{DuchiWainwright2013}, which had useful consequences for minimax estimation with a specific zero-one valued loss function \cite[Section 3]{DuchiWainwright2013}. The techniques of \cite{DuchiWainwright2013} have been vastly generalized in \cite{ChenGuntuboyinaZhang2016} and \cite{XuRaginsky2017} to obtain lower bounds on Bayes risks in terms of \emph{$f$-informativity}, cf. \cite{Csiszar1972}, and conditional mutual information (with auxiliary random variables), respectively. In this paper, we will utilize the key result in \cite[Theorem 1, Equation (6)]{XuRaginsky2017}. The lemma below presents the result in \cite[Theorem 1, Equation (6)]{XuRaginsky2017} specialized to our relative $\ell^{q}$-loss setting. 

\begin{lemma}[Generalized Fano's Method {\cite[Theorem 1]{XuRaginsky2017}}]
\label{Lemma: Generalized Fano's Method}
For any $q \in [1,+\infty]$, the Bayes risk on the right hand side of \eqref{Eq: Bayes Risk Lower Bound} is lower bounded by:
$$ \inf_{\hat{\pi}}{\E\!\left[\frac{\left\|\hat{\pi} - \pi\right\|_{q}}{\left\|\pi\right\|_{q}}\right]} \geq \sup_{t > 0}{ \, t \!\left(1 - \frac{I(\pi ; Z) + \log(2)}{\log(1/\L_q(t))}\right)} $$
where we define the \emph{small ball probability} $\L_q(\cdot)$ as (cf. \cite[Equation (2)]{XuRaginsky2017}):
\begin{equation}
\label{Eq: Small Ball Prob Def}
\forall t > 0, \enspace \L_q(t) \triangleq \sup_{\nu \in \S_n}{\P\!\left(\frac{\left\|\pi - \nu\right\|_q}{\left\|\pi\right\|_q} \leq t\right)} \, , 
\end{equation}
and $I(\pi ; Z)$ denotes the mutual information (defined in \eqref{Eq: MI Def}) between the canonically scaled skill parameters $\pi$ (defined in \eqref{Eq: Canonically scaled merit parameters}) and the observation matrix $Z$ (defined in \eqref{Eq: Observation matrix}).
\end{lemma} 

We remark that several variants of \lemref{Lemma: Generalized Fano's Method} exist in the literature, such as \cite[Theorem 6.1]{Zhang2006} and \cite[Remark 10, Corollary 12(i)]{ChenGuntuboyinaZhang2016}. As expounded in \cite{ChenGuntuboyinaZhang2016}, in order to compute lower bounds such as that in \lemref{Lemma: Generalized Fano's Method}, we need to establish two things:
\begin{enumerate}
\item Tight upper bounds on the mutual information $I(\pi ; Z)$.
\item Tight upper bounds on the small ball probability $\L_q(t)$.
\end{enumerate}
We have already derived an upper bound on $I(\pi ; Z)$ in \propref{Prop: Upper Bound on MI} using the covering number argument presented in subsection \ref{App: Covering numbers}. Next, 
we prove upper bounds on the small ball probability $\L_q(t)$ for $q \in \{1,+\infty\}$.

\subsection{Upper bounds on small ball probability}
\label{Auxiliary lemmata 2}

As noted in both \cite{ChenGuntuboyinaZhang2016} and \cite{XuRaginsky2017}, there is no general recipe for obtaining upper bounds on $\L_q(t)$. So, we develop our bounds via direct computation. To this end, the ensuing lemma presents an upper bound on the mode of the joint PDF of $\pi$, or more precisely, the joint PDF of:
\begin{equation}
\label{Eq: Projected pi}
\tilde{\pi} \triangleq (\pi(1),\dots,\pi(n-1))
\end{equation}
with respect to the Lebesgue measure on $\R^{n-1}$, which excludes $\pi(n)$, because $\pi(n) = 1 - \pi(1) - \cdots - \pi(n-1)$.

\begin{lemma}[Bound on Mode of Joint PDF of $\tilde{\pi}$]
\label{Lemma: Bound on Mode of Joint PDF of pi}
Let the joint PDF of $\tilde{\pi}$ with respect to the Lebesgue measure on $\R^{n-1}$ be denoted $P_{\tilde{\pi}}$. The mode of $P_{\tilde{\pi}}$ is upper bounded by:
$$ \esssup_{\tau \in \R^{n-1}}{\, P_{\tilde{\pi}}(\tau)} \leq \frac{n^{n-1}}{(1-\delta)^n} $$
where $\esssup$ denotes the essential supremum.
\end{lemma}

\begin{proof}
First, consider the map $h : [\delta,1]^n \rightarrow \image(h)$:
$$ \forall \beta = (\beta_1,\dots,\beta_n) \in [\delta,1]^n, \enspace h(\beta) \triangleq \left(\frac{\beta_1}{\sum_{i = 1}^{n}{\beta_i}},\dots,\frac{\beta_{n-1}}{\sum_{i = 1}^{n}{\beta_i}},\sum_{i = 1}^{n}{\beta_i}\right) $$
where $\image(h) \triangleq \big\{(\tau_1,\dots,\tau_{n-1},\sigma) \in
\big[\frac{\delta}{n - 1 + \delta}, \frac{1}{1 + \delta(n-1)}\big]^{n-1} \times [n\delta,n] : \exists \, \beta_1,\dots,\beta_n \in [\delta,1], \, \sigma = \sum_{j = 1}^{n}{\beta_j}\text{ and }\forall i \allowbreak \in [n-1], \, \tau_i = \beta_i/\sigma\big\}$ denotes the range (or image) of $h$. Clearly, we have $h(\alpha_1,\dots,\alpha_n) = (\tilde{\pi},\alpha_1 + \cdots + \alpha_n)$ using \eqref{Eq: Canonically scaled merit parameters} and \eqref{Eq: Projected pi}. Furthermore, $h$ is a bijection with inverse function $h^{-1} : \image(h) \rightarrow [\delta,1]^n$:
$$ \forall (\tau_1,\dots,\tau_{n-1},\sigma) \in \image(h) , \enspace h^{-1}(\tau_1,\dots,\tau_{n-1},\sigma) = \left(\sigma \tau_1,\dots,\sigma \tau_{n-1},\sigma\!\left(1 - \sum_{i = 1}^{n-1}{\tau_i}\right)\right) . $$
By direct evaluation, the \emph{Jacobian matrix} of $h^{-1}$, denoted $\nabla h^{-1} : \image(h) \rightarrow \R^{n \times n}$, is:
$$ \forall (\tau_1,\dots,\tau_{n-1},\sigma) \in \image(h) , \enspace \left[\nabla h^{-1}\right]_{i,j} = 
\begin{cases}
\sigma \I\{i = j\} \, , & i,j \in [n-1] \\
-\sigma \, , & i = n, \, j \in [n-1] \\
\tau_i \, , & i \in [n-1], \, j = n \\
\displaystyle{1 - \sum_{i = 1}^{n-1}{\tau_i}} \, , & i = j = n
\end{cases}
$$
where $[\nabla h^{-1}]_{i,j}$ denotes the $(i,j)$th entry of the matrix $\nabla h^{-1}$ for $i,j \in [n]$. (Note that $\nabla h^{-1}$ is also well-defined on the boundary of $\image(h)$, because there exists an open set containing $\image(h)$ such that the first partial derivatives of $h^{-1}$ exist on this open set.) Now define the successive sub-matrices:
$$ \forall r \in \{0,1,\dots,n-2\}, \enspace M_{n-r} \triangleq \left[
\begin{array}{ccccc}
1 & 0 & \cdots & 0 & \tau_{r+1} \\
0 & 1 & \cdots & 0 & \tau_{r+2} \\
\vdots & \vdots & \ddots & \vdots & \vdots \\
0 & 0 & \cdots & 1 & \tau_{n-1} \\
-1 & -1 & \cdots & -1 & 1 - \sum_{i = 1}^{n-1}{\tau_i}
\end{array}
\right] \in \R^{(n-r)\times(n-r)} $$
where $M_n$ is closely related to $\nabla h^{-1}$ (as shown below in \eqref{Eq: Multilinear}), and let the transpose of the \emph{Frobenius companion matrix} of the monic polynomial $q_n(t) = 1 + t + t^2 + \cdots + t^n$ be (cf. \cite[Definition 3.3.13]{HornJohnson2013}):
$$ C_n \triangleq \left[
\begin{array}{cccccc}
0 & 1 & 0 & \cdots & 0 \\
0 & 0 & 1 & \cdots & 0 \\
\vdots & \vdots & \vdots & \ddots & \vdots \\
0 & 0 & 0 & \cdots & 1 \\
-1 & -1 & -1 & \cdots & -1 
\end{array}
\right] \in \R^{n \times n} \, .
$$
Then, the corresponding \emph{Jacobian determinant} satisfies the recurrence relation:
\begin{align}
\forall (\tau_1,\dots,\tau_{n-1},\sigma) \in \image(h) , \enspace \det\!\left(\nabla h^{-1}\right) & = \det\!\left(\left[
\begin{array}{ccccc}
\sigma & 0 & \cdots & 0 & \tau_1 \\
0 & \sigma & \cdots & 0 & \tau_2 \\
\vdots & \vdots & \ddots & \vdots & \vdots \\
0 & 0 & \cdots & \sigma & \tau_{n-1} \\
-\sigma & -\sigma & \cdots & -\sigma & 1 - \sum_{i = 1}^{n-1}{\tau_i}
\end{array}
\right]\right) \nonumber \\
& = \sigma^{n-1} \det(M_n) \label{Eq: Multilinear} \\
& = \sigma^{n-1} \left(\det(M_{n-1}) + (-1)^{n+1} \tau_1 \det(C_{n-1}) \right) \label{Eq: Laplace} 
\end{align} 
where $\det(\cdot)$ denotes the determinant operator, \eqref{Eq: Multilinear} follows from the multilinearity of the determinant, and \eqref{Eq: Laplace} uses the \emph{Laplace (cofactor) expansion} of determinants by minors along the first row (see, e.g., \cite[Section 0.3.1]{HornJohnson2013}). We next compute this Jacobian determinant.

It is easy to calculate $\det(C_{n-1})$ in \eqref{Eq: Laplace}, because $q_{n-1}(t)$ is also the characteristic polynomial of its (adjoint) companion matrix $C_{n-1}$, cf. \cite[Theorem 3.3.14]{HornJohnson2013}. The $n-1$ distinct roots of $q_{n-1}(t)$ are the following $n$th roots of unity:
$$ \forall r \in [n-1], \enspace q_{n-1}\big(\omega^r\big) = 0 $$
where $\omega = \exp\!\big(\frac{2 \pib \iotab}{n}\big)$. (Note that unlike the rest of this paper, in the definition of $\omega$, we use $\iotab$ and $\pib$ to represent the imaginary unit $\iotab = \sqrt{-1}$ and the mathematical constant $\pib = 3.14159\dots$, respectively.) Hence, $\big\{\omega^r : r \in [n-1]\big\}$ are the eigenvalues of $C_{n-1}$, and we have:
$$ \det(C_{n-1}) = \prod_{r = 1}^{n-1}{\omega^r} = \omega^{n(n-1)/2} = (-1)^{n-1}$$
since the determinant is the product of the eigenvalues (see, e.g., \cite[Section 1.2]{HornJohnson2013}). Combining this with \eqref{Eq: Laplace}, we get:
\begin{align}
\forall (\tau_1,\dots,\tau_{n-1},\sigma) \in \image(h) , \enspace \det\!\left(\nabla h^{-1}\right) & = \sigma^{n-1} \left(\det(M_{n-1}) + (-1)^{n+1} \tau_1 (-1)^{n-1} \right) \nonumber \\
& = \sigma^{n-1} \underbrace{\left(\det(M_{n-1}) + \tau_1 \right)}_{= \, \det(M_n)} \nonumber \\
& = \sigma^{n-1} \left(\det(M_{2}) + \sum_{j = 1}^{n-2}{\tau_j}\right) \nonumber \\
& = \sigma^{n-1} \left(1 - \sum_{j = 1}^{n-1}{\tau_j} + \tau_{n-1} + \sum_{j = 1}^{n-2}{\tau_j}\right) \nonumber \\
& = \sigma^{n-1}
\label{Eq: Jacobian det}
\end{align} 
where the third equality follows from unwinding the recursion in the second line. 

Now observe that the joint PDF of $(\alpha_1,\dots,\alpha_n)$ (with respect to the Lebesgue measure on $\R^{n}$) is given by:
$$ \forall \beta \in \R^n, \enspace P_{\alpha_1,\dots,\alpha_n}(\beta) = \frac{1}{(1-\delta)^n} \, \I\!\left\{\beta \in [\delta,1]^n\right\} $$
since $\alpha_1,\dots,\alpha_n$ are i.i.d. $\pdf = \unif([\delta,1])$. As a consequence, the joint PDF of $h(\alpha_1,\dots,\alpha_n) = (\tilde{\pi},\alpha_1+\cdots+\alpha_n)$ (with respect to the Lebesgue measure on $\R^{n}$) is given by the \emph{change-of-variables formula}: 
\begin{align}
P_{\tilde{\pi},\alpha_1+\cdots+\alpha_n}(\tau_1,\dots,\tau_{n-1},\sigma) & = P_{\alpha_1,\dots,\alpha_n}(h^{-1}(\tau_1,\dots,\tau_{n-1},\sigma)) \left|\det\!\left(\nabla h^{-1}\right)\right| \nonumber \\
& = \frac{\sigma^{n-1}}{(1-\delta)^n} \, \I\!\left\{(\tau_1,\dots,\tau_{n-1},\sigma) \in \image(h)\right\}
\label{Eq: Joint PDF Calculation}
\end{align}
for all $(\tau_1,\dots,\tau_{n-1},\sigma) \in \R^n$, where we utilize our earlier computation of the Jacobian determinant in \eqref{Eq: Jacobian det}. Although we only seek to bound the joint PDF of $\tilde{\pi}$, the joint PDF in \eqref{Eq: Joint PDF Calculation} includes an additional random variable $\alpha_1+\cdots+\alpha_n$ as an artifact of our calculation approach (which requires an invertible map $h$ with a well-defined and invertible Jacobian matrix $\nabla h$). 

So, in the final step of this proof, we marginalize the joint PDF in \eqref{Eq: Joint PDF Calculation} and then bound the desired joint PDF of $\tilde{\pi}$ (with respect to the Lebesgue measure on $\R^{n-1}$):
\begin{align*}
\forall \tau \in \R^{n-1}, \enspace P_{\tilde{\pi}}(\tau) & = \I\!\big\{\tau \in \tilde{\S}_n\big\} \int_{[n\delta,n]}{P_{\tilde{\pi},\alpha_1+\cdots+\alpha_n}(\tau,\sigma) \, \diff \sigma} \\
& = \frac{\I\!\big\{\tau \in \tilde{\S}_n\big\}}{(1-\delta)^n} \int_{[n\delta,n]}{\sigma^{n-1} \, \I\!\left\{(\tau,\sigma) \in \image(h)\right\} \diff \sigma} \\
& \leq \frac{\I\!\big\{\tau \in \tilde{\S}_n\big\}}{(1-\delta)^n} \int_{[n\delta,n]}{\sigma^{n-1} \, \diff \sigma} \\
& = \frac{n^{n-1} (1 - \delta^n)}{(1-\delta)^n} \, \I\!\big\{\tau \in \tilde{\S}_n\big\} \\
& \leq \frac{n^{n-1}}{(1-\delta)^n} \, \I\!\big\{\tau \in \tilde{\S}_n\big\}
\end{align*} 
where $\tilde{\S}_n \triangleq \big\{(\tau_1,\dots,\tau_{n-1}) \in \big[\frac{\delta}{n - 1 + \delta},\frac{1}{1 + \delta(n-1)}\big]^{n-1} : \exists \, \beta_1,\dots,\beta_n \in [\delta,1], \, \forall i \allowbreak \in [n-1], \, \tau_i = \beta_i/\big(\sum_{j = 1}^{n}{\beta_j}\big)\big\}$. Taking the (essential) supremum over all $\tau \in \R^{n-1}$ in the above bound completes the proof.
\end{proof}

We now use \lemref{Lemma: Bound on Mode of Joint PDF of pi} to upper bound the small ball probabilities $\L_q(t)$ for $q \in \{1,+\infty\}$ in the lemmata below.

\begin{lemma}[Upper Bound on Small Ball Probability for $q=\infty$]
\label{Lemma: Upper Bound on Small Ball Probability}
For every $t > 0$, we have:
$$ \L_{\infty}(t) \leq \left(\frac{2}{\delta (1-\delta)}\right)^{\! n} t^{n-1} \, . $$
\end{lemma}

\begin{proof}
Starting with \eqref{Eq: Small Ball Prob Def}, observe that:
\begin{align*}
\forall t > 0, \enspace \L_{\infty}(t) & = \sup_{\nu = (\nu_1,\dots,\nu_n) \in \S_n}{\P\!\left(\left\|\pi - \nu\right\|_{\infty} \leq t \left\|\pi\right\|_{\infty}\right)} \\
& \leq \sup_{\nu = (\nu_1,\dots,\nu_n) \in \S_n}{\P\!\left(\left\|\pi - \nu\right\|_{\infty} \leq \frac{t}{\delta n}\right)} \\
& \leq \sup_{\nu = (\nu_1,\dots,\nu_n) \in \S_n}{\P\!\left(\left\|\tilde{\pi} - (\nu_1,\dots,\nu_{n-1})\right\|_{\infty} \leq \frac{t}{\delta n}\right)} \\
& = \sup_{\nu = (\nu_1,\dots,\nu_n) \in \S_n}{\int_{\R^{n-1}}{P_{\tilde{\pi}}(\tau) \, \I\!\left\{\left\|\tau - (\nu_1,\dots,\nu_{n-1})\right\|_{\infty} \leq \frac{t}{\delta n}\right\} \diff \tau}} \\
& \leq \frac{n^{n-1}}{(1-\delta)^n} \sup_{\nu = (\nu_1,\dots,\nu_n) \in \S_n}{\underbrace{\int_{\R^{n-1}}{\I\!\left\{\left\|\tau - (\nu_1,\dots,\nu_{n-1})\right\|_{\infty} \leq \frac{t}{\delta n}\right\} \diff \tau}}_{\text{volume of } \ell^{\infty}\text{-ball with radius } t/(\delta n)}} \\
& = \frac{n^{n-1}}{(1-\delta)^n} \left(\frac{2 t}{\delta n}\right)^{\! n-1} \\
& \leq \left(\frac{2}{\delta (1-\delta)}\right)^{\! n} t^{n-1}
\end{align*}
where the second inequality uses the bound:
$$ \left\|\pi\right\|_{\infty} = \frac{\max_{i \in [n]}{\alpha_i}}{\sum_{i = 1}^{n}{\alpha_i}} \leq \frac{1}{\delta n} $$
which follows from \eqref{Eq: Canonically scaled merit parameters} and the fact that $\alpha_1,\dots,\alpha_n \in [\delta,1]$, the third inequality uses \eqref{Eq: Projected pi} and the fact that $\|\tilde{\pi} - (\nu_1,\dots,\nu_{n-1})\|_{\infty} \leq \|\pi - \nu\|_{\infty}$, the fifth inequality follows from \lemref{Lemma: Bound on Mode of Joint PDF of pi}, the sixth equality uses the well-known volume of the $\ell^{\infty}$-ball (or hypercube) with radius $t/(\delta n)$, and the seventh inequality follows from the fact that $2/\delta \geq 1$. This completes the proof.
\end{proof}

\begin{lemma}[Upper Bound on Small Ball Probability for $q=1$]
\label{Lemma: Upper Bound on Small Ball Probability 2}
For every $t > 0$, we have:
$$ \L_1(t) \leq \frac{1}{5 \sqrt{n}} \left(\frac{2 e}{1-\delta}\right)^{\! n} t^{n-1} \, . $$
\end{lemma}

\begin{proof}
As before, starting with \eqref{Eq: Small Ball Prob Def} and the fact that $\|\pi\|_1 = 1$, observe that:
\begin{align*}
\forall t > 0, \enspace \L_1(t) & = \sup_{\nu = (\nu_1,\dots,\nu_n) \in \S_n}{\P\!\left(\left\|\pi - \nu\right\|_1 \leq t\right)} \\
& \leq \sup_{\nu = (\nu_1,\dots,\nu_n) \in \S_n}{\P\!\left(\left\|\tilde{\pi} - (\nu_1,\dots,\nu_{n-1})\right\|_1 \leq t\right)} \\
& = \sup_{\nu = (\nu_1,\dots,\nu_n) \in \S_n}{\int_{\R^{n-1}}{P_{\tilde{\pi}}(\tau) \, \I\!\left\{\left\|\tau - (\nu_1,\dots,\nu_{n-1})\right\|_1 \leq t\right\} \diff \tau}} \\
& \leq \frac{n^{n-1}}{(1-\delta)^n} \sup_{\nu = (\nu_1,\dots,\nu_n) \in \S_n}{\underbrace{\int_{\R^{n-1}}{\I\!\left\{\left\|\tau - (\nu_1,\dots,\nu_{n-1})\right\|_1 \leq t\right\} \diff \tau}}_{\text{volume of } \ell^{1}\text{-ball with radius } t}} \\
& = \frac{n^{n-1} 2^{n-1}}{(1-\delta)^n (n-1)!} \, t^{n-1} \\
& \leq \frac{1}{5 \sqrt{n}} \left(\frac{2 e}{1-\delta}\right)^{\! n} t^{n-1}
\end{align*}
where the second inequality uses \eqref{Eq: Projected pi} and the fact that $\|\tilde{\pi} - (\nu_1,\dots,\nu_{n-1})\|_1 \leq \|\pi - \nu\|_1$ (and this bound is not too loose because $\|\pi - \nu\|_1 \leq 2 \, \|\tilde{\pi} - (\nu_1,\dots,\nu_{n-1})\|_1$ via the triangle inequality), the fourth inequality follows from \lemref{Lemma: Bound on Mode of Joint PDF of pi}, the fifth equality uses the well-known volume of the $\ell^1$-ball (or cross-polytope) with radius $t$, and the sixth inequality follows from the Stirling's formula bound (see, e.g., \cite[Chapter II, Section 9, Equation (9.15)]{Feller1968}): 
$$ n! \geq \frac{5}{2} \sqrt{n} \left(\frac{n}{e}\right)^{\! n} . $$
This completes the proof.
\end{proof}

Next, we provide proofs of Theorems \ref{Thm: Minimax Relative l^infty-Risk} and \ref{Thm: Minimax l^1-Risk} using \propref{Prop: Upper Bound on MI}, Lemmata \ref{Lemma: Relative ell^infty-Loss Bound}, \ref{Lemma: Generalized Fano's Method}, \ref{Lemma: Upper Bound on Small Ball Probability}, and \ref{Lemma: Upper Bound on Small Ball Probability 2}, and the result in \cite[Theorem 5.2]{Chenetal2019}.

\subsection{Proof of \thmref{Thm: Minimax Relative l^infty-Risk}}
\label{Proof of Thm Minimax Relative l^infty-Risk}

\begin{proof}
We first prove the minimax upper bound. The inequality:
$$ \inf_{\hat{\pi}}{\sup_{\pdf \in \pdfs}{\E_{\pdf}\!\left[\frac{\left\|\hat{\pi} - \pi\right\|_{\infty}}{\left\|\pi\right\|_{\infty}}\right]}} \leq \sup_{\pdf \in \pdfs}{\E_{\pdf}\!\left[\frac{\left\|\hat{\pi}_* - \pi\right\|_{\infty}}{\left\|\pi\right\|_{\infty}}\right]} $$
holds trivially, because $\hat{\pi}_* \in \S_n$ in \eqref{Eq: Empirical Stat Dist} is an estimator for $\pi$ based on $Z$. To prove an upper bound on the extremal Bayes risk on the right hand side of this inequality, we define the event in \lemref{Lemma: Relative ell^infty-Loss Bound} as:
$$ A \triangleq \left\{\frac{\left\|\hat{\pi}_* - \pi\right\|_{\infty}}{\left\|\pi\right\|_{\infty}} \leq \frac{c_4}{\delta} \sqrt{\frac{\log(n)}{n p k}} \right\} $$
where $c_4 > 0$ is the universal constant from \lemref{Lemma: Relative ell^infty-Loss Bound}. Then, \eqref{Eq: Chen event} states that for any PDF $\pdf \in \pdfs$ and for all sufficiently large $n \in \N$:
\begin{equation}
\label{Eq: Lemma Restatement}
\P_{\pdf}\!\left(A\right) \geq 1 - \frac{c_5}{n^5} 
\end{equation}
where $c_5 > 0$ is another universal constant from \lemref{Lemma: Relative ell^infty-Loss Bound}, and $\P_{\pdf}(\cdot)$ denotes the probability measure with respect to general (not necessarily uniform) $\pdf$. Hence, for every PDF $\pdf \in \pdfs$ and for all sufficiently large $n$, we have:
\begin{align}
\E_{\pdf}\!\left[\frac{\left\|\hat{\pi}_* - \pi\right\|_{\infty}}{\left\|\pi\right\|_{\infty}}\right] & = \E_{\pdf}\!\left[\frac{\left\|\hat{\pi}_* - \pi\right\|_{\infty}}{\left\|\pi\right\|_{\infty}}\middle|A\right] \P_{\pdf}(A) + \E_{\pdf}\!\left[\frac{\left\|\hat{\pi}_* - \pi\right\|_{\infty}}{\left\|\pi\right\|_{\infty}}\middle|A^{c}\right] \P_{\pdf}(A^{c}) \nonumber \\
& \leq \frac{c_4}{\delta} \sqrt{\frac{\log(n)}{n p k}} + \frac{c_5}{n^4} \nonumber \\
& \leq \frac{2c_4}{\delta} \sqrt{\frac{\log(n)}{n p k}}
\label{Eq: Intermediate Extremal Risk Bound}
\end{align} 
where the first equality uses the law of total expectation, the second inequality follows from \eqref{Eq: Lemma Restatement} and the facts that $\|\pi\|_{\infty} \geq 1/n$ and $\|\hat{\pi}_{*} - \pi\|_{\infty} \leq 1$, and the third inequality \eqref{Eq: Intermediate Extremal Risk Bound} holds for all sufficiently large $n$ because $k = \Theta(1)$. Letting $c_{15} = 2 c_4/(\delta \sqrt{p k})$ and substituting it into \eqref{Eq: Intermediate Extremal Risk Bound}, and then taking the supremum in \eqref{Eq: Intermediate Extremal Risk Bound} over all PDFs $\pdf \in \pdfs$ yields the desired upper bound in the theorem statement. 

We next prove the information theoretic lower bound. Fix any $\varepsilon > 0$, and consider any sufficiently large $n \geq 2$ such that:
\begin{equation}
\label{Eq: Pre Lower Bound Condition Large n 1} 
n \geq \max\!\left\{2 + \frac{1}{\varepsilon}, \left(\frac{2 }{\delta(1-\delta)}\right)^{\! 4/\varepsilon} \! , \, \exp\!\left(\frac{(1-\delta)^2 \!\left( 2 + \delta + \frac{1}{\delta}\right) k p + 4 \log(2)\delta^2}{\delta^2 \varepsilon}\right)\right\} .
\end{equation}
Then, observe that:
\begin{align}
\inf_{\hat{\pi}}{\sup_{\pdf \in \pdfs}{\E_{\pdf}\!\left[\frac{\left\|\hat{\pi} - \pi\right\|_{\infty}}{\left\|\pi\right\|_{\infty}}\right]}} & \geq \sup_{t > 0}{ \, t \!\left(1 - \frac{I(\pi ; Z) + \log(2)}{\log(1/\L_{\infty}(t))}\right)} \nonumber \\
& \geq \sup_{t > 0}{ \, t \!\left(1 - \frac{\frac{1}{2} n \log(n) + c(\delta,p,k) n + \log(2)}{\log(1/\L_{\infty}(t))}\right)} \nonumber \\
& \geq \sup_{t > 0}{ \, t \!\left(1 - \frac{\frac{1}{2} n \log(n) + c(\delta,p,k) n + \log(2)}{(n-1)\log(1/t) - \log(2 / (\delta (1-\delta))) n}\right)} \nonumber \\
& = \sup_{t > 0}{ \, t \!\left(1 - \frac{1 + \frac{2 c(\delta,p,k)}{\log(n)} + \frac{\log(4)}{n \log(n)}}{\frac{2 (n-1)\log(1/t)}{n \log(n)} - \frac{2 \log(2 / (\delta (1-\delta)))}{\log(n)}}\right)} \nonumber \\
& \geq \frac{1}{n^{\frac{1}{2} + \varepsilon}} \!\left(1 - \frac{1 + \frac{2 c(\delta,p,k)}{\log(n)} + \frac{\log(4)}{n \log(n)}}{(1 + 2 \varepsilon) \! \left(1 - \frac{1}{n}\right) - \frac{2 \log(2 / (\delta (1-\delta)))}{\log(n)}}\right) \nonumber \\
& \geq \frac{1}{n^{\frac{1}{2} + \varepsilon}} \!\left(1 - \frac{1 + \frac{\varepsilon}{4}}{1 + \frac{\varepsilon}{2}}\right) \nonumber \\
& = \frac{1}{n^{\frac{1}{2} + \varepsilon}} \!\left(\frac{\varepsilon}{4 + 2 \varepsilon}\right)
\label{Eq: Pre Lower Bound 1}
\end{align}
where the first inequality follows from \eqref{Eq: Bayes Risk Lower Bound} and \lemref{Lemma: Generalized Fano's Method}, the second inequality follows from \propref{Prop: Upper Bound on MI} and we let:
\begin{equation}
\label{Eq: Constant in the MI Prop}
c(\delta,p,k) = \frac{(1-\delta)^2}{8 \delta^2} \left(2 + \delta + \frac{1}{\delta}\right) k p
\end{equation}
for clarity, the third inequality holds due to \lemref{Lemma: Upper Bound on Small Ball Probability}, the fifth inequality follows from setting $t = n^{-(1/2) - \varepsilon}$, and the sixth inequality follows from \eqref{Eq: Pre Lower Bound Condition Large n 1}, which implies the following bounds:
\begin{align}
\label{Eq: Detailed Large n condition 1}
n \geq \exp\!\left(\frac{(1-\delta)^2 \!\left( 2 + \delta + \frac{1}{\delta}\right) k p + 4 \log(2)\delta^2}{\delta^2 \varepsilon}\right) \quad & \Rightarrow \quad \frac{2 c(\delta,p,k)}{\log(n)} + \frac{\log(4)}{n \log(n)} \leq \frac{\varepsilon}{4} \, , \\
\label{Eq: Detailed Large n condition 2}
n \geq 2 + \frac{1}{\varepsilon} \quad & \Leftrightarrow \quad 1 - \frac{1}{n} \geq \frac{1 + \varepsilon}{1 + 2\varepsilon} \, ,  \\
n \geq \left(\frac{2}{\delta(1-\delta)}\right)^{\! 4/\varepsilon} \quad & \Leftrightarrow \quad \frac{2 \log\!\left(\frac{2 }{\delta (1-\delta)}\right)}{\log(n)} \leq \frac{\varepsilon}{2} \, . \nonumber
\end{align}

Now, let us define the constant:
$$ c_{16} = \max\!\left\{4 \log\!\left(\frac{2}{\delta(1 - \delta)}\right),\frac{(1-\delta)^2 \!\left( 2 + \delta + \frac{1}{\delta}\right) k p + 4 \log(2)\delta^2}{\delta^2}\right\} $$
and set $\varepsilon = c_{16}/\!\log(n)$. It is straightforward to verify that \eqref{Eq: Pre Lower Bound Condition Large n 1} is satisfied for this choice of $\varepsilon$ for all sufficiently large $n$. Moreover, since $\varepsilon \leq 1$ for all sufficiently large $n$, we see that \eqref{Eq: Pre Lower Bound 1} can be recast as:
$$ \inf_{\hat{\pi}}{\sup_{\pdf \in \pdfs}{\E_{P_{\alpha}}\!\left[\frac{\left\|\hat{\pi} - \pi\right\|_{\infty}}{\left\|\pi\right\|_{\infty}}\right]}} \geq \frac{c_{16}}{6 \log(n) \, n^{\frac{1}{2} + \frac{c_{16}}{\log(n)}}} = \left(\frac{c_{16}}{6 \exp(c_{16})}\right) \frac{1}{\log(n) \sqrt{n}} $$
for all sufficiently large $n$. Finally, letting $c_{14} = c_{16}/(6 \exp(c_{16}))$ yields the minimax lower bound in the theorem statement. This completes the proof.
\end{proof}

\subsection{Proof of \thmref{Thm: Minimax l^1-Risk}}
\label{Proof of Thm Minimax l^1-Risk}

Although the upper bound in \thmref{Thm: Minimax l^1-Risk} can be established using \lemref{Lemma: Relative ell^infty-Loss Bound}, we can remove an extra $\sqrt{\log(n)}$ factor by utilizing \cite[Theorem 5.2]{Chenetal2019} (also see \cite[Theorem 2]{NegahbanOhShah2017}). So, we present this result in the lemma below.

\begin{lemma}[Relative $\ell^{2}$-Loss Bound {\cite[Theorem 5.2]{Chenetal2019}}]
\label{Lemma: Relative ell^2-Loss Bound}
Suppose that $\delta = \Theta(1)$ and $p \geq c_{19} \log(n)/n$ for some sufficiently large constant $c_{19} > 0$ (which may depend on $\delta$). Then, there exists a constant $c_{20} > 0$ (which may depend on $\delta$) and a (universal) constant $c_{21} > 0$ such that for all sufficiently large $n \in \N$, we have:
$$ \P\!\left(\frac{\left\|\hat{\pi}_* - \pi\right\|_{2}}{\left\|\pi\right\|_{2}} \leq \frac{c_{20}}{\sqrt{n p k}} \, \middle| \, \alpha_1,\dots,\alpha_n\right) \geq 1 - \frac{c_{21}}{n^{5}} $$
where the probability is computed with respect to the conditional distribution of the observation matrix $Z$ and the random graph $\G(n,p)$ given any realizations of the skill parameters $\alpha_1,\dots,\alpha_n$, and the estimator $\hat{\pi}_* \in \S_n$ is defined in \eqref{Eq: Empirical Stat Dist}. 
\end{lemma}

This lemma is an analog of \lemref{Lemma: Relative ell^infty-Loss Bound}, but for $\ell^2$-norm 
instead of $\ell^\infty$-norm. As remarked after \lemref{Lemma: Relative ell^infty-Loss Bound}, in contrast to this work, the conditioning on $\alpha_1,\dots,\alpha_n$ in \lemref{Lemma: Relative ell^2-Loss Bound} reflects the non-Bayesian scenario considered in \cite{Chenetal2019} 
(where $\alpha_1,\dots,\alpha_n$ are deterministic). We next derive \thmref{Thm: Minimax l^1-Risk}.

\begin{proof}[Proof of \thmref{Thm: Minimax l^1-Risk}]
The proof strategy is identical to the proof of \thmref{Thm: Minimax Relative l^infty-Risk}, but we present the details again for completeness. We first prove the minimax upper bound. As before, the inequality:
$$ \inf_{\hat{\pi}}{\sup_{\pdf \in \pdfs}{\E_{\pdf}\!\left[\left\|\hat{\pi} - \pi\right\|_{1}\right]}} \leq \sup_{\pdf \in \pdfs}{\E_{\pdf}\!\left[\left\|\hat{\pi}_* - \pi\right\|_{1}\right]} $$
holds trivially. To prove an upper bound on the extremal Bayes risk on the right hand side of this inequality, we define the event in \lemref{Lemma: Relative ell^2-Loss Bound} as:
$$ A \triangleq \left\{\frac{\left\|\hat{\pi}_* - \pi\right\|_{2}}{\left\|\pi\right\|_{2}} \leq \frac{c_{20}}{\sqrt{n p k}} \right\} . $$
Then, after taking expectations in \lemref{Lemma: Relative ell^2-Loss Bound} with respect to the law of $\alpha_1,\dots,\alpha_n$, we get that for any PDF $\pdf \in \pdfs$ and for all sufficiently large $n \in \N$:
\begin{equation}
\label{Eq: Lemma Restatement 2}
\P_{P_{\alpha}}\!\left(A\right) \geq 1 - \frac{c_{21}}{n^5} \, .
\end{equation}
Hence, for every PDF $\pdf \in \pdfs$ and for all sufficiently large $n$, we have:
\begin{align}
\E_{\pdf}\!\left[\left\|\hat{\pi}_* - \pi\right\|_{1}\right] & = \E_{\pdf}\!\left[\left\|\hat{\pi}_* - \pi\right\|_{1}\middle|A\right] \P_{\pdf}(A) + \E_{\pdf}\!\left[\left\|\hat{\pi}_* - \pi\right\|_{1}\middle|A^{c}\right] \P_{\pdf}(A^{c}) \nonumber \\
& \leq \frac{c_{20}}{\delta \sqrt{n p k}} + \frac{2 c_{21}}{n^5} \nonumber \\
& \leq \frac{2 c_{20}}{\delta \sqrt{n p k}}
\label{Eq: Intermediate Extremal Risk Bound 2}
\end{align} 
where the first equality uses the law of total expectation, the second inequality follows from \eqref{Eq: Lemma Restatement 2} and the facts that $\|\hat{\pi}_{*} - \pi\|_{1} \leq \|\hat{\pi}_{*}\|_{1} + \|\pi\|_{1} = 2$ (via the triangle inequality), and conditioned on $A$, we get:
$$ \left\|\hat{\pi}_{*} - \pi\right\|_{1} \leq \sqrt{n} \left\|\hat{\pi}_{*} - \pi\right\|_{2} \leq \frac{c_{20}}{\sqrt{p k}} \left\|\pi\right\|_2 \leq \frac{c_{20}}{\delta \sqrt{n p k}}  $$
which, in turn, uses the equivalence of $\ell^1$ and $\ell^2$-norms (via the Cauchy-Schwarz-Bunyakovsky inequality) and the simple bound $\|\pi\|_2 \leq 1/(\delta \sqrt{n})$ (due to \eqref{Eq: Canonically scaled merit parameters} and $\alpha_1,\dots,\alpha_n \in [\delta,1]$), and the third inequality \eqref{Eq: Intermediate Extremal Risk Bound 2} holds for all sufficiently large $n$ because $k = \Theta(1)$. Letting $c_{18} = 2 c_{20}/(\delta \sqrt{p k})$ and substituting it into \eqref{Eq: Intermediate Extremal Risk Bound 2}, and then taking the supremum in \eqref{Eq: Intermediate Extremal Risk Bound 2} over all PDFs $\pdf \in \pdfs$ yields the desired upper bound in the theorem statement. 

We next prove the information theoretic lower bound. As before, fix any $\varepsilon > 0$, and consider any sufficiently large $n \geq 2$ such that:
\begin{equation}
\label{Eq: Pre Lower Bound Condition Large n} 
n \geq \max\!\left\{2 + \frac{1}{\varepsilon}, \left(\frac{2 e}{1-\delta}\right)^{\! 4/\varepsilon} \! , \, \exp\!\left(\frac{(1-\delta)^2 \!\left( 2 + \delta + \frac{1}{\delta}\right) k p + 4 \log(2)\delta^2}{\delta^2 \varepsilon}\right)\right\} .
\end{equation}
Then, observe that:
\begin{align}
\inf_{\hat{\pi}}{\sup_{\pdf \in \pdfs}{\E_{P_{\alpha}}\!\left[\left\|\hat{\pi} - \pi\right\|_{1}\right]}} & \geq \sup_{t > 0}{ \, t \!\left(1 - \frac{I(\pi ; Z) + \log(2)}{\log(1/\L_{1}(t))}\right)} \nonumber \\
& \geq \sup_{t > 0}{ \, t \!\left(1 - \frac{\frac{1}{2} n \log(n) + c(\delta,p,k) n + \log(2)}{\log(1/\L_{1}(t))}\right)} \nonumber \\
& \geq \sup_{t > 0}{ \, t \!\left(1 - \frac{\frac{1}{2} n \log(n) + c(\delta,p,k) n + \log(2)}{(n-1)\log(1/t) - \log(2 e / (1-\delta)) n}\right)} \nonumber \\
& = \sup_{t > 0}{ \, t \!\left(1 - \frac{1 + \frac{2 c(\delta,p,k)}{\log(n)} + \frac{\log(4)}{n \log(n)}}{\frac{2 (n-1)\log(1/t)}{n \log(n)} - \frac{2 \log(2 e / (1-\delta))}{\log(n)}}\right)} \nonumber \\
& \geq \frac{1}{n^{\frac{1}{2} + \varepsilon}} \!\left(1 - \frac{1 + \frac{2 c(\delta,p,k)}{\log(n)} + \frac{\log(4)}{n \log(n)}}{(1 + 2 \varepsilon) \! \left(1 - \frac{1}{n}\right) - \frac{2 \log(2 e / (1-\delta))}{\log(n)}}\right) \nonumber \\
& \geq \frac{1}{n^{\frac{1}{2} + \varepsilon}} \!\left(1 - \frac{1 + \frac{\varepsilon}{4}}{1 + \frac{\varepsilon}{2}}\right) \nonumber \\
& = \frac{1}{n^{\frac{1}{2} + \varepsilon}} \!\left(\frac{\varepsilon}{4 + 2 \varepsilon}\right)
\label{Eq: Pre Lower Bound}
\end{align}
where the first inequality follows from \eqref{Eq: Bayes Risk Lower Bound}, \lemref{Lemma: Generalized Fano's Method}, and the fact that $\|\pi\|_1 = 1$, the second inequality follows from \propref{Prop: Upper Bound on MI} with $c(\delta,p,k)$ as defined in \eqref{Eq: Constant in the MI Prop}, the third inequality holds due to the following consequence of \lemref{Lemma: Upper Bound on Small Ball Probability 2}:
$$ \forall t > 0, \enspace \L_{1}(t) \leq \left(\frac{2 e}{1-\delta}\right)^{\! n} t^{n-1} \, , $$
the fifth inequality follows from setting $t = n^{-(1/2) - \varepsilon}$, and the sixth inequality follows from \eqref{Eq: Pre Lower Bound Condition Large n}, which implies the bounds \eqref{Eq: Detailed Large n condition 1}, \eqref{Eq: Detailed Large n condition 2}, and:
$$ n \geq \left(\frac{2 e}{1-\delta}\right)^{\! 4/\varepsilon} \quad \Leftrightarrow \quad \frac{2 \log\!\left(\frac{2 e}{1-\delta}\right)}{\log(n)} \leq \frac{\varepsilon}{2} \, . $$

Now, let us define the constant:
$$ c_{22} = \max\!\left\{4 \log\!\left(\frac{2 e}{1 - \delta}\right),\frac{(1-\delta)^2 \!\left( 2 + \delta + \frac{1}{\delta}\right) k p + 4 \log(2)\delta^2}{\delta^2}\right\} $$
and set $\varepsilon = c_{22}/\!\log(n)$. As mentioned earlier, it is straightforward to verify that \eqref{Eq: Pre Lower Bound Condition Large n} is satisfied for this choice of $\varepsilon$ for all sufficiently large $n$. Finally, as before, we can rewrite \eqref{Eq: Pre Lower Bound} as:
$$ \inf_{\hat{\pi}}{\sup_{\pdf \in \pdfs}{\E_{\pdf}\!\left[\left\|\hat{\pi} - \pi\right\|_{1}\right]}} \geq \left(\frac{c_{22}}{6 \exp(c_{22})}\right) \frac{1}{\log(n) \sqrt{n}} $$
for all sufficiently large $n$. Letting $c_{17} = c_{22}/(6 \exp(c_{22}))$ yields the minimax lower bound in the theorem statement. This proves the theorem.
\end{proof}

%% file: concentration.tex
\section{Concentration of measure inequalities}
\label{App: Concentration of Measure Inequalities}

In this final appendix, we present two well-known exponential concentration of measure inequalities that are used in this paper. The first of these results bounds the tail probability of the empirical average of a collection of independent bounded random variables.

\begin{lemma}[Hoeffding's Inequality {\cite[Theorems 1 and 2]{Hoeffding1963}}]
\label{Lemma: Hoeffding's Inequality}
Given independent random variables $X_1,\dots,X_n \in [a,b]$, for some constants $a < b$, we have for every $\varepsilon \geq 0$:
$$ \P\!\left(\frac{1}{n}\sum_{i = 1}^{n}{X_i - \E\!\left[X_i\right]} \geq \varepsilon\right) \leq \exp\!\left(-\frac{2 n \varepsilon^2}{(b-a)^2}\right) $$
where $\exp(\cdot)$ denotes the natural exponential function with base $e$ throughout this paper.
\end{lemma}

The second of these results provides a tighter bound on the tail probability of the empirical average of a collection of independent bounded random variables using information about the variances of the random variables.

\begin{lemma}[Bernstein's Inequality {\cite{Bernstein1946}}]
\label{Lemma: Bernstein's Inequality}
Given independent random variables $X_1,\dots,X_n$ such that for some constants $a,b > 0$, $\big|X_i - \E[X_i]\big| \leq a$ and $\VAR(X_i) \leq b$ for all $i \in [n]$, we have for every $\varepsilon \geq 0$:
$$ \P\!\left(\frac{1}{n}\sum_{i = 1}^{n}{X_i - \E\!\left[X_i\right]} \geq \varepsilon\right) \leq \exp\!\left(-\frac{n \varepsilon^2}{2 b + \frac{2}{3} a \varepsilon}\right) . $$
\end{lemma}